\documentclass[acmsmall, screen,nonacm]{acmart}
\usepackage{amsmath,graphicx}
\usepackage{bm}
\usepackage{xcolor}
\usepackage{url}
\usepackage{algorithm}
\usepackage[]{algpseudocode}
\usepackage{amsmath}
\usepackage{multirow}
\usepackage{graphicx}
\usepackage{subcaption}
\usepackage{amsfonts}
\usepackage{enumitem}
\usepackage{amsthm}
\usepackage{algorithm}
\usepackage{algpseudocode}
\usepackage{color}
\usepackage{graphicx}
\usepackage{amsthm}
\usepackage[normalem]{ulem}
\usepackage{soul}
\allowdisplaybreaks

\newcommand{\BB}[1]{\mathbb{#1}}
\newcommand{\BF}[1]{\mathbf{#1}}
\newcommand{\BC}[1]{\mathcal{#1}}
\newcommand{\proj}[2]{\pi_{#1}(#2)}

\definecolor{azure(colorwheel)}{rgb}{0.0, 0.5, 1.0}
\definecolor{darkpastelgreen}{rgb}{0.01, 0.75, 0.24}

\newcommand{\addresscomments}[1]{\textcolor{black}{#1}}

\newcommand{\Mod}[1]{\ensuremath{\ (\mathrm{mod}\ #1)}}
\newtheorem{theorem}{Theorem}[]

\newtheorem{assumptionvfl}{Assumption}[]
\newtheorem{lemmavfl}{Lemma}[]
\theoremstyle{remark}
\newtheorem{remark}{Remark}
\newcommand{\Lc}{{\mathcal{L}}}
\newcommand{\pdj}{{\nabla_{(j)}}}
\newcommand{\gm}{\tilde{\bm{\theta}}}
\newcommand{\gmj}{\tilde{\bm{\theta}}_{j}}
\newcommand{\tjk}{{\bm{\theta}}_{k,j}}
\newcommand{\ykj}{{\bm{y}}_{k,j}}
\newcommand{\tjminus}{{\bm{\theta}}_{-j}^{t_0}}
\newcommand{\gkj}{ g_{k,j}}

\newcommand{\ex}{\mathbb{E}}
\newcommand{\Gj}{{\bm{G}}_{(j)}}

\newcommand{\squeezeuppicture}{\vspace{0mm}} %reduce whitespace after pictures

\newcommand{\eqdef}{\overset{\mathrm{def}}{=}}

\newcommand{\condY}{\BF{Y}^{t_0}}
\newcommand{\fullconditional}{\bigg |~\condY}
\usepackage{xr}
\AtBeginDocument{%
  \providecommand\BibTeX{{%
    \normalfont B\kern-0.5em{\scshape i\kern-0.25em b}\kern-0.8em\TeX}}}
    \algblock{Input}{EndInput}
    \algnotext{EndInput}
    \algblock{Output}{EndOutput}
    \algnotext{EndOutput}
 
\makeatletter
\algrenewcommand\ALG@beginalgorithmic{\footnotesize}
\makeatother
\setcopyright{acmcopyright}
\copyrightyear{2022}
\acmYear{2022}
\acmDOI{00.0000/0000000.0000000}

\acmJournal{TIST}
\acmVolume{00}
\acmNumber{0}
\acmArticle{000}
\acmMonth{0}

\begin{document}

%\title{Multi-Tier Communication-Efficient Distributed Learning from Vertically Partitioned Data}
\title{Cross-Silo Federated Learning for Multi-Tier Networks with Vertical and Horizontal Data Partitioning}

%\head{Cross-Silo Federated Learning}
\author{Anirban Das}
\affiliation{%
  \institution{Rensselaer Polytechnic Institute}
  %\streetaddress{P.O. Box 1212}
  \city{Troy}
  \state{NY}
  \country{USA}
}
\email{dasa2@rpi.edu}

\author{Timothy Castiglia}
\affiliation{%
  \institution{Rensselaer Polytechnic Institute}
  %\streetaddress{P.O. Box 1212}
  \city{Troy}
  \state{NY}
  \country{USA}
}
\email{castit@rpi.edu}

\author{Shiqiang Wang}
\affiliation{%
  \institution{IBM Research}
  %\streetaddress{1 Th{\o}rv{\"a}ld Circle}
  \city{Yorktown Heights}
  \state{NY}
  \country{USA}}
\email{wangshiq@us.ibm.com}

\author{Stacy Patterson}
\affiliation{%
  \institution{Rensselaer Polytechnic Institute}
  \city{Troy}
  \state{NY}
  \country{USA}
}
\email{sep@cs.rpi.edu}

% \renewcommand{\shortauthors}{Das, et al.}

\begin{abstract}
We consider federated learning in tiered communication networks.
Our network model consists of a set of silos, each holding a vertical partition of the data. Each silo contains a hub and a set of clients, with the silo's vertical data shard partitioned horizontally across its clients.
We propose Tiered Decentralized Coordinate Descent (TDCD), a communication-efficient decentralized training algorithm for such two-tiered networks. 
The clients in each silo perform multiple local gradient steps before sharing updates with their hub to reduce communication overhead. 
Each hub adjusts its coordinates by averaging its workers' updates, and then hubs exchange intermediate updates with one another. 
We present a theoretical analysis of our algorithm and show the dependence of the convergence rate on the number of vertical partitions and the number of local updates. 
We further validate our approach empirically via simulation-based experiments using a variety of datasets and objectives.
\end{abstract}




\maketitle
%%%%%%%%%%%%%%%%%%%%%%%%%%%%%%%%%%%%%%%%%%%%%%%%%%%%%%%%%%%%%%%%%%%%%%%%%%%%%%%%
%%%%%%%%%%%%%%%%%%%%%%%%%%  INTRODUCTION  %%%%%%%%%%%%%%%%%%%%%%%%%%%%%%%%%%%%%%
%%%%%%%%%%%%%%%%%%%%%%%%%%%%%%%%%%%%%%%%%%%%%%%%%%%%%%%%%%%%%%%%%%%%%%%%%%%%%%%%
\section{Introduction}
\label{sec:intro}

In recent times, we have seen an exponential increase in the amount of data produced at the edge of communication networks. In many settings, it is infeasible to transfer an entire dataset to a centralized cloud for downstream analysis, either due to practical constraints such as high communication cost or latency, or to maintain user privacy and security~\cite{kairouz2019advances}. This has led to the deployment of distributed machine learning and deep-learning techniques where computation is performed collaboratively by a set of clients, each close to its own data source.
Federated learning has emerged as a popular technique in this space, which performs 
iterative collaborative training of a global machine learning model over data distributed over a large number of clients without sending raw data over the network.

The more commonly studied approach
of federated learning is \emph{horizontal federated learning}. In horizontal federated learning, the clients' datasets share the same set of features, but each client holds only a subset of the sample space, i.e., the data is horizontally partitioned among clients~~\cite{mcmahan2016communication, konevcny2016federated, kairouz2019advances}. In this setting, the clients train a copy of a model on their local datasets for a few iterations and then communicate their updates in the form of model weights or gradients directly to a centralized parameter server. The parameter server then creates the centralized model by aggregating the individual client updates, and the process is repeated until the desired convergence criterion is met.

Another scenario that arises in federated learning
is when clients have different sets of features, but there is a sizable overlap in the sample ID space among their datasets~\cite{yang2019federated}.
For example, the training dataset may be distributed across silos in a multi-organizational context, such as healthcare, banking, finance, retail, etc.~\cite{yang2019federated, sun2019privacy}.
Each silo holds a distinct set of features (e.g., customer/patient features);
the data within each silo may even be of a different modality; for example, one silo may have audio features, whereas another silo has image data. At the same time, there exists an overlap in the sample ID space. More specifically, the silos have a large number of customers in common.
Further, the silos may not want to share data with one another directly due to different policy restrictions, regulations, privacy/security concerns, or even bandwidth limitations in the communication network.
The paradigm of training a global model over such feature-partitioned data is called \emph{vertical federated learning}~\cite{yang2019parallel, liu2019communication}.

A common approach for model training over vertically partitioned data builds on parallel coordinate gradient descent-type algorithms~\cite{richtarik2016distributed,druvMahajanl1regularized,Kang2014}.
Here, each silo executes independent local training iterations to optimize a global model along its subset of the coordinates.
The silos exchange intermediate information to ensure that the parallel updates lead to convergence of the global model.
In early vertical federated learning works~\cite{yang2019parallel,feng2020multi,chen2020vafl}, the silos exchanged intermediate information after every local training iteration.
However, such frequent information exchanges can lead to high training latency, especially when the communication latency is high and the intermediate information is large~\cite{liu2019communication}.
The authors in a more recent work~\cite{liu2019communication} proposed an algorithm that addresses the communication bottleneck by performing multiple local training iterations in each silo before the silos exchange intermediate information for the next round of training.
While this approach is similar to the multiple local iterations in horizontal federated learning, an important distinction is that in vertical federated learning, each silo updates only on its own subset of coordinates in contrast to updating all the coordinates in the case of horizontal federated learning.

Previous works on vertical federated learning have assumed that a silo's entire dataset is contained in a single client~\cite{yang2019parallel,feng2020multi,chen2020vafl}.
However, in silos with a large organizational structure, it is natural that the data will be distributed horizontally across multiple clients.
Previous works thus fail to capture the case where the dataset within each silo is further horizontally partitioned across multiple clients.
As a motivating example, we consider two smartphone application providers %\sep{owners?}
who wish to train a global model over the datasets stored on the smartphones of their respective customer bases.
Here, the two application companies do not want to share the customer data directly with each other. At the same time, the users do not wish to share raw personal data with the companies. In this case, the data is vertically partitioned across different applications of the different companies.
Further, the data of a single company is horizontally partitioned across different application users. The silos can be thought of as the top \emph{tier} of our system architecture.
The clients inside the silos can be thought of as the bottom \emph{tier}.

We can think of another use case when there is a significant overlap in the customer ID space
between a banking silo $\BC{B}$, e.g., a banking holding company with multiple independent subsidiary banks, and an insurance silo $\BC{I}$, e.g., an insurance holding company with multiple independent subsidiaries. The two silos want to build a global model collaboratively to predict credit scores, for example.
Here, the entire data is vertically partitioned between the bank and the insurance holding companies in the top tier.
Each subsidiary of $\BC{B}$ 
acts as a client and may have a separate customer base, operating in a different geographical location in a country. Thus the entire customer data of the bank silo is horizontally partitioned across its subsidiaries in the bottom tier.
Silo $\BC{I}$ for the insurance holding company also has a similar structure. Silos $\BC{B}$ and $\BC{I}$ do not want to share information directly with each other for reasons such as privacy and security concerns and regulations.

The settings above have aspects of both the horizontal federated learning and vertical federated learning paradigms. However, it is not possible to directly apply any one paradigm  in this setting. We cannot directly apply horizontal federated learning since horizontal federated learning requires each client to have the same set of features for training the global model. In contrast, clients from different silos hold disjoint sets of features in our setting.
\addresscomments{These disjoint feature sets may not be sufficient to train an accurate model independently.
For example, a pair of labels can be only distinguishable by features
in silo A, while another pair of labels can be only distinguishable by features in silo B.}
Further, we cannot use existing vertical federated algorithms directly since the data in each silo is also horizontally distributed across clients.
Therefore, such a system setting calls for a new algorithm that combines aspects of both horizontal and vertical federated learning.

We, therefore, propose Tiered Decentralized Coordinate Descent (TDCD), a federated learning algorithm for training over data that is vertically partitioned across silos and further horizontally partitioned within silos.
Each silo internally consists of multiple clients (shown as the grey circles in Fig.~\ref{fig:system_model}) connected to a hub (shown as the orange circles in Fig.~\ref{fig:system_model}), which is a server in the silo.
The hub can be a cloud server maintained by a mobile application company. The hub can also be a server maintained by the company in charge of orchestrating the IT infrastructure of a particular silo. A hub itself does not contain data but facilitates training by coordinating clients' information.
The goal is to jointly train a model on the features of the data contained across silos, without explicitly sharing raw data between the clients and the hubs and between clients across different silos.

TDCD works by performing a non-trivial combination of parallel coordinate descent on the top tier between silos and distributed stochastic gradient descent, with multiple local steps per communication round, in the bottom tier of clients inside each silo.
More specifically, each hub maintains a block of the trainable parameters. Clients within a silo train these parameters on their local data sets for multiple iterations. The hub then aggregates the client models to update its parameter block.
To execute the local iterations, clients need embeddings from the data from clients in other silos. The hubs orchestrate this information exchange every time they update their parameter blocks. By waiting for several training iterations to exchange this information, the communication cost of the algorithm is reduced.
Our approach is thus a communication efficient combination of learning with both vertically and horizontally partitioned data in a multi-tiered network.
TDCD is novel since it interleaves both the horizontal and vertical federated learning paradigms and thus has to account for both the perturbed gradients from the horizontal federated learning component and the stale information from the vertical federated learning component. This combination leads to a different convergence analysis.

Specifically, our contributions are the following:
\begin{enumerate}
\item We present a system model for decentralized learning in a two-tier network, where data is both vertically and horizontally partitioned.
\item We develop TDCD, a communication-efficient decentralized learning algorithm using principles from coordinated descent and stochastic gradient descent.
\item We analyze the convergence of TDCD for non-convex objectives and show how it depends on the number of silos, the number of clients, and the number of local training steps. With appropriate assumptions,
TDCD achieves a convergence rate of $\BC{O}\left(\frac{1}{\sqrt{R}}\right)$, where $R$ is the total number of communication rounds between hubs and clients.
\item We validate our analysis via experiments using different objectives and different datasets. We further explore the effect of communication latency on the convergence rate.
\end{enumerate}
We observe from our experiments that
when communication latency is high with respect to the computation latency at the clients, a larger number of local iterations leads to faster convergence. 
Further, our experimental results indicate that TDCD shows little performance degradation as the number of vertical and horizontal partitions increases.

%%%%% Related Work starts here
\subsection{Related Work}
Machine learning model training with vertically-partitioned features has been studied in the context of coordinate descent algorithms.
A coordinate descent algorithm with gradient steps was first proposed in~\cite{tseng2009coordinate}, where a fully centralized system is considered.
In the decentralized context, parallel coordinate descent methods have been proposed in \cite{richtarik2016distributed,druvMahajanl1regularized,Kang2014}. These algorithms typically require a shared memory architecture or raw data sharing and shuffling of the vertical partitions between the workers. Since TDCD operates in a federated learning setting, such access to raw data and shared memory is not possible, and therefore, these algorithms are not applicable.

In recent years, in distributed learning, where data is horizontally partitioned across multiple clients periodic averaging-based methods like federated averaging and its variations have been extensively studied. Federated averaging methods have been studied in the context of convex objectives~\cite{stich2018local, wang2019adaptive, khaled2019first} and non-convex objectives~\cite{yu2019parallel,haddadpour2019convergence,tlifedprox, wang2018cooperative}.
The common approach is to take multiple (stochastic) gradient steps in parallel locally to achieve better communication efficiency as well preserve privacy by not sharing the raw data. Other works have analyzed versions of horizontal federated learning for hierarchical or multi-tier networks where data is partitioned horizontally across all clients participating in the process~\cite{9148862, wang2020local,castiglia2020multi,abad2020hierarchical}.
In horizontal hierarchical federated learning, all clients update all the weights of the global model.
In contrast, in TDCD, the clients only contribute towards updating their own silo's subset of the global model weights.

Several works have focused on distributed training with vertical partitions in a federated setting. The authors in works~\cite{chen2020vafl, hardy2017private, yang2019parallel, wu2020privacy, feng2020multi, kang2020fedmvt} propose vertical federated learning algorithms for single-tier communication networks, but they do not use local iterations in the parties during training. In all of these works, each party communicates in each iteration of training, which is communication-wise expensive. 
A more recent work on vertical federated learning~\cite{liu2019communication} solved this problem by allowing clients to take multiple local steps of training before communicating to the hub.
\addresscomments{Another recent work~\cite{wang2020federated}
considers the problem of training a latent Dirichlet allocation model
in a cross-silo federated learning setting.}
However, all these works only addressed the case where all data in a silo is contained within a single client.
In this work, we consider vertical and horizontal partitions of the dataset simultaneously in the two tiers.
TDCD performs model training in such a multi-tiered system architecture by fusing horizontal and vertical learning approaches in a novel manner.

%% new papers for revision 2:
\addresscomments{Recent work by Wang et al.~\cite{wang2021efficient}
also considers the cross-silo setting. In their method,
the silos first exchange information to form augmented local datasets. They then perform
 horizontal federated learning to train a global model. 
This is in contrast to our system, where silos only store a subset of all features and 
must combine horizontal and vertical federated learning to effectively train a model.
}

\subsection{Paper Outline}
The rest of the paper is organized as follows. We describe our system model, data partitioning, 
and loss function in Section~\ref{sec.system_model}. In Section~\ref{sec.proposedalgorithm}, we present TDCD, followed by its convergence analysis in Section~\ref{sec.convanalysis}. The proofs of the theorem and supporting lemmas are deferred to Appendix~\ref{AppendixOne}. We provide an experimental evaluation of TDCD in Section~\ref{sec.expresults} and finally conclude in Section~\ref{sec:conclusion}.

%%%%%%%%%%%%%%%%%%%%%%%%%%%%%%%%%%%%%%%%%%%%%%%%%%%%%%%%%%%%%%%%%%%%%%%%%%%%%%%%
%%%%%%%%%%%%%%%%%%%%%%%%%%  SYSTEM MODEL  %%%%%%%%%%%%%%%%%%%%%%%%%%%%%%%%%%%%%%
%%%%%%%%%%%%%%%%%%%%%%%%%%%%%%%%%%%%%%%%%%%%%%%%%%%%%%%%%%%%%%%%%%%%%%%%%%%%%%%%
\section{System Model and Problem Formulation}
\label{sec.system_model}

This section describes the system architecture, the allocation of the training data, and the loss function we seek to minimize.

\subsection{System Architecture and Training Data} \label{sec.sysarch_and_data}
We consider a decentralized system consisting of $N$ silos, shown Fig.~\ref{fig:system_model}.
Each silo consists of a hub and multiple clients connected to it in a hub-and-spoke fashion. Each silo $j$ has $K_j$ clients. The number of clients per silo may be unequal.
Our network model thus has two tiers as shown in Fig.~\ref{fig:system_model}: the top tier of hubs, shown in orange
and the bottom tier of clients in each silo, shown in gray. The hub of each silo can communicate with the hubs of all other silos. Thus, the hub communication network forms a complete graph.
The clients within a silo only communicate with that silo's hub.
 If we map this system model on the bank and insurance example in Section~\ref{sec:intro}, there are two silos, a banking silo corresponding to a banking holding company and an insurance silo corresponding to an insurance holding company. The clients in the banking silo are subsidiary banks of the banking holding company,
and similarly, for the insurance silo, the clients are different divisions of the insurance holding company.

\begin{figure*}[htpb]
\centering
\begin{subfigure}[h]{.48\linewidth}
\centering
    \includegraphics[width=\linewidth]{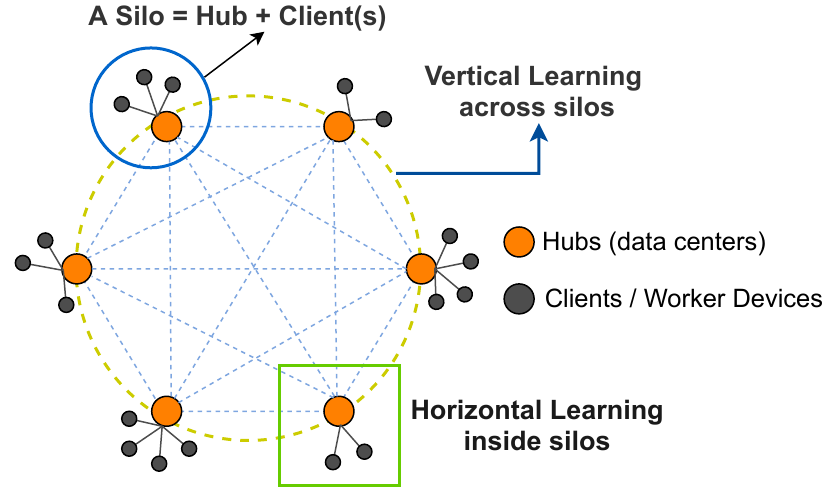}
    \caption{}
    \label{fig:system_model}
\end{subfigure}%
\begin{subfigure}[h]{.48\linewidth}
\centering
    \includegraphics[width=\linewidth]{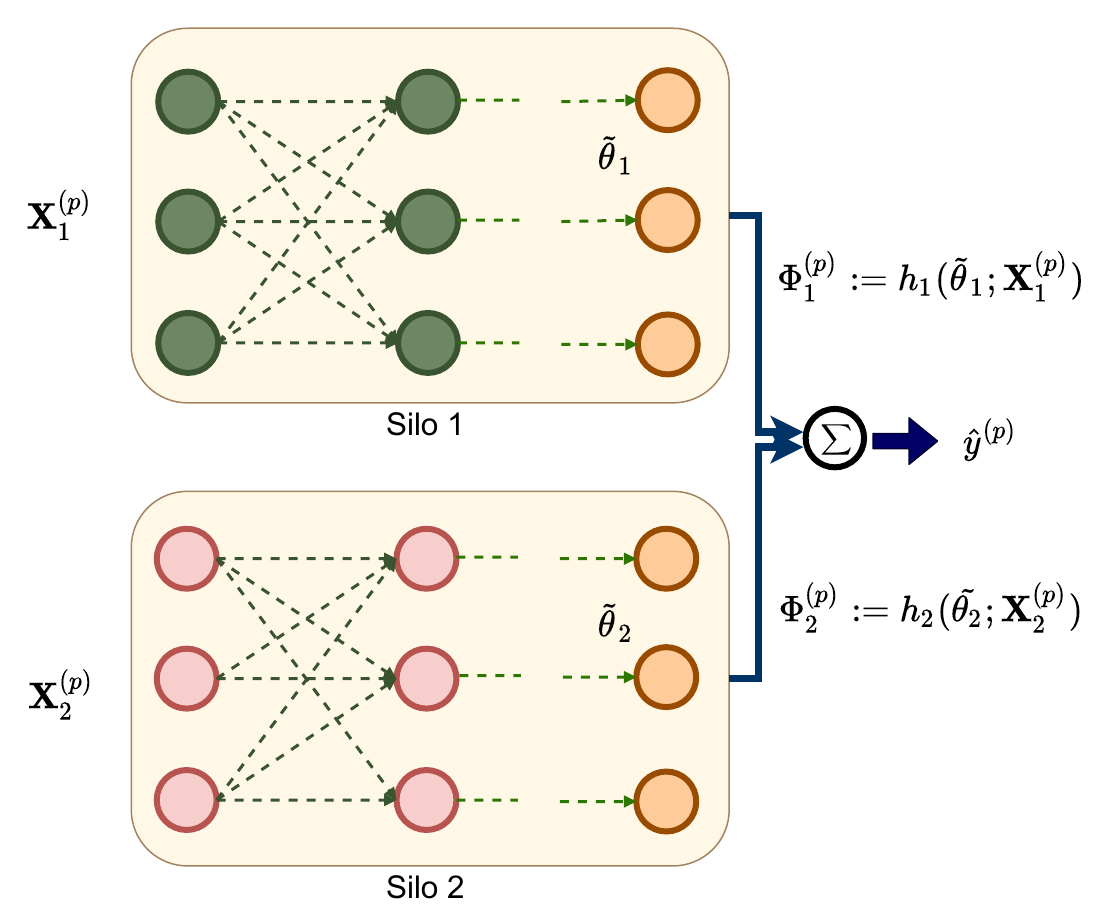}
    \caption{}
    \label{fig:global_system_model}
\end{subfigure}
\caption{(a) System architecture (b) An example of client models and the inputs and outputs to the system. The figure shows a client $k$ at each of the two silos and the concatenation of their output embedding to find the target $\tilde{y}$ corresponding to a single data sample $p$. }
\end{figure*}

\begin{figure*}[htpb]
\centering
\centering
    \includegraphics[width=0.5\linewidth]{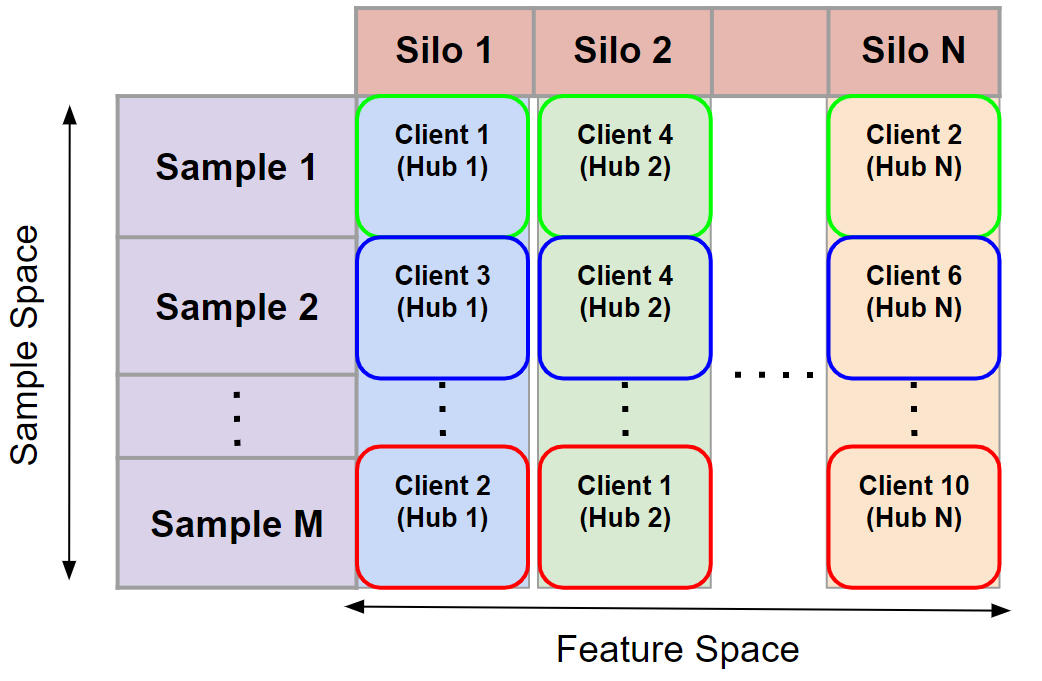}
    \caption{\addresscomments{
        Illustration of data partitioning among hubs and clients in silos. 
        Each silo $j$ owns a vertical partition of the full dataset $\BF{X}$. The features of each sample $i$ are distributed among clients of different silos. %A client $k$ in silo $j$ stores a horizontal partition of silo $j$'s dataset, $\BF{X}_{k,j}$ which is a subset of sample IDs.
   Each client in each hub owns a subset of features of some sample IDs. 
}}
    \label{fig:data_partition}
\end{figure*}

The training data consists of $M$ sample IDs that are common across all silos.
Each sample has $D$ features.
The data is partitioned vertically across the $N$ silos so that each silo owns a disjoint subset of the features.
Within each silo, its data is partitioned horizontally across its clients so that each client holds data for the vertical partition of a subset of IDs.
\addresscomments{This data partitioning is illustrated in Figure~\ref{fig:data_partition}.}
More formally, we can express the entire training dataset by a matrix $\BF{X} \in \BB{R}^{M\times D}$.
We denote the set of data, i.e., the columns of $\BF{X}$, owned by silo $j$ by $\BF{X}_{j}$.
The vertical slice of data in silo $j$, $\BF{X}_j$, is, therefore, a matrix of dimension $M\times D_j$, where $D_j$ denotes the number of features held by silo $j$.
Within silo $j$, each client holds some rows of $\BF{X}_{j}$.
We denote the horizontal shard of $\BF{X}_{j}$ that is held by client $k$ in silo $j$ by the matrix $\BF{X}_{k,j}$.
Lastly, we denote a sample $p$ of the dataset (single row of $\BF{X}$) as $\BF{X}^{(p)}$, and $\BF{X}^{(p)}_{j}$ denotes the features of the $p$th sample corresponding to silo $j$.
We note that for each sample $p$, $\BF{X}_{j}^{(p)}$ is only held in a single client in silo $j$.
A sample ID in the banking and insurance example corresponds to an individual who is a common customer to both the banking and the insurance silos.
 \addresscomments{The features of the samples in the banking silo consist of balance, installments, debit history, credit history, etc., and the banking features for each customer are stored by the bank subsidiary that serves this customer. Similarly, the insurance silo has insurance-related features such as policy details, premium, age, dependent information, past claims, etc., and each customer's insurance features are stored by the insurance subsidiary that serves that customer.} 
We assume that the datasets do not change for the duration of the training.

We denote the label for a sample $p$ by $y^{(p)}$. We assume that each client stores a copy of the sample labels for each of its sample IDs. We denote the copy of sample labels for $\BF{X}_{k,j}$ by $\BF{y}_{k,j}$. In many low-risk scenarios, such as credit history, labels can be shared across silos. Labels can also be shared when the sample IDs themselves are anonymized. In scenarios where the labels are sensitive, for example, the case when the labels are only present in a single silo, it is still possible to extend TDCD to work without label sharing, albeit with extra communication steps, in an approach similar to ~\cite{liu2019communication}. We elaborate on this further in Section~\ref{sec.proposedalgorithm}, Remark~\ref{remark.2}.

\subsection{Model Formulation and Loss Function} \label{sec.problemformulation}

The objective is to train a global model $\gm$ that can be decomposed into $N$ coordinate block partitions, each corresponding to the model parameters assigned to a silo:
\begin{align}
\gm = [ \gm_{1}^{\intercal}, \ldots, \gm_{N}^{\intercal}]^{\intercal} \label{eqn.definition_of_gobal_model}
\end{align}
where $\gmj$ is the block of coordinate partition for silo $j$. $\gm$ could be weights of a linear model such as linear or logistic regression, for example, or could be trainable weights of a neural network.
Let us represent the function computed by a silo $j$ as $h_j$, where $h_j$ is a map from the higher dimensional input space $\BF{X}_j$ to the lower-dimensional embedding space $\Phi_j$ parameterized by $\gm_j$. 
We denote the embedding that is output by silo $j$ for a sample $p$ by $\Phi_j^{(p)} \eqdef h_j(\gm_j;\BF{X}_j^{(p)})$.

In Fig.~\ref{fig:global_system_model} we show an example of a global model expressed by Equation~\ref{eqn.definition_of_gobal_model}. In this example, the global model is a neural network (e.g., a CNN or an LSTM) partitioned across Silo 1 and Silo 2.
Each partition of the global model is also a neural network
and each client in a silo $j$ holds a copy of the neural network architecture represented by $h_j$.
Each silo $j$ is responsible for updating its partition of $\gm$, i.e., $\gm_j$ in the training algorithm.
A client can independently compute the embedding $\Phi_{j}^{(p)}$ if $p$ is a sample ID contained in that client.

The goal of the training algorithm is to minimize an objective function with the following structure:
\begin{align}
& \min_{\gm} \Lc (\gm,\BF{\BF{X;y}}) \eqdef \frac{1}{M} \sum\limits_{p=1}^M \ell(\Phi_{1}^{(p)}; \Phi_{2}^{(p)}; \cdots \Phi_{N}^{(p)}; y^{(p)})
\end{align}
where $\ell$ is the loss corresponding to a data sample $p$ with features $\BF{X}^{(p)}$.
Similar to the work in~\cite{liu2019communication}, our system architecture considers the following form for $\ell$:
\begin{align}
\ell (\Phi_{1}^{(p)}; \Phi_{2}^{(p)}, \cdots, \Phi_{N}^{(p)}; y^{(p)}) \eqdef \ell \left( \sum\limits_{j=1}^N h_j(\gm_j;\BF{X}_j^{(p)});y^{(p)} \right). \label{eqn.loss_separable}
\end{align}
The objective function $\Lc$ can be non-convex.
Since $\ell$ is a function of $\gm$, consequently, $\Lc$ is a function of $\gm$. 
When the context is clear, we drop $\BF{X}, \BF{y}$ from $\Lc(\cdotp)$ to simplify the notation.

%%%%%%%%%%%%%%%%%%%%%%%%%%%%%%%%%%%%%%%%%%%%%%%%%%%%%%%%%%%%%%%%%%%%%%%%%%%%%%%%
%%%%%%%%%%%%%%%%%%%%%%%%%%  ALGORITHM     %%%%%%%%%%%%%%%%%%%%%%%%%%%%%%%%%%%%%%
%%%%%%%%%%%%%%%%%%%%%%%%%%%%%%%%%%%%%%%%%%%%%%%%%%%%%%%%%%%%%%%%%%%%%%%%%%%%%%%%
\section{Proposed Algorithm}
\label{sec.proposedalgorithm}

\begin{table}[t]
\centering
\caption{{Description of all the variables used in Algorithm~\ref{alg:algorithm}.}}
\resizebox{\textwidth}{!}{%
\begin{tabular}{|l|l|}
\hline
\textbf{Notation} & \textbf{Definition}                                                                                                                                                  \\ \hline
N                  & Number of silos / number of vertical partitions.                                                                                                                      \\ \hline
$K_j$              & Number of clients in silo $j$.                                                                                                                                        \\ \hline
$\BF{X}_j^{(p)}, y^{(p)}$  & Features of $p$th sample held by in silo $j$ and corresponding target label.
\\ \hline
$\Lc$              & Objective function.
\\ \hline
$l$                & Loss function.                                                                                                                       \\ \hline
$h_j$             & Function computed by the silo $j$. Can be represented by a neural network.
\\ \hline
$\gm$              & Global model parameters.                                                                                                                                                         \\ \hline
$\gmj$             & Block of global model  parameters corresponding to silo $j$, maintained by hub $j$.                                                                                                                       \\ \hline
$\tjk^t$   & \begin{tabular}[c]{@{}l@{}}Local version of the weights $\tilde{\theta}_j^t$ each\\ client in hub $j$ updates in iteration $t$. $\gmj^t=\frac{1}{K_j} \sum\limits_{k=1}^{K_j} \left [\tjk^{t}\right]$. \end{tabular}
\\ \hline
$\Phi_j^{(p)}$     & Intermediate information for the $j$th coordinate block for a single sample $p$.                                                                                      \\ \hline
$\zeta$            & Set of sample IDs forming a mini-batch. $\zeta^{t_0}$ denotes the set at iteration $t_0$.                                                                                                             
\\ \hline
$\zeta_{k,j}^{t_0}$& Subset of sample IDs from a mini-batch $\zeta^{t_0}$ held at $k$th client of silo $j$.                                                                                                              \\ \hline
$\eta$            & Learning rate.
\\ \hline
$\Phi_{j}^{\zeta^{t_0}}$            & Intermediate information of mini-batch $\zeta^{t_0}$ formed at hub $j$.                                                                                                              \\ \hline
$\Phi_{k,j}^{\zeta^{t_0}}$          & Subset of information from $\Phi_j^{\zeta^{t_0}}$ that is relevant to samples from $\zeta^{t_0}$ held by client $k$ of silo $j$.                                                                                                              \\ \hline
$\Phi_{-j}^{\zeta^{t_0}}$             & Intermediate information that is received by silo $j$ from all other silos for the mini-batch $\zeta^{t_0}$.
\\ \hline
$\Phi_{-k,j}^{\zeta^{t_0}}$               & Subset of information
from $\Phi_{-j}^{\zeta^{t_0}}$ that is relevant to samples from $\zeta^{t_0}$ held by client $k$ of silo $j$.
\\ \hline
$\gkj$            & Local stochastic partial derivative of the loss function in client $k$ of hub $j$.                                                                                                              \\ \hline
$\pi_{k,j}$            & A projection function such that $\proj{k,j}{\Phi_{-j}^{\zeta}} = \Phi_{-k,j}^{\zeta}$.                                                                                                               \\ \hline
$Q$                & Number of local iterations between each global communication round.
\\ \hline
\end{tabular}%
}\color{black}
\label{tab:variables}
\end{table}%

In this section, we present our Tiered Decentralized Coordinate Descent algorithm.
The pseudocode is given in Algorithm~\ref{alg:algorithm}.
The notation used in this section is summarized in Table~\ref{tab:variables}.

\begin{algorithm}[h]
    \caption{Tiered Decentralized Coordinate Descent (TDCD)}
    \begin{algorithmic}[1]
        \For {$t = 0, \ldots, T$}
            \If{$t=0$}
                \For {$j = 1, \ldots, N$ silos \emph{in parallel}}
                    \State Initialize hub $j$ model: $\gmj^{0}$ \label{line.hub_init}
                \EndFor
            \EndIf
            \If{$t\Mod{Q}=0$}
                \For {$j = 1, \ldots, N$ silos \emph{in parallel}}
                    \State Randomly select a set of sample IDs: $\zeta^{t_0}$. \label{line.pickmini-batch}
                    \For {$k = 1,\ldots, K_j$ clients \emph{in parallel}} \label{line:start_updates_from_clients_client}
                        \State Hub $j$ sends copy of $\gmj^{t}$, $\zeta^{t_0}$ to client
                        \State Client $k$ initializes local model: $\tjk^{t} \leftarrow \gmj^{t}$ \label{line:client_replace_local_model}
                        \State Client $k$ calculates intermediate information $\Phi_{k,j}^{\zeta^{t_0}}$ and sends it to hub $j$ \label{line:send_updates_from_clients_client}
                    \EndFor \label{line:end_updates_from_clients_client}
                \EndFor

                \For {$j = 1, \ldots, N$ silos \emph{in parallel}} \label{line:start_hubs_aggregate_info}
                 \State Hub $j$ forms $\Phi_{j}^{\zeta^{t_0}}  \gets \cup_{k=1}^{K_j} \{\Phi_{k,j}^{\zeta^{t_0}}\}$
                    \State All hubs exchange $\Phi_{j}^{\zeta^{t_0}}, j=1,\ldots, N$ \label{line:broadcast_hubs}
                    \State Hub $j$ calculates $\Phi_{-j}^{\zeta^{t_0}} \leftarrow \sum_{l=1,l\neq j}^N \Phi_{l}^{\zeta^{t_0}}$
                              \EndFor \label{line:end_hubs_aggregate_info}

                \For {$j = 1, \ldots, N$ silos \emph{in parallel}}

                            \State For each client $k$, hub $j$ calculates $\Phi_{-k,j}^{\zeta^{t_0}} = \proj{k,j}{\Phi_{-j}^{\zeta^{t_0}}}$ and sends it to client $k$ \label{line.projection}
                \EndFor \label{line:decen_stop}
            \EndIf

            \For {$j= 1,\ldots, N$ silos \emph{in parallel}} \label{line:dc_train_start}
                \For {$k = 1, \ldots, K_j$ clients \emph{in parallel}} \label{line:GD_start}
                    \State $\tjk^{t+1}\leftarrow \tjk^{t} - \eta g_{k,j}(\Phi_{-k,j}^{\zeta^{t_0}},\Phi_{k,j}^{\zeta^{t_0}};\zeta^{t_0}_{k,j})$ \text{(See Eqn.~\ref{eqn.partial_derivative_def})} \label{line:local_gradient_update}
                \EndFor \label{line:GD_stop}
            \EndFor  \label{line:dc_train_stop}
            \If{$(t+1)\Mod{Q}=0$}
                \For {$j = 1, \ldots, N$ silos \emph{in parallel}}
                    \State Hub $j$ computes $\gmj^{t+1} \leftarrow \frac{1}{K_j} \sum_{k=1}^{K_j} \tjk^t$ \label{line:federated_avg} 
                \EndFor
            \EndIf
     \EndFor
     \end{algorithmic}
    \label{alg:algorithm}
\end{algorithm}

In Algorithm~\ref{alg:algorithm}, in iteration $t=0$ (line~\ref{line.hub_init}), the hub in each silo initializes its model weights $\gmj^{0}$.
Then, in iteration $t=0$, and every $Q$th iteration thereafter,
the hubs first select a set of sample IDs $\zeta^{t_0}$, 
samples, randomly drawn from the global dataset $\BF{X}$, to form a mini-batch. 
This selection can be achieved, for example, via pseudo-random number generators at each silo that are initialized with the same seed.
This step is described in line~\ref{line.pickmini-batch}.
In lines~\ref{line:start_updates_from_clients_client}-\ref{line:end_updates_from_clients_client}, each hub $j$ first sends a copy of its weights $\gmj^t$ and the sample IDs $\zeta^{t_0}$ 
to the clients in its silo.
The clients replace their local model weights with the updated weights received from the hub, in line~\ref{line:client_replace_local_model}. We define $\tjk^t$ as the local version of the weights $\gmj^t$ that client $k$ updates.
Next, TDCD computes the information needed for each client to calculate its local stochastic partial derivatives.
We call the embedding $\Phi_j^{(p)}$ the \emph{intermediate information} for the $j$th coordinate block for a single sample $p$.
In TDCD, to obtain the set of intermediate information for a single mini-batch $\zeta$, client $k$ in silo $j$ computes the set of information $\Phi_{k,j}^{\zeta} = \{ \Phi_{j}^{(p)}\}_{p \in \zeta}$, for example, by executing forward propagation through its local copy of the neural network $h_j$. 
Let $\zeta_{k,j}^{t_0}$ denote the subset of sample IDs from $\zeta^{t_0}$ held at client $k$ of silo $j$.   
In line~\ref{line:send_updates_from_clients_client}, each client $j$ computes the intermediate information corresponding to the sample IDs in $\zeta_{k,j}^{t_0}$. We denote this by
$\Phi_{k,j}^{\zeta^{t_0}}$. The client then sends this information to its hub.

In the next step between lines~\ref{line:start_hubs_aggregate_info}-\ref{line:end_hubs_aggregate_info}, each hub gathers the intermediate information from its clients.
Each hub computes the union over the sets of updates $\{\Phi_{k,j}^{\zeta^{t_0}}\}$ from each of its clients to obtain $\Phi_{j}^{\zeta^{t_0}}$.
This information must be propagated to the clients in the other silos so that they can compute their stochastic partial derivatives for the mini-batch $\zeta^{t_0}$.
Each hub $j$ broadcasts $\Phi_j^{\zeta^{t_0}}$ to all other hubs in line~\ref{line:broadcast_hubs}.
For hub $j$, we denote the intermediate information obtained from other hubs for sample $p$ by $\Phi_{-j}^{(p)} \eqdef \sum_{l=1, l\neq j}^N \Phi_{j}^{(p)}$. Let $\Phi_{-j}^{\zeta^{t_0}}$ denote the entire intermediate information that is received by silo $j$ from all other silos for the set of sample IDs in $\zeta^{t_0}$. 
We similarly denote $\Phi_{-k,j}^{\zeta^{t_0}}$ as the subset of information
from $\Phi_{-j}^{\zeta^{t_0}}$ corresponding to $\zeta^{t_0}_{k,j}$.
We define projection functions $\pi_{k,j}$ such that $\proj{k,j}{\Phi_{-j}^{\zeta^{t_0}}} = \Phi_{-k,j}^{\zeta^{t_0}}$. 
Hub $j$ sends $\Phi_{-k,j}^{\zeta}$ to each clients $k$ in line~\ref{line.projection}.
Alternatively, a hub can send the entire $\Phi_{-j}^{\zeta}$ to the client, and the client can extract the rows corresponding to its own samples.

Finally, after receiving this intermediate information, each client $k$ of silo $j$ can calculate its local stochastic partial derivative for 
its subset of $\zeta^{t_0}$ denoted by $\zeta^{t_0}_{k,j}$.
Let $g_{k,j}$ denote the local stochastic partial derivative of the loss function with respect to the block of coordinates $j$ in client $k$:
\begin{align}
g_{k,j}(\Phi_{-k,j}^{\zeta^{t_0}},\Phi_{k,j}^{\zeta^{t_0}};\zeta^{t_0}_{k,j}) \eqdef \frac{1}{|\zeta^{t_0}_{k,j}|} \sum_{p \in \zeta^{t_0}_{k,j}} \nabla_{\tjk} \ell (\Phi_{1}^{(p)}; \Phi_{2}^{(p)}; \cdots \Phi_{N}^{(p)}; y^{(p)}). \label{eqn.partial_derivative_def}
\end{align}

To train its local model, each client executes $Q$ local stochastic gradient steps on the feature block corresponding to its silo (lines~\ref{line:dc_train_start}-\ref{line:dc_train_stop}):
\begin{align}
\tjk^{t+1}= \tjk^{t} - \eta g_{k,j}(\Phi_{-k,j}^{\zeta^{t_0}},\Phi_{k,j}^{\zeta^{t_0}};\zeta^{t_0}_{k,j}) \label{eq.gradient_step_local_client_vfl} .
\end{align}
Here, $\eta$ is the learning rate.
Note that $t_0$ represents the most recent iteration $t_0 < t$ in which the client received intermediate information from the other silos. 
A client uses the same subset of sample IDs $\zeta^{t_0}_{k,j}$ for all $Q$ iterations.
During local training, when a client executes (\ref{eq.gradient_step_local_client_vfl}), it recalculates its own local intermediate information $\Phi_{k,j}^{\zeta^{t_0}}$ using the most up to date value of its set of coordinates $\tjk^t$. However, it reuses the intermediate information $\Phi_{-k,j}^{\zeta^{t_0}}$ that it received from the other silos throughout the $Q$ iterations of local training. Therefore the information about the rest of the coordinates after iteration $t_0+1$ is effectively stale.

Finally, after $Q$ local iterations, in line~\ref{line:federated_avg}, the hubs average the model weights from the clients, where hub $j$ updates the $j$th block coordinates of global model weights as $\gmj^{t+1}=\frac{1}{K_j} \sum\limits_{k=1}^{K_j} \left [\tjk^{t}\right]$. This step is similar to horizontal federated learning.
The entire process from start to the completion of $Q$ iterations of parallel training inside silos comprises a \textit{communication round}, and this communication round is repeated until convergence.

In TDCD, clients only communicate their local model weights and intermediate information every $Q$ iterations. This contrasts to distributed SGD algorithms, where the clients need to synchronize with a coordinating hub
in each iteration. This allows TDCD to save bandwidth by increasing $Q$, especially when the size of the model or intermediate information is large.
We explore how $Q$ impacts the convergence of TDCD in the next section.

\begin{remark}
Throughout the training process, each hub maintains the block of model parameters for its silo.
To form the full global model, each hub's model definition and parameters can be copied to a central server for downstream inference purposes. This can be done periodically for model checkpointing or at the end of the training procedure.
\end{remark}

\begin{remark} \label{remark.2}
We assume that every client has the labels for all of its samples.
However, in cases when the labels are sensitive and sharing the labels for a sample ID across silos is not feasible, the label information for a sample ID may only be present in a client in one silo. In this case, we could modify our algorithm in the following way, similar to~\cite{liu2019communication}: the clients in all silos send the intermediate information for a sample to the client that has the label for the sample. The client with the label information calculates the loss and the partial derivatives, which can then be propagated back to the other clients for use in the local gradient steps. This modification would significantly increase the communication cost of the algorithm. We note that the modified algorithm is mathematically equivalent to TDCD, albeit with a higher communication cost. Hence, the convergence analysis given in Section~\ref{sec.convanalysis} can be trivially extended to this case. 
\end{remark}

%%%%%%%%%%%%%%%%%%%%%%%%%%%%%%%%%%%%%%%%%%%%%%%%%%%%%%%%%%%%%%%%%%%%%%%%%%%%%%%%
%%%%%%%%%%%%%%%%%%%%%%%%%%%%%%%%%%%%%%%%%%%%%%%%%%%%%%%%%%%%%%%%%%%%%%%%%%%%%%%%

%%%%%%%%%%%%%%%%%%%%%%%%%%%%%%%%%%%%%%%%%%%
%%%%%%%%%%%%%%%%%%%%%%%%%%%%%%%%%%%%%%%%%%%%%%%%%%%%%%%%%%%%%%%%%%%%%%%%%%%%%%%%
%%%%%%%%%%%%%%%%%%%%%%%%%%  Conv. Analysis%%%%%%%%%%%%%%%%%%%%%%%%%%%%%%%%%%%%%%
%%%%%%%%%%%%%%%%%%%%%%%%%%%%%%%%%%%%%%%%%%%%%%%%%%%%%%%%%%%%%%%%%%%%%%%%%%%%%%%%
\section{Convergence Analysis}
\label{sec.convanalysis}
In this section, we provide the convergence analysis of the TDCD algorithm.
We first make the following set of common assumptions~\cite{bottou2018optimization, yu2019parallel, liu2019communication} about the loss function $\Lc$ and the local stochastic partial derivatives $\gkj$ at each client.
\begin{assumptionvfl} \label{assumption.lipschitz}
    The gradient of the loss function is Lipschitz continuous with constant $L$; further, the local stochastic partial derivative of $\BC{L}$, with respect to each coordinate block $j$, denoted by $\gkj(\cdot)$, is Lipschitz continuous with constant $L_{k,j}$, i.e., for all $\bm{\theta}_1, \bm{\theta}_2 \in \mathbb{R}^D$
   \begin{align}
     &\textcolor{black}{\lVert \nabla \Lc(\bm{\theta}_1) - \nabla \Lc(\bm{\theta}_2) \parallel \leq L\parallel \bm{\theta}_1 - \bm{\theta}_2\rVert} \\
    & \| \gkj (\bm{\theta}_1) - \gkj (\bm{\theta}_2) \| \leq L_{k,j} \| \bm{\theta}_1 - \bm{\theta}_2\|.
    \end{align}
   \end{assumptionvfl}
\begin{assumptionvfl}\label{assumption.lowerbound}
The function $\Lc$ is lower bounded so that for all $\bm{\theta} \in \mathbb{R}^D, \Lc(\bm{\theta}) \geq \Lc_{inf}$.
\end{assumptionvfl}

\begin{assumptionvfl}\label{assumption.unbiased}
Let $\zeta$ be a mini-batch drawn uniformly at random from all samples.
Then, for all $\bm{\theta} \in \mathbb{R}^D$,
\begin{align}
    &\ex_{\zeta} \left[\gkj(\Phi_{-j}^{\zeta},\Phi_{j}^{\zeta};\zeta) \right]= \pdj \Lc(\bm{\theta}) \label{assumption.3a}\\
      &\ex_{\zeta} \left[\| \gkj(\Phi_{-j}^{\zeta},\Phi_{j}^{\zeta};\zeta) - \pdj \Lc(\bm{\theta}) \|^2 \right] \leq \sigma^2_j \label{assumption.3b}
\end{align}
where the expectation is over the choice of mini-batch $\zeta$, and $\pdj \Lc (\cdot)$ denotes the partial derivative of $\BC{L}$, with respect to coordinate block $j$. Recall $\Phi_{j}^{\zeta}= h_j(\bm{\theta}_j;\BF{X}_j^{\zeta}) $ for a given $\theta$ and hence the local stochastic partial derivatives $g_{k,j}$ are functions of ${\bm{\theta}}$.
\end{assumptionvfl}

We also use the following definition:
\begin{align}
L_{max} \eqdef \max_{k} L_{k,j}. ~
\end{align}

Our analysis is based on the evolution of the global model $\gm $ following Algorithm \ref{alg:algorithm}. It is noted that the components of $\gm$, $\gmj$ are realized every $Q$ iterations, but for theoretical analysis we study the evolution of a virtual $\gmj$ at each iteration with $\gmj^t= \frac{1}{K_j} \sum_{k=1}^{K_j} \tjk^t$.

We further define the following two quantities :
\begin{align}
& \Gj^t =  \frac{1}{K_j} \sum_{k=1}^{K_j} \gkj(\Phi_{-k,j}^{\zeta^{t}},\Phi_{k,j}^{\zeta^{t}};\zeta^t_{k,j}) \\
&\bm{G}^t = [(\bm{G}_{(1)}^t)^T, \ldots, ({\bm{G}_{(N)}}^t)^T]^T.
\end{align}
We can then write the evolution of the global model as follows,
\begin{align}
\gm^{t+1} = \gm^t - \eta {\bm{G}}^t. \label{eqn.original}
\end{align}

We now provide the main theoretical result of this paper. We observe that since each mini-batch is used for $Q$ local iterations, this reuse leads to a bias in some of the stochastic partial derivatives. Nevertheless, with some care, it is possible to prove the algorithm convergence.
The proof of the theorem is deferred to Appendix~\ref{AppendixOne}
\begin{theorem} \label{thm.main_theorem}
Under Assumptions \ref{assumption.lipschitz}, \ref{assumption.lowerbound}, \ref{assumption.unbiased}, when the learning rate $\eta$ satisfies the following condition:
\begin{align}
0 \leq \eta \leq \frac{1}{8Q \max (L, L_{max})} 
\end{align}
the expected averaged squared gradients of $\Lc$ over $T=QR>0$ local iterations satisfies the following bound: 
\begin{align}
\frac{1}{R}\sum\limits_{t_0=0}^{R-1} \ex \left[\| \nabla \Lc(\gm^{t_0})\|^2 \right] \leq \frac{4 (\Lc(\gm^{0}) - \ex [ \Lc(\gm^{T})])}{\eta QR} +  4 \left (\eta L Q + 4\eta^2 Q^2 L^{ 2}_{max} + 8\eta^3 Q^3 L L^{ 2}_{max} \right) \sum\limits_{j=1}^N \sigma_j^2. \label{eqn.final_convergence_rate}
\end{align}
\end{theorem}
The convergence error in Theorem~\ref{thm.main_theorem} results from the parallel updates on the coordinate blocks, which depend on the number of vertical partitions $N$ via the summation over $\sigma_j$, as well the staleness of the intermediate information due to multiple local iterations, i.e., the value of $Q$.
With an increase in the number of vertical partitions $N$, the error term increases linearly. The error also depends on $Q$; however, all the terms containing $Q$ also contains the learning rate $\eta$ with the same exponent. Hence, in practice, if $Q$ is offset by a suitable learning rate $\eta$, then we can leverage multiple local iterations to achieve faster convergence in terms of wall clock time as we show in Section~\ref{sec.expresults}. Choosing a very small $\eta$ will decrease the convergence error, but it will but increase the first term on the right-hand side of (\ref{eqn.final_convergence_rate}), leading to slower convergence.

\begin{remark}
When $N=1$, i.e., there is only one silo, then the algorithm reduces to local SGD with horizontally partitioned data, and the convergence rate is similar to that obtained in~\cite{yu2019parallel} with respect to the orders of $Q, K_j, L$.
\end{remark}

\begin{remark}
When $K_j=1$, for $j=1,2,\ldots, N$, i.e., there exists only one client per silo, the algorithm reduces to vertical federated learning with a single tier, and the convergence rate is similar with respect to the orders of $Q, N, L$ to that obtained in \cite{liu2019communication}.
\end{remark}

\begin{remark}
If $\eta \propto \frac{1}{Q\sqrt{R}}$, then 
the (average) expected gradient squared norm converges with a rate of $\mathcal{O}\left(\frac{1}{\sqrt{R}}\right)$, where $R$ is the total number of communication rounds. 
\end{remark}

%%%%%%%%%%%%%%%%%%%%%%%%%%%%%%%%%%%%%%%%%%%%%%%%%%%%%%%%%%%%%%%%%%%%%%%%%%%%%%%%
%%%%%%%%%%%%%%%%%%%%%%%%%%  EXPERIMENTS   %%%%%%%%%%%%%%%%%%%%%%%%%%%%%%%%%%%%%%
%%%%%%%%%%%%%%%%%%%%%%%%%%%%%%%%%%%%%%%%%%%%%%%%%%%%%%%%%%%%%%%%%%%%%%%%%%%%%%%%

\section{Experimental Results}
\label{sec.expresults}

We verify the convergence properties of TDCD with respect to the different algorithm parameters of the system.

\subsection{Datasets and Models}

\textbf{CIFAR-10:}
We train a CNN model on the CIFAR-10 dataset~\cite{cifar10}.
CIFAR-10 is a set of $3\times32\times32$ pixels color images with
$50$,$000$ images in the training set and $10$,$000$ images in the test set. The goal is multi class classification with ten classes of images. 
We use $N=2$ for all the experiments and divide each CIFAR-10 image vertically into two parts in all three channels ($3\times32\times16$). Each silo trains a local ResNet18 model with a penultimate linear layer, and finally, we use a cross-entropy loss function that is parameterized by the outputs from all the silos. We assign the left ($3\times32\times16$) vertical partition of all images to one silo and the right vertical partition of all images to the other silo.
Comparing with the system model in Fig.~\ref{fig:global_system_model}, each $h_j$ of a silo is represented by the ResNet18 model described above.
\\
\textbf{MIMIC-III:}
We train an LSTM model using the MIMIC-III (Medical Information Mart for Intensive Care) dataset,~\cite{johnson2016mimic}, which consists of anonymized information of patients admitted to critical care units in a hospital. We follow the data processing steps from~\cite{Harutyunyan2019} to obtain $14,681$ training samples and $3,236$ test samples. Each sample consists of $48$ time steps corresponding to 48 hours, and each time step has 76 features. The goal is to predict in-hospital mortality as a binary classification task. 
For training, we split the data vertically along the feature axis. Each silo trains an LSTM model 
with a linear layer, and finally we use binary cross-entropy as class prediction output.
Comparing with the system model in Fig.~\ref{fig:global_system_model}, each $h_j$ of a silo is represented by the LSTM model described above.
\\
\addresscomments{\textbf{ModelNet40:}
We train a CNN model on the ModelNet40 dataset~\cite{modelnet}.
ModelNet40 is a set of images of CAD models of different objects.
For each CAD model, there are $12$ images from different camera views.
The dataset consists of $9$,$843$ CAD models in the training set and $2$,$468$ CAD models in the test set. 
The goal is multi-class classification with ten classes of objects. 
We use $N=12$ for all the experiments, assigning a single view of each CAD model to each silo.
Each silo trains a local ResNet18 model with a penultimate linear layer, and finally, 
we use a cross-entropy loss function that is parameterized by the outputs from all the silos. 
Comparing with the system model in Fig.~\ref{fig:global_system_model}, each 
$h_j$ of a silo is represented by the ResNet18 model described above.}

\begin{figure*}
\centering
\begin{subfigure}[b]{.48\linewidth}
      \centering
    \includegraphics[width=\linewidth]{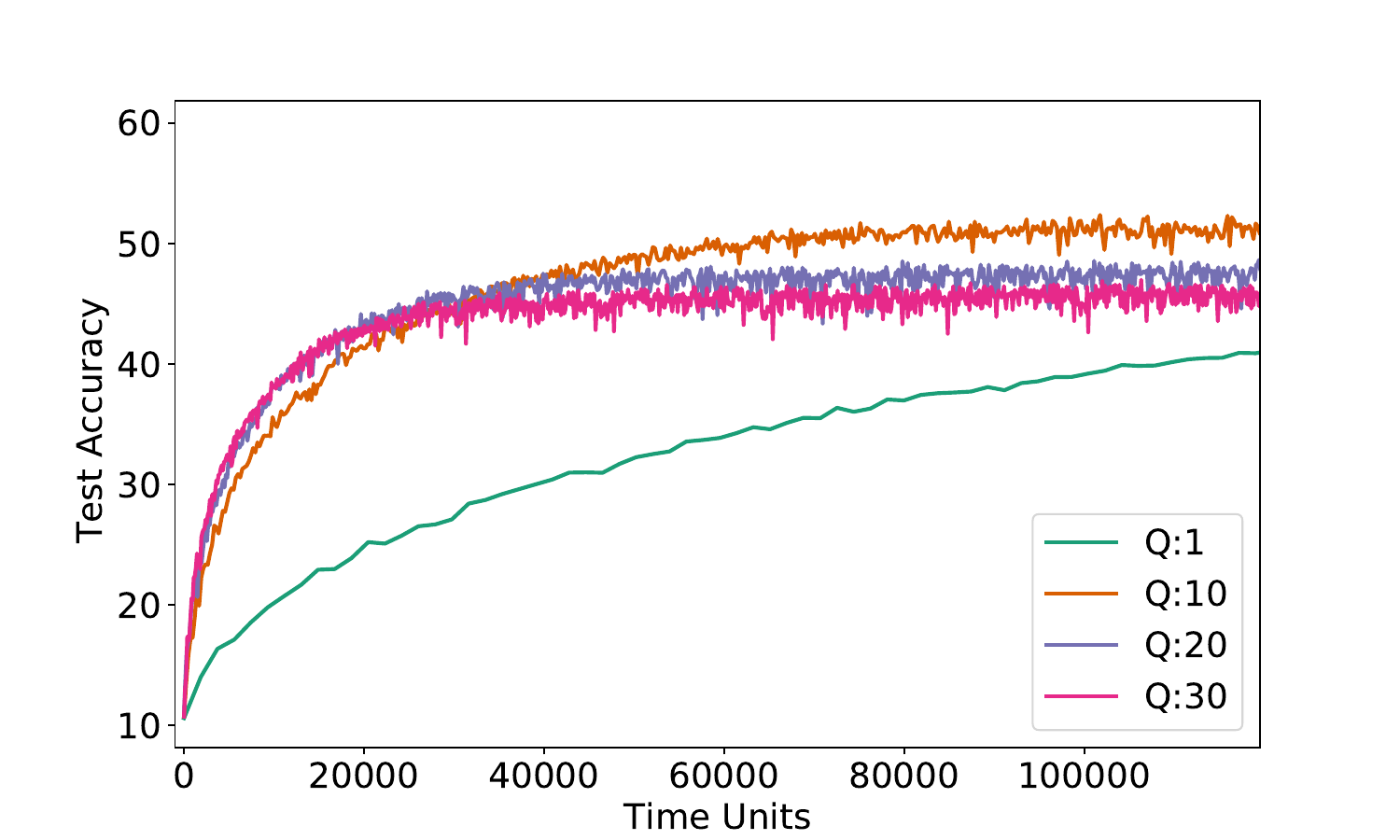}
    \caption{$t_{comm}=10,~t_{comp}=1,~K=50$.}
    \label{fig:cifar10_comm_vs_Q_10_K50}
     \end{subfigure}
\begin{subfigure}[b]{.48\linewidth}
      \centering
    \includegraphics[width=\linewidth]{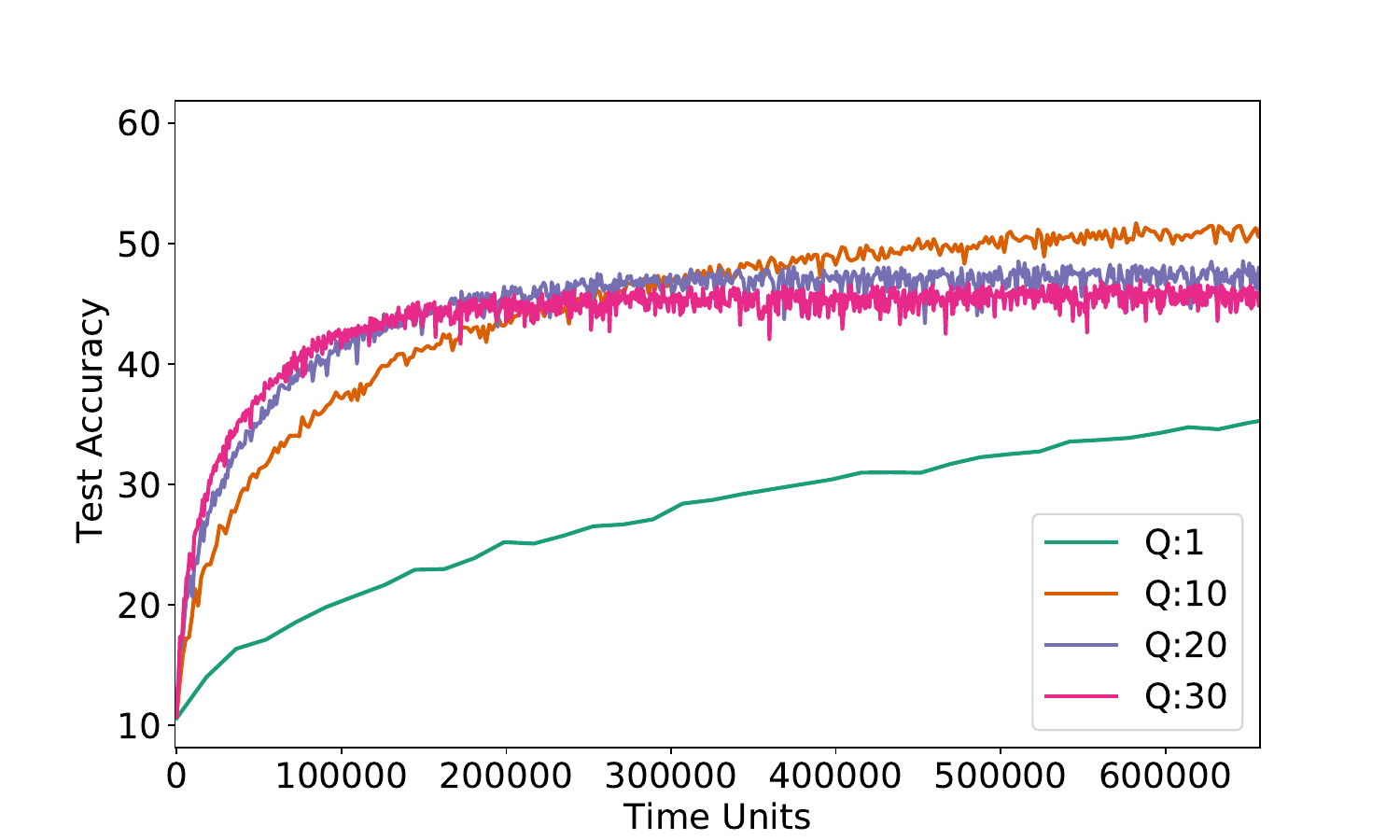}
    \caption{$t_{comm}=100,~t_{comp}=1,~K=50$.}
    \label{fig:cifar10_comm_vs_Q_100_K50}
     \end{subfigure}
\begin{subfigure}[b]{.48\linewidth}
      \centering
    \includegraphics[width=\linewidth]{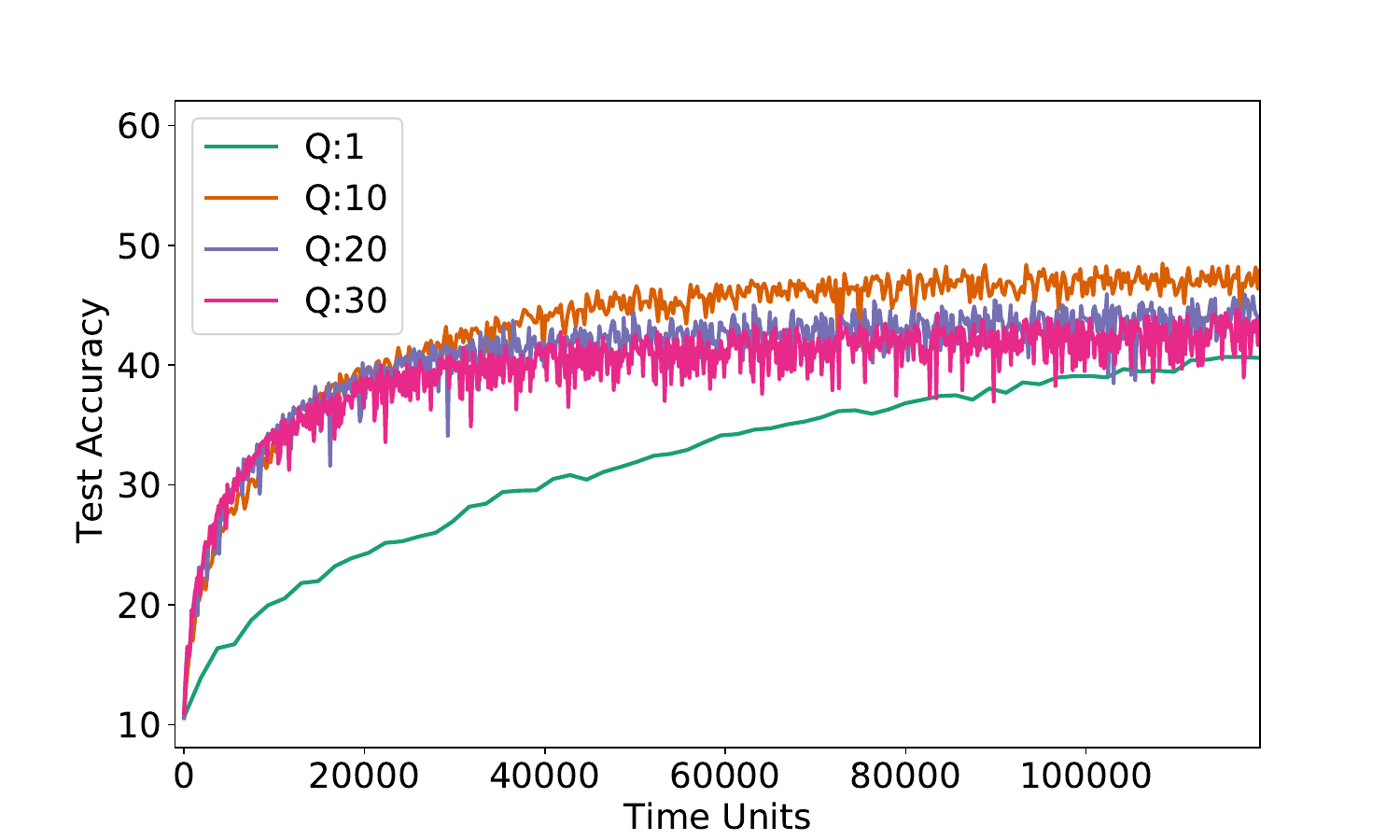}
    \caption{$t_{comm}=10,~t_{comp}=1,~K=100$.}
    \label{fig:cifar10_comm_vs_Q_10_K100}
     \end{subfigure}
\begin{subfigure}[b]{.48\linewidth}
      \centering
    \includegraphics[width=\linewidth]{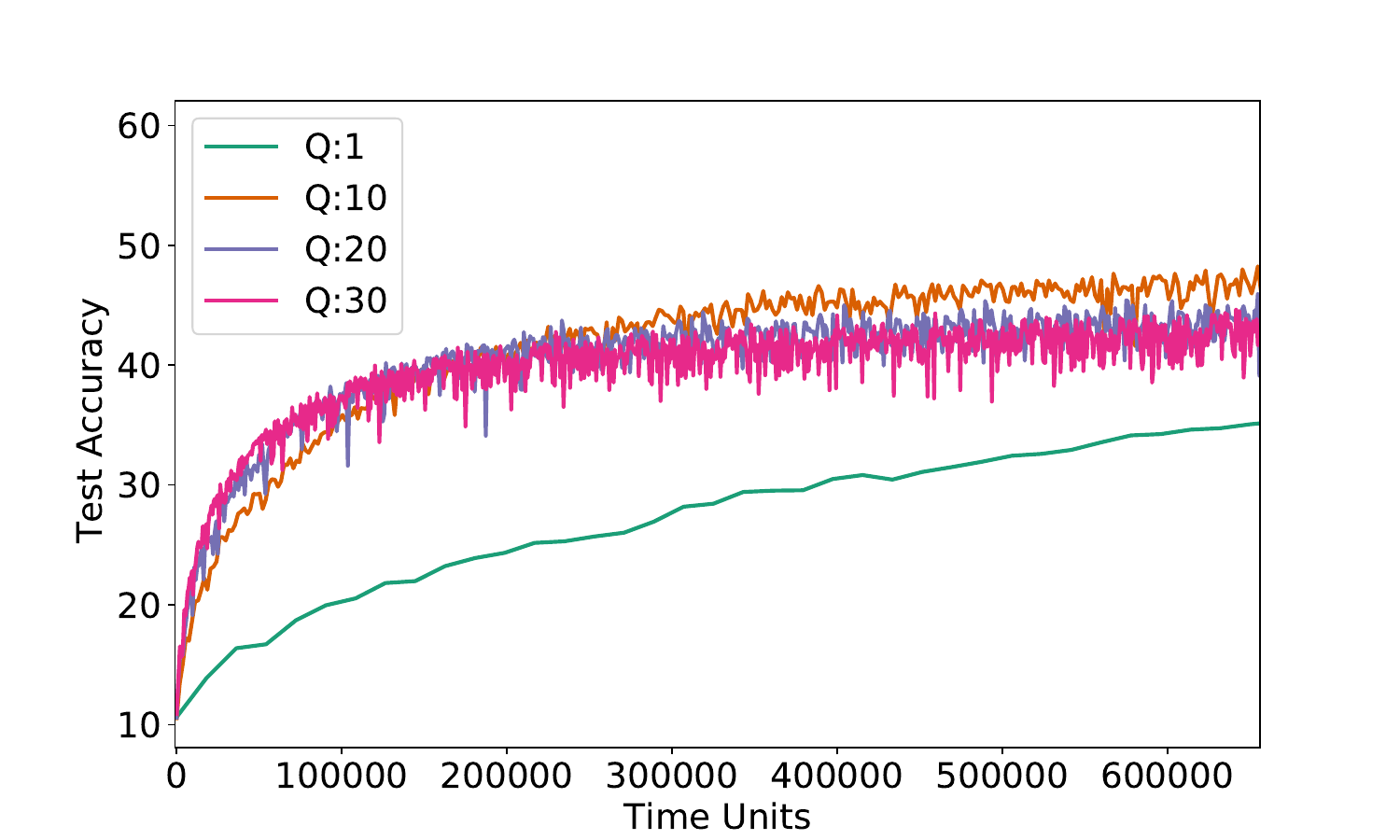}
    \caption{$t_{comm}=100,~t_{comp}=1,~K=100$.}
    \label{fig:cifar10_comm_vs_Q_100_K100}
     \end{subfigure}
\caption{Performance of TDCD on CIFAR-10. We show the test accuracy score vs. training time units $t_Q$.
$N=2$ in all four figures. $K=50$ in Figures (a) and (b), and  $K=100$ in Figures (c) and (d). In Figures (a) and (c), we use a lower ratio of communication vs. computation latency with $t_{comm}=10$. In Figures (b) and (d), we use a higher ratio of communication vs. computation latency with $t_{comm}=100$. $t_{comp}=1$ in all figures.}
     \label{fig.cifar10_commvsQgeneral}
\squeezeuppicture
\end{figure*}

\subsection{Experimental Setup} \label{subsec.tdcd_exp_setup}

\addresscomments{In our experiments, we simulate a multi-tier communication network
and train TDCD on the CIFAR-10, MIMIC-III, and ModelNet40 datasets\footnote{\addresscomments{The source code is available at \url{https://github.com/rpi-nsl/TDCD}.}}.}
In all figures, $N$ represents the number of silos (vertical partitions of the dataset), and $K_j$ represents the number of clients (horizontal partitions) in each silo.
For simplicity in our experiments, we consider that all the silos have the same number of clients $K_j=K$ for all $j$. Thus, there are $N \times K$ clients, in total, participating in the TDCD algorithm.
To distribute the training data among the clients in a silo, we generate a random permutation of the $M$ sample IDs. Each client in silo $j$ is  assigned a contiguous block of $\frac{M}{K}$ sample IDs from this permutation. 
\addresscomments{Across all experiments, we use a mini-batch size of 2000 for CIFAR-10 and MIMIC-III, 
and we use a mini-batch size of 320 for ModelNet40.}
In expectation, each client holds data for an equal fraction of each mini-batch.

The training loss is calculated using the global model $\gm$ and the full training data matrix. In CIFAR-10, we use the test accuracy metric calculated on the entire test dataset as generalization performance measure.
Further, since MIMIC-III is an imbalanced class dataset, with only 16\% positive sample, we use the F1 score as generalization performance measure on the test dataset. F1 score is the harmonic mean of the precision and recall performance of the global model $\gm$ calculated on the entire test dataset.
\addresscomments{For ModelNet40, we use top-$5$ test accuracy as our performance measure.
When using top-$5$ accuracy, a prediction is considered correct if any of the $5$ highest
probabilities of the model's output matches the correct class label.}

In each experiment, for each value of $Q$, we choose the learning rate using a grid search. For CIFAR-10 we search for a learning rate in the range [0.0001, 0.00001], for MIMIC-III we search in the range [0.1, 0.001], \addresscomments{and for ModelNet40 we search in the range [0.001, 0.00005].} For each $Q$ and $K$, we let TDCD train for 5,000 iterations for CIFAR-10, 10,000 iterations for MIMIC-III, \addresscomments{and 4,000 iterations for ModelNet40,} and pick the learning rate with the lowest training loss. 

Lastly, we study the performance of TDCD in terms of both the number of training iterations and simulated latency time units.
We denote the round trip communication latency between the hub and clients, and between hubs and one another, as $t_{comm}$ . We denote the local computation latency of each step of training by $t_{comp}=1$ time unit.
Since for each communication round, i.e., $Q$ local iterations, clients and hubs communicate twice, and hubs communicate with each other once, the total latency of one communication round can be simplified as $t_{Q} = 3 \times t_{comm} + Q \times t_{comp}$.
\addresscomments{In our experiments, we simulate different ratios between the computation time $t_{comp}$
and communication time $t_{comm}$ to explore the impact of lower and higher ratios of computation speed to communication latency.}

\subsection{Results}
Here, we describe our experimental results.

\subsubsection{Communication Efficiency} \label{sec.latency_vs_convergence}

In the first set of experiments, we study the training latency of TDCD with simulated latency time units.
Since in many communication networks, the communication latency between nodes dominates the computation latency inside a node, we study higher values of $t_{comm}$ compared to $t_{comp}$.
We study two cases, one with a low $t_{comm}$ of 10 units and one with a high $t_{comm}$ latency of 100 units, with $t_{comp}=1$ unit in both cases.

\begin{figure*}[t]
\centering
\begin{subfigure}[b]{.48\linewidth}
      \centering
    \includegraphics[width=\linewidth]{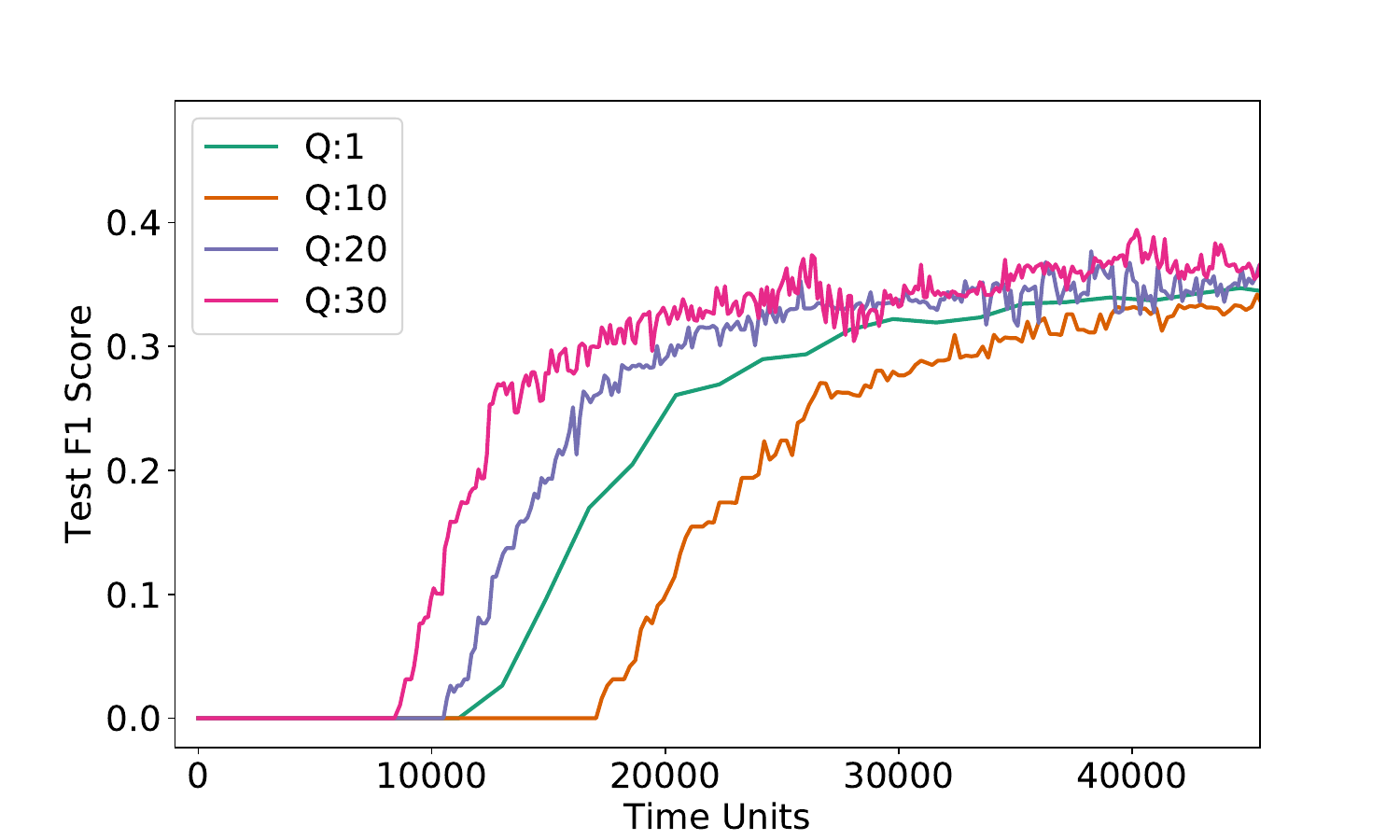}
    \caption{$t_{comm}=10,~t_{comp}=1,~K=20$.}
    \label{fig:mimic3_comm_vs_Q_10_K20}
     \end{subfigure}
\begin{subfigure}[b]{.48\linewidth}
      \centering
    \includegraphics[width=\linewidth]{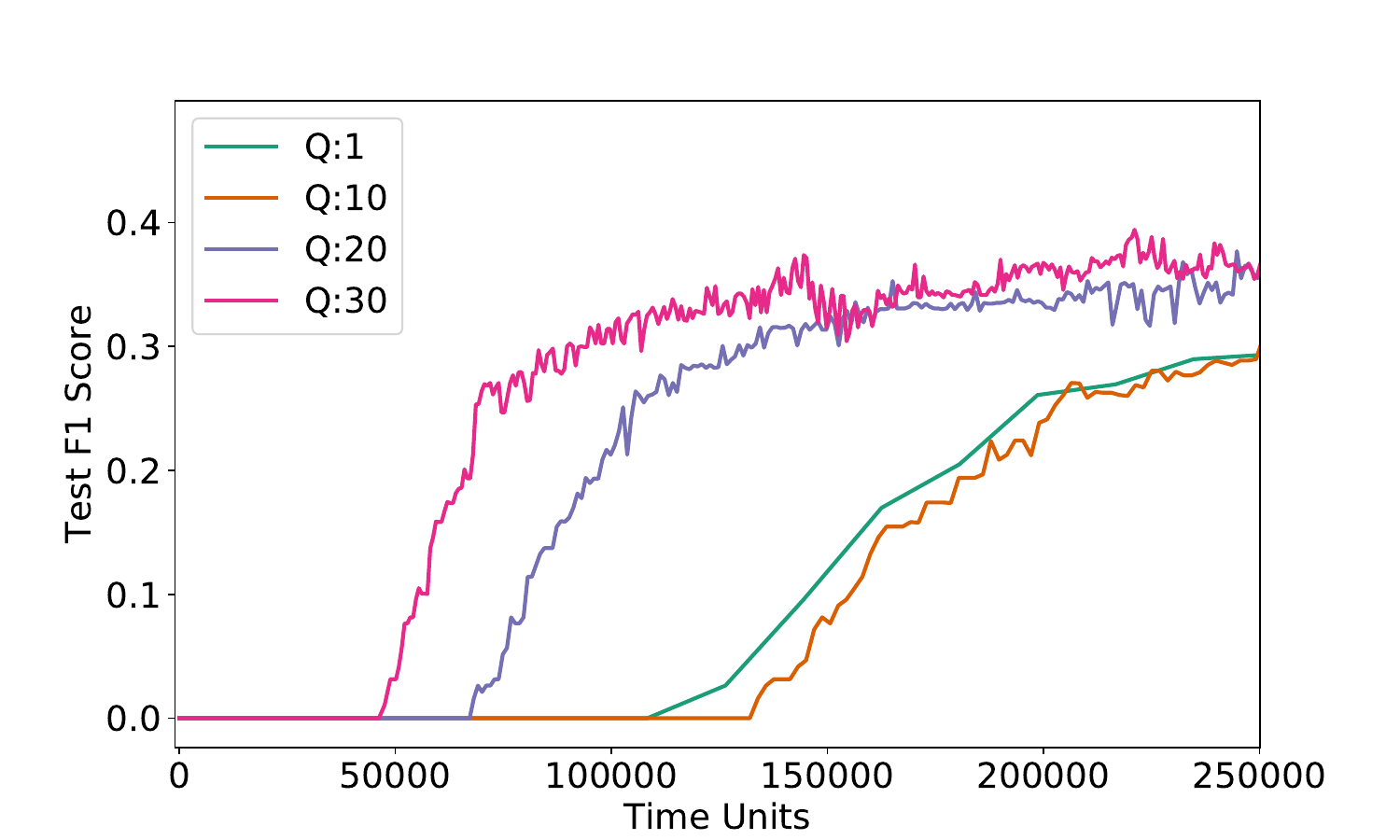}
    \caption{$t_{comm}=100,~t_{comp}=1,~K=20$.}
    \label{fig:mimic3_comm_vs_Q_100_K20}
     \end{subfigure}
\begin{subfigure}[b]{.48\linewidth}
      \centering
    \includegraphics[width=\linewidth]{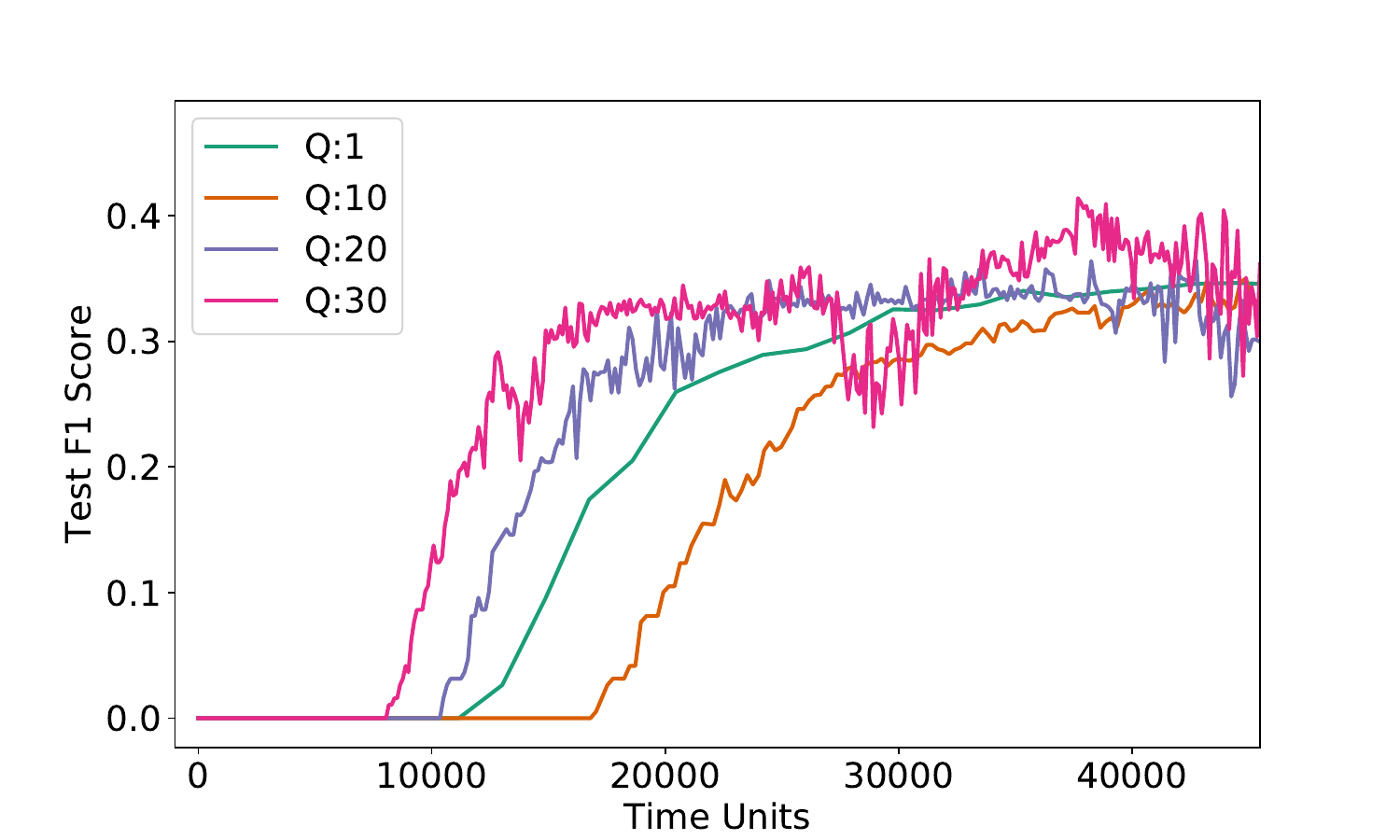}
    \caption{$t_{comm}=10,~t_{comp}=1,~K=50$.}
    \label{fig:mimic3_comm_vs_Q_10_K50}
     \end{subfigure}
\begin{subfigure}[b]{.48\linewidth}
      \centering
    \includegraphics[width=\linewidth]{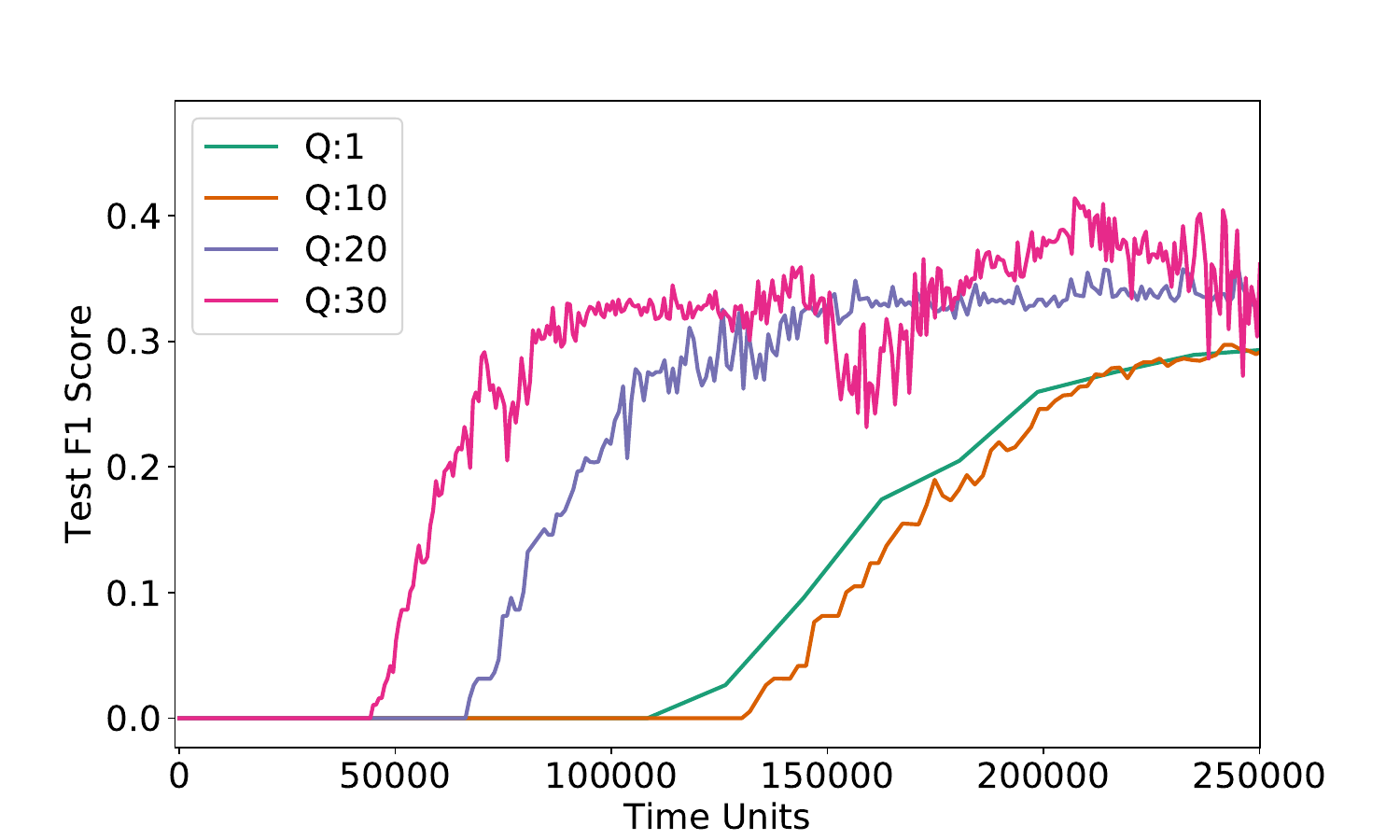}
    \caption{$t_{comm}=100,~t_{comp}=1,~K=50$.}
    \label{fig:mimic3_comm_vs_Q_100_K50}
     \end{subfigure}
\caption{Performance of TDCD on MIMIC-III. We show the test F1 score vs. training time units $t_Q$. $N=4$ in all four figures.  $K=20$ in Figures (a) and (b) and $K=50$ in Figures (c) and (d). In Figures (a) and (c), we use a lower ratio of communication vs. computation latency with $t_{comm}=10$. In Figures (b) and (d), we use a higher ratio of communication vs. computation latency with $t_{comm}=100$. $t_{comp}=1$ in all figures.}
     \label{fig.mimic3_commvsQgeneral}
\squeezeuppicture
\end{figure*}

In Fig.~\ref{fig:cifar10_comm_vs_Q_10_K50} and Fig.~\ref{fig:cifar10_comm_vs_Q_100_K50}, we show results for CIFAR-10 with 
$N=2$ silos (vertical partitions) and $K=50$ clients per silo.
The Y-axis gives test accuracy, and the X-axis gives the training latency in time units.
When $t_{comm}=10$, we observe in Fig.\ref{fig:cifar10_comm_vs_Q_10_K50} that the experiments with $Q>1$ converge faster than for $Q=1$. Similar trends can be seen for the higher value of $t_{comm}=100$ in Fig.~\ref{fig:cifar10_comm_vs_Q_100_K50}.
Moreover, comparing the results for $t_{comm}=10$ and $t_{comm}=100$, we observe that the gap in test accuracy between $Q>1$ and $Q=1$ widens as the ratio between $t_{comm}$ and $t_{comp}$ increases. We thus see more benefits from larger a $Q$ when the communication latency is larger.

Results for MIMIC-III are shown in Fig.~\ref{fig:mimic3_comm_vs_Q_10_K20} and Fig.~\ref{fig:mimic3_comm_vs_Q_100_K20} for $N=4$ silos and $K=20$ clients per silo. The Y-axis gives the test F1 Score, and the X-axis gives the training latency in time units.
The results follow similar trends to those in the previous set of experiments. Comparing the case with $t_{comm}=10$ in Fig.~\ref{fig:mimic3_comm_vs_Q_10_K20}
with the case with $t_{comm}=100$ in Fig.~\ref{fig:mimic3_comm_vs_Q_100_K20},
we again observe that the gap in the test F1 score widens as the ratio between $t_{comm}$ and $t_{comp}$ increases, resulting in similar benefits from a larger $Q$ as in CIFAR-10.
\begin{figure*}
\centering
\begin{subfigure}[b]{.48\linewidth}
      \centering
    \includegraphics[width=\linewidth]{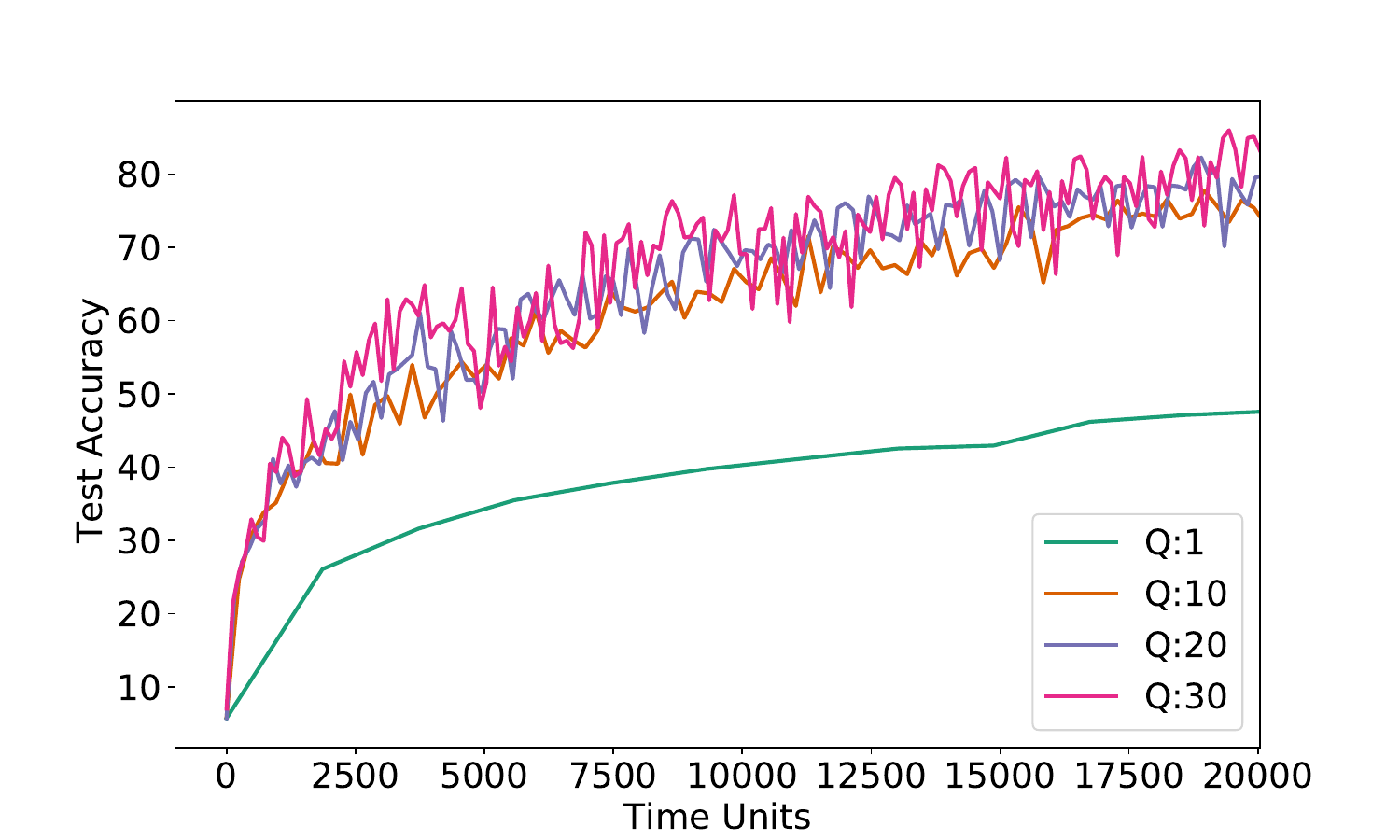}
    \caption{$t_{comm}=10,~t_{comp}=1,~K=10$.}
    \label{fig:modelnet_comm_vs_Q_10_K50}
     \end{subfigure}
\begin{subfigure}[b]{.48\linewidth}
      \centering
    \includegraphics[width=\linewidth]{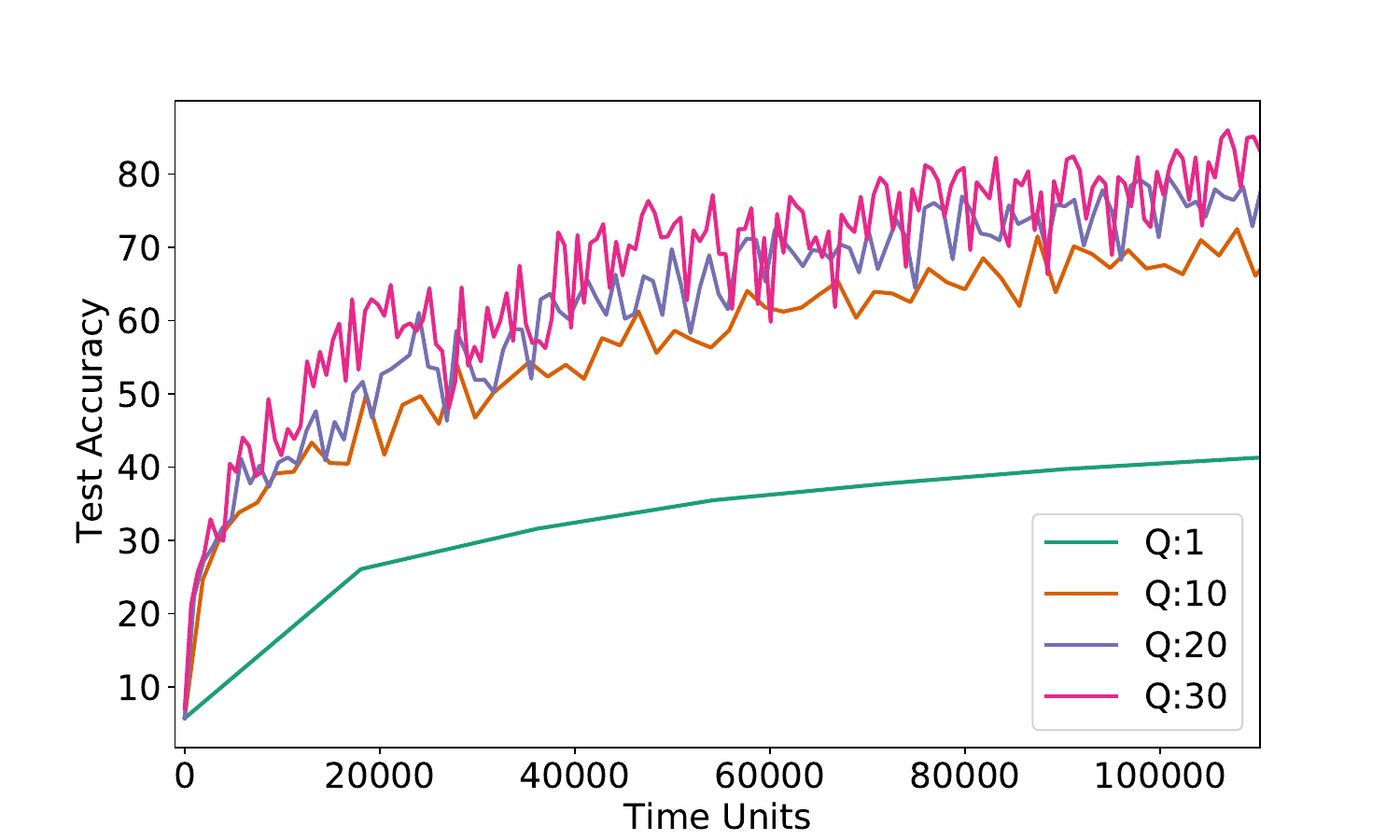}
    \caption{$t_{comm}=100,~t_{comp}=1,~K=10$.}
    \label{fig:modelnet_comm_vs_Q_100_K50}
     \end{subfigure}
\begin{subfigure}[b]{.48\linewidth}
      \centering
    \includegraphics[width=\linewidth]{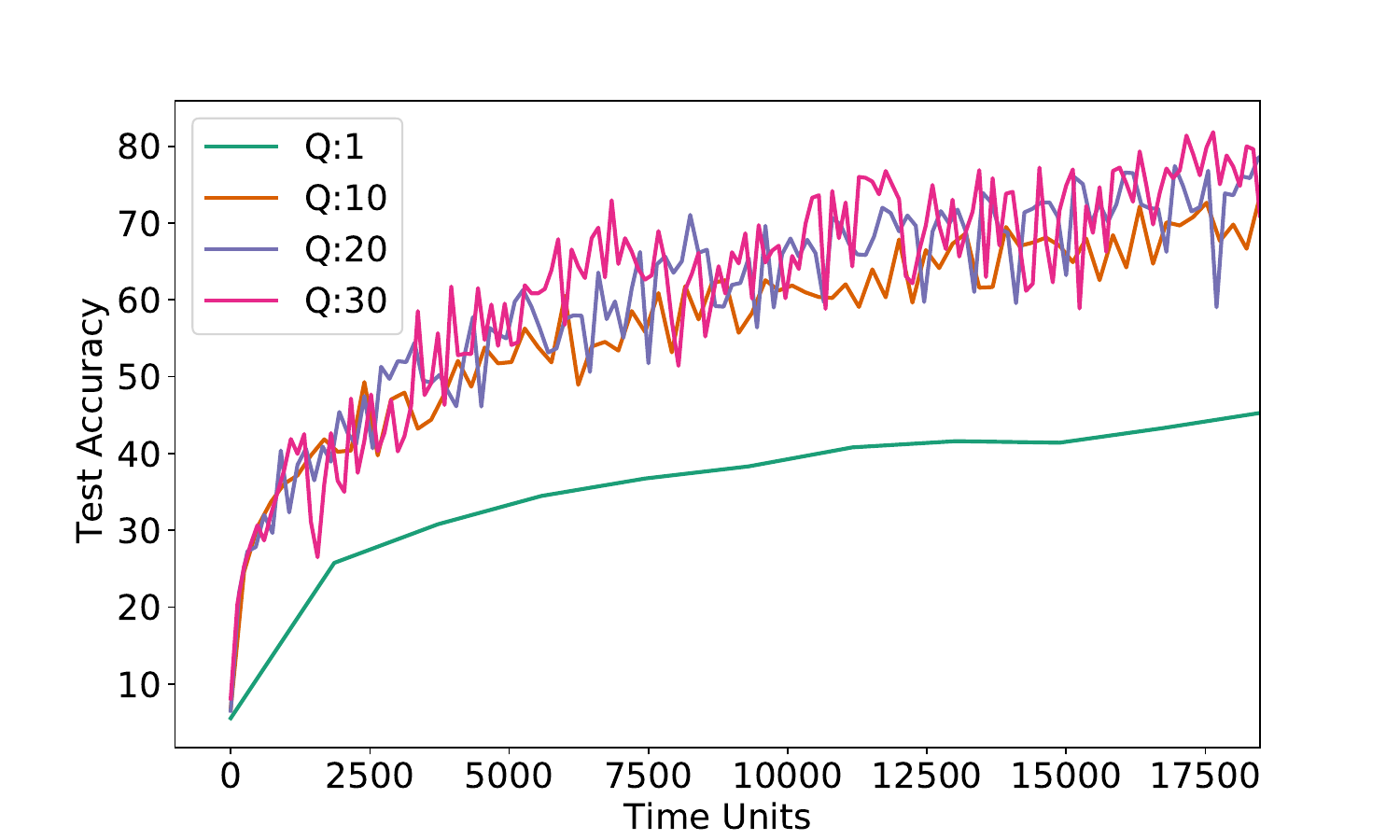}
    \caption{$t_{comm}=10,~t_{comp}=1,~K=20$.}
    \label{fig:modelnet_comm_vs_Q_10_K100}
     \end{subfigure}
\begin{subfigure}[b]{.48\linewidth}
      \centering
    \includegraphics[width=\linewidth]{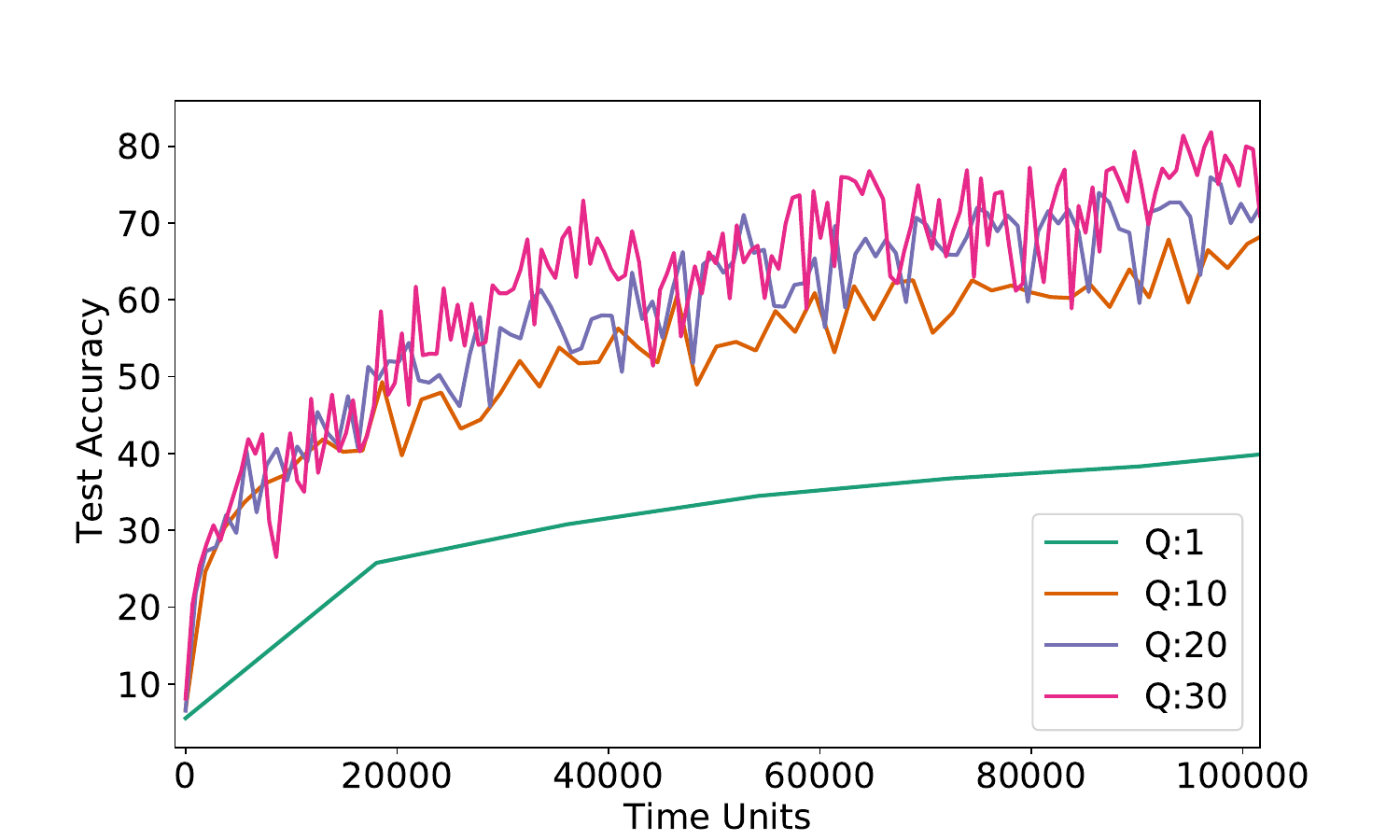}
    \caption{$t_{comm}=100,~t_{comp}=1,~K=20$.}
    \label{fig:modelnet_comm_vs_Q_100_K100}
     \end{subfigure}
\caption{\addresscomments{Performance of TDCD on ModelNet40. We show the top-$5$ test accuracy score vs. training time units $t_Q$.
    $N=12$ in all four figures. $K=10$ in Figures (a) and (b), and  $K=20$ in Figures (c) and (d). In Figures (a) and (c), we use a lower ratio of communication vs. computation latency with $t_{comm}=10$. In Figures (b) and (d), we use a higher ratio of communication vs. computation latency with $t_{comm}=100$. $t_{comp}=1$ in all figures.}}
     \label{fig.modelnet_commvsQgeneral}
\squeezeuppicture
\end{figure*}

\addresscomments{
Results for ModelNet40 are shown in Fig.~\ref{fig:modelnet_comm_vs_Q_10_K50} and 
Fig.~\ref{fig:modelnet_comm_vs_Q_100_K50} for $N=12$ silos and $K=10$ clients per silo. 
The Y-axis gives the top-$5$ test accuracy, 
and the X-axis gives the training latency in time units.
As with previous results, we see that $Q>1$ converges faster than $Q=1$.
We again observe that increasing the ratio between $t_{comm}$ and $t_{comp}$
results in larger $Q$ performing better than smaller $Q$.
}

\begin{figure*}[t]
\centering
\begin{subfigure}[b]{.48\linewidth}
      \centering
    \includegraphics[width=\linewidth]{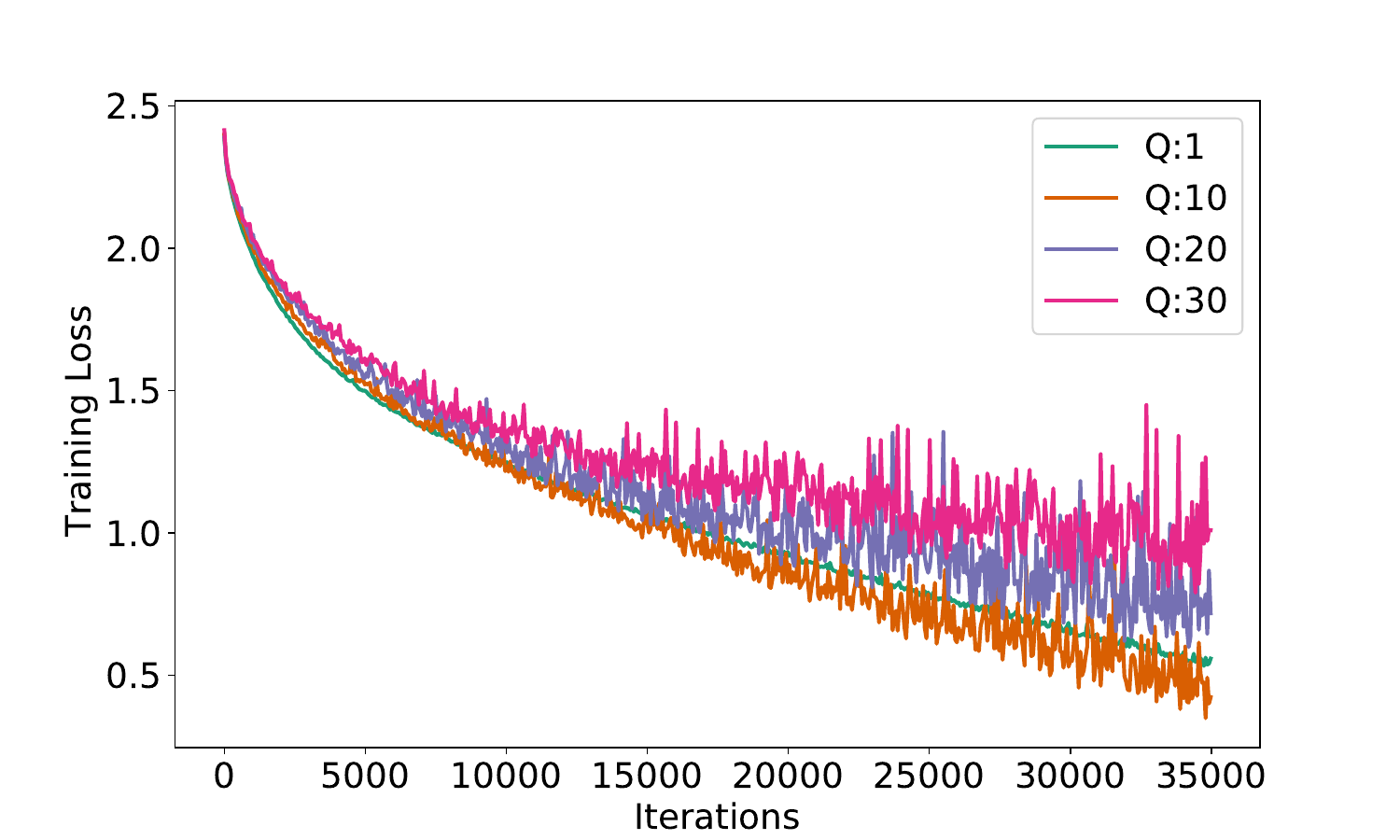}
    \caption{Variation with $Q$, N=2, K=50.}
    \label{fig:cifar10varK50_loss}
     \end{subfigure}
\begin{subfigure}[b]{.48\linewidth}
      \centering
    \includegraphics[width=\linewidth]{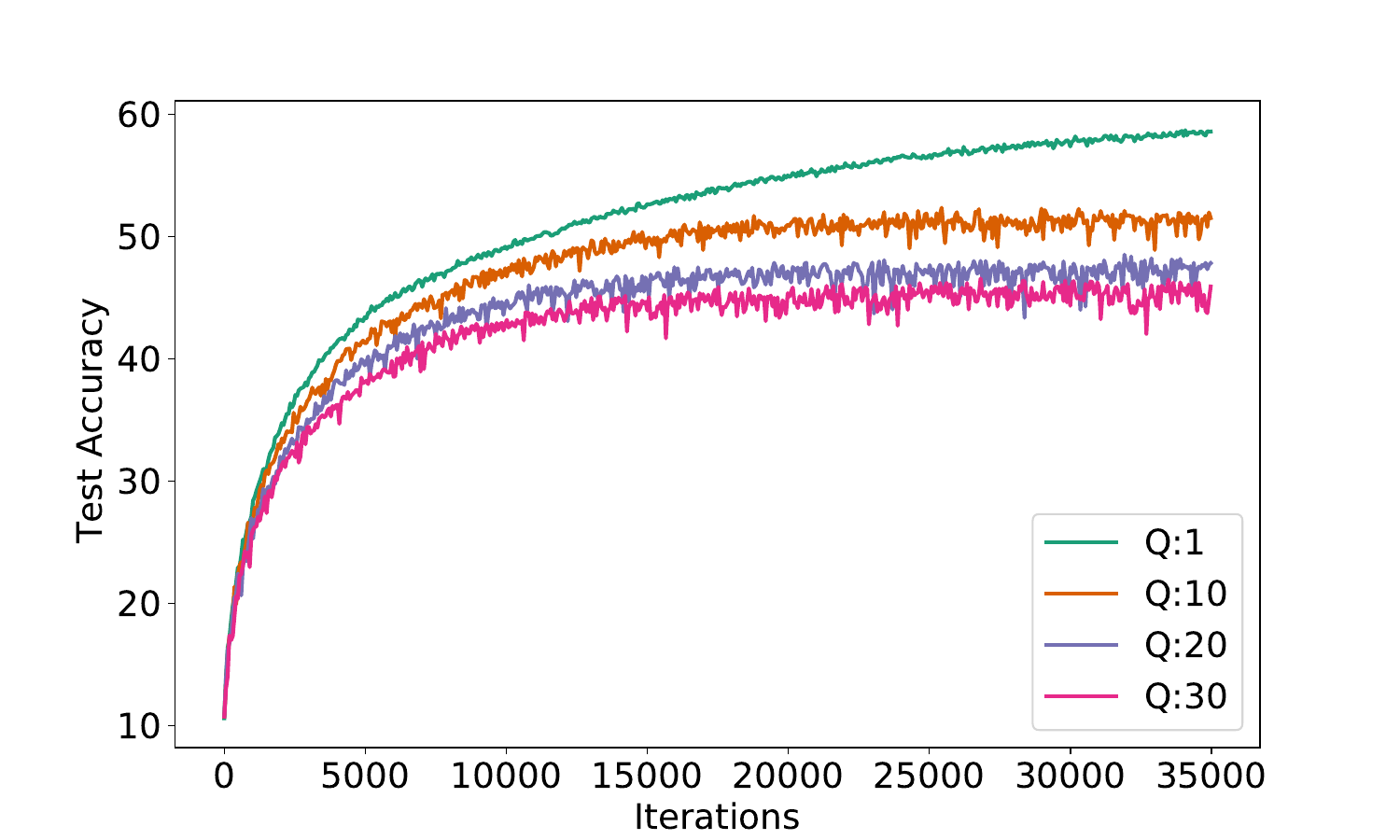}
    \caption{Variation with $Q$, N=2, K=50.}
    \label{fig:cifar10varK50_f1score}
     \end{subfigure}
\begin{subfigure}[b]{.48\linewidth}
      \centering
    \includegraphics[width=\linewidth]{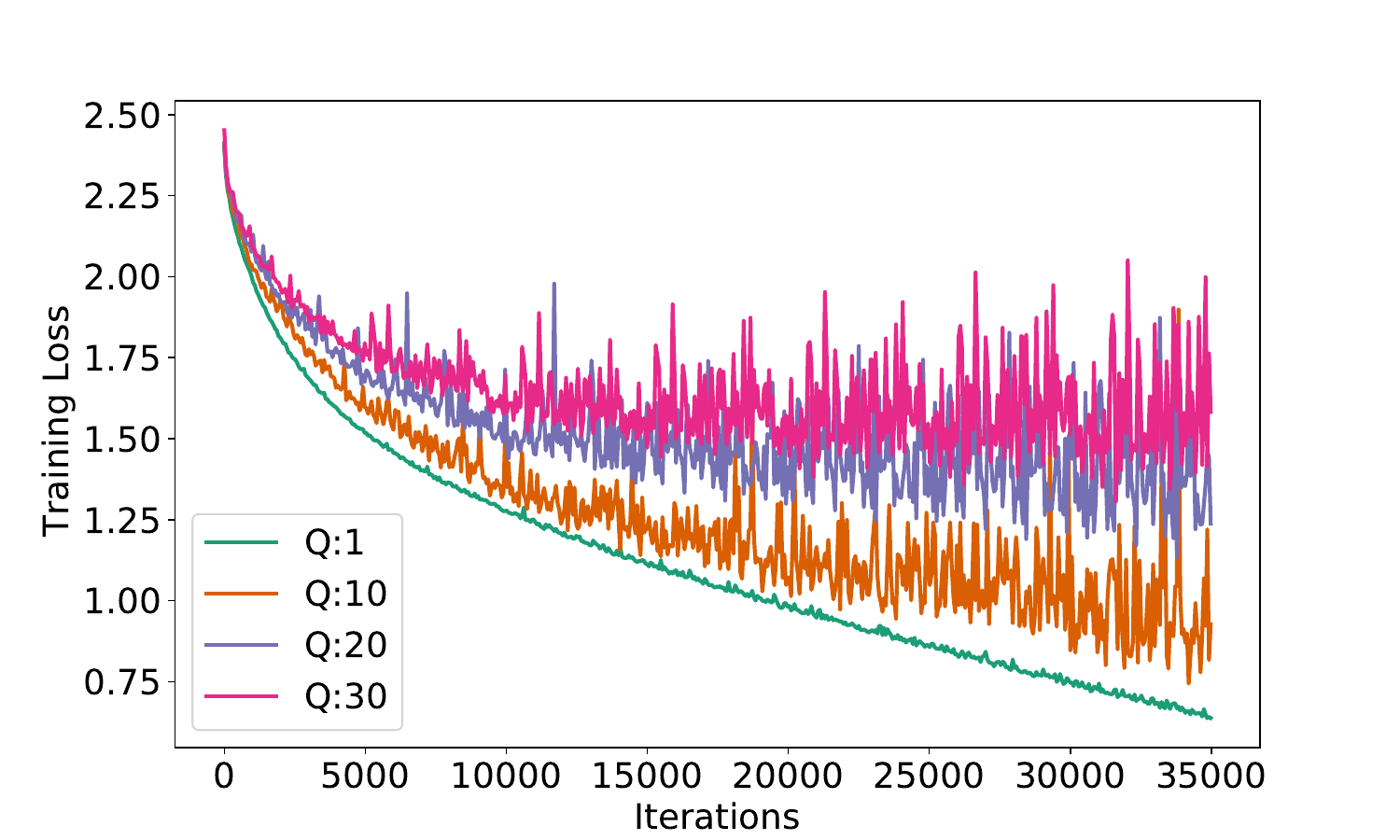}
    \caption{Variation with $Q$, N=2, K=100.}
    \label{fig:cifar10varK100_loss}
     \end{subfigure}
\begin{subfigure}[b]{.48\linewidth}
      \centering
    \includegraphics[width=\linewidth]{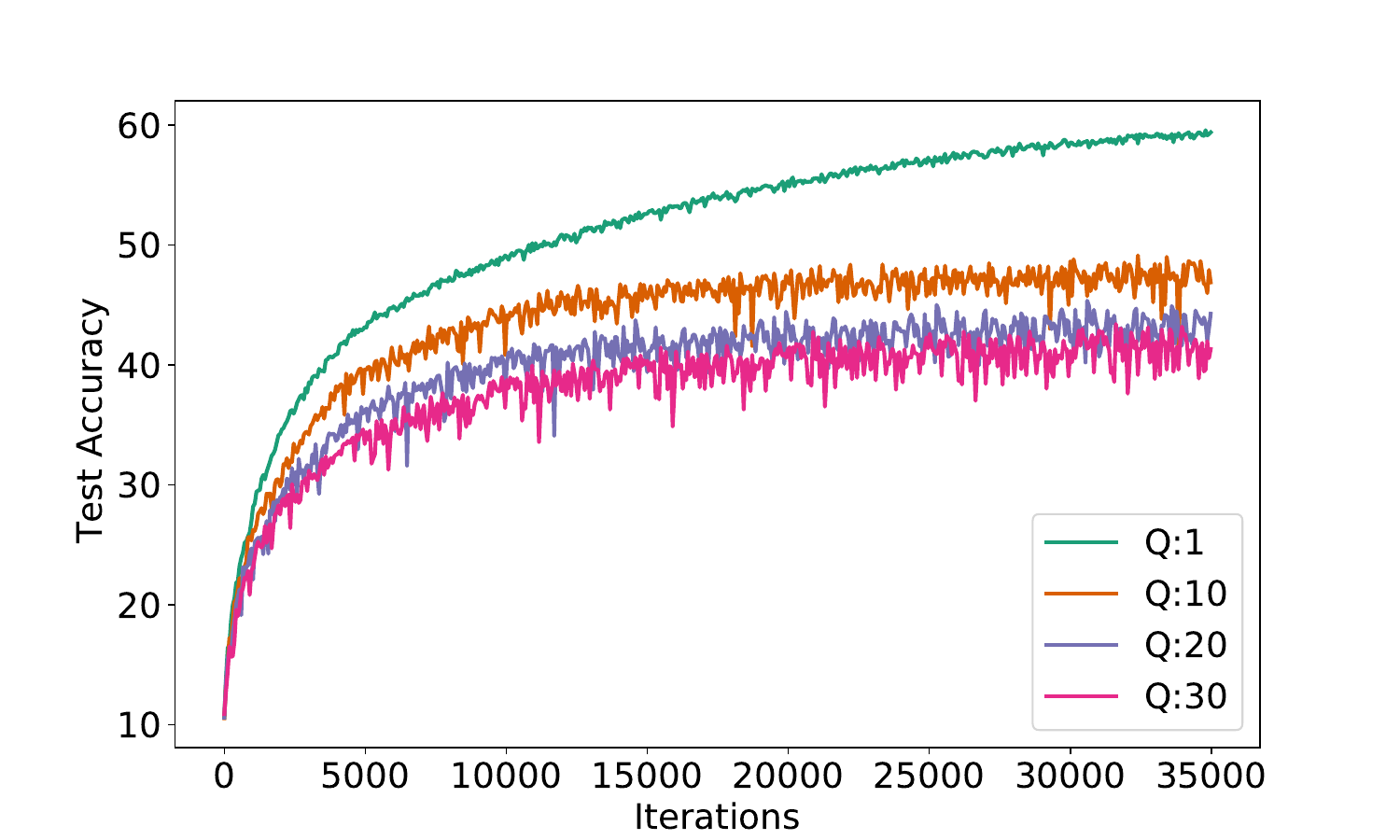}
    \caption{Variation with $Q$, N=2, K=100.}
    \label{fig:cifar10varK100_f1score}
     \end{subfigure}
\caption{Performance of TDCD on CIFAR-10 in terms of the number of training iterations for various values of $Q$. $N=2$ in all figures. $K=50$ in Figures (a) and (b), and $K=100$ in Figures (c) and (d). }
\label{fig:cifar10varQ}
\squeezeuppicture
\end{figure*}

For completeness, we also include figures for the same experiments that plot the training loss and test performance with respect to the number of local iterations. These results are shown in Fig.~\ref{fig:cifar10varK50_loss} and Fig.~\ref{fig:cifar10varK50_f1score} for CIFAR-10, Fig.~\ref{fig:mimic3varK20_loss} and Fig.~\ref{fig:mimic3varK20_f1score} for MIMIC-III, \addresscomments{and Fig.~\ref{fig:modelnetvarK50_loss} and Fig.~\ref{fig:modelnetvarK50_f1score} for ModelNet40.}
We observe for that CIFAR-10 with $K=50$ clients, in Fig.~\ref{fig:cifar10varK50_loss} 
with increasing $Q$, the convergence rate in terms of iterations is worse. The convergence error is higher as well. The results of model test performance in Fig.~\ref{fig:cifar10varK50_f1score}
also reflect a performance decay as $Q$ increases.
We observe similar trends for the convergence rate and test performance for MIMIC-III
with $K=20$ clients in Fig.~\ref{fig:mimic3varK20_loss} and Fig.~\ref{fig:mimic3varK20_f1score}
\addresscomments{and ModelNet40 with $10$ clients in Fig.~\ref{fig:modelnetvarK50_loss} and Fig.~\ref{fig:modelnetvarK50_f1score}.}

These observations are in accordance with Theorem~\ref{thm.main_theorem}.
However, as shown in Fig.~\ref{fig.cifar10_commvsQgeneral} and Fig.~\ref{fig.mimic3_commvsQgeneral}, there is definite benefit to larger $Q$ in terms of training latency.

\begin{figure*}[t]
\centering
\begin{subfigure}[b]{.48\linewidth}
      \centering
    \includegraphics[width=\linewidth]{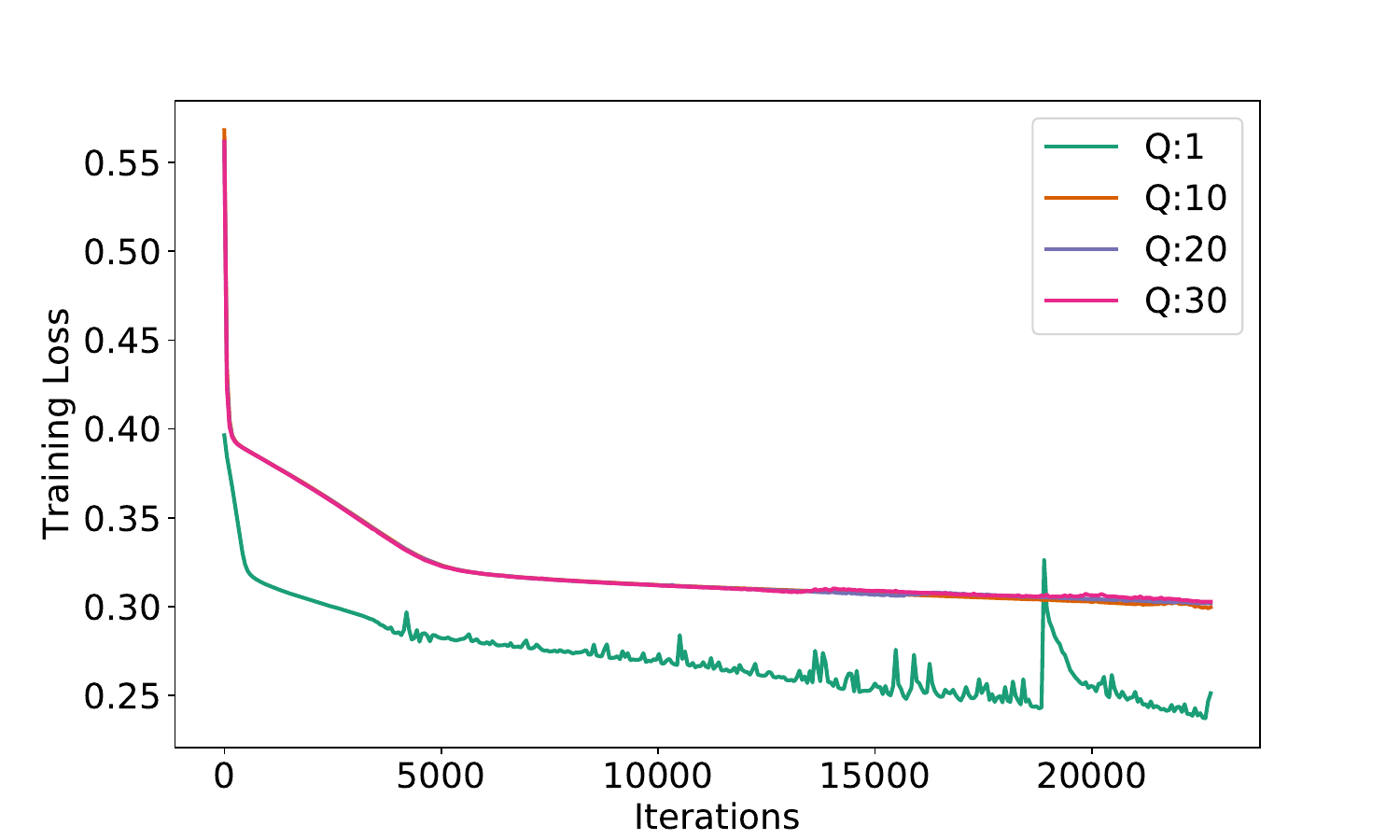}
    \caption{Variation with $Q$, N=4, K=20.}
    \label{fig:mimic3varK20_loss}
     \end{subfigure}
\begin{subfigure}[b]{.48\linewidth}
      \centering
    \includegraphics[width=\linewidth]{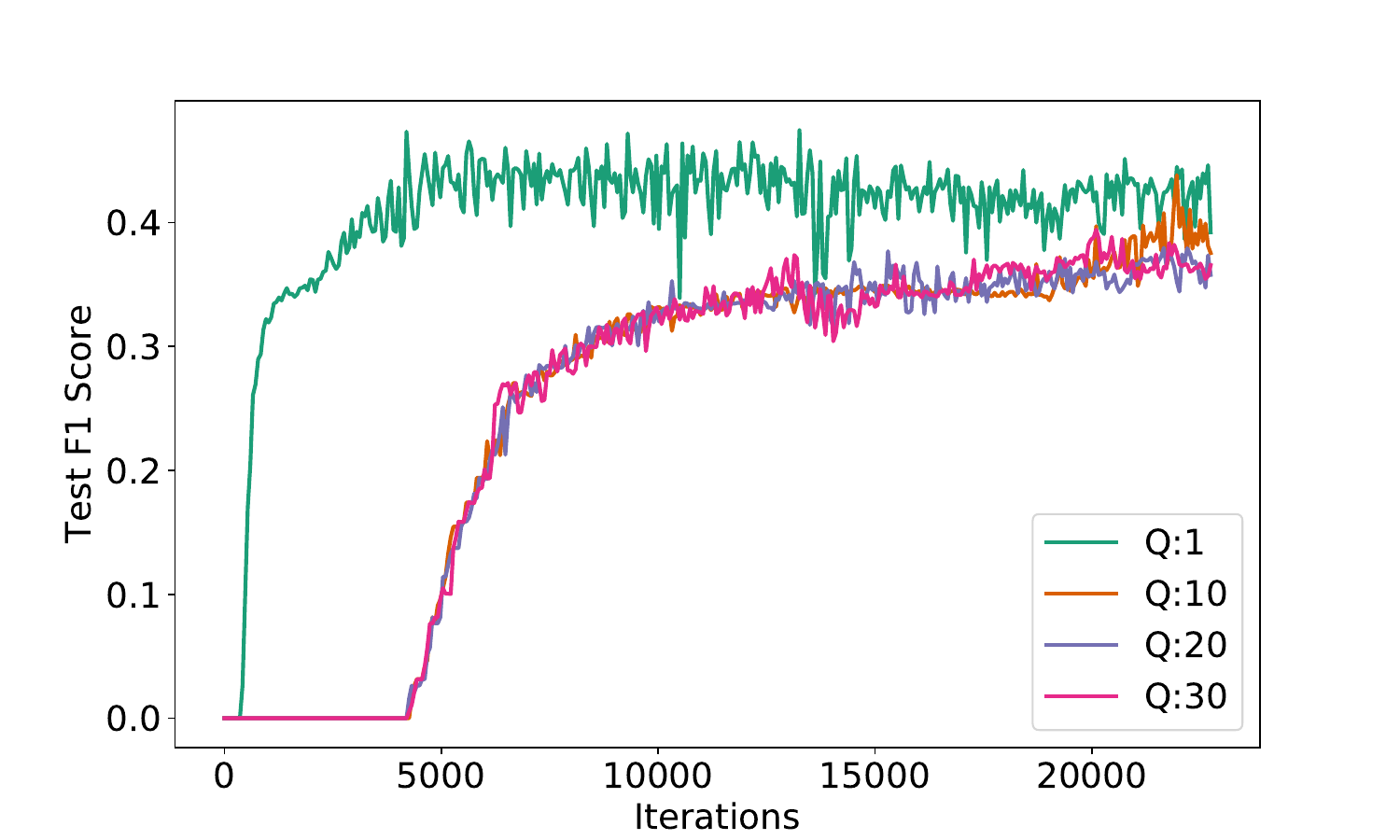}
    \caption{Variation with $Q$, N=4, K=20.}
    \label{fig:mimic3varK20_f1score}
     \end{subfigure}
\begin{subfigure}[b]{.48\linewidth}
      \centering
    \includegraphics[width=\linewidth]{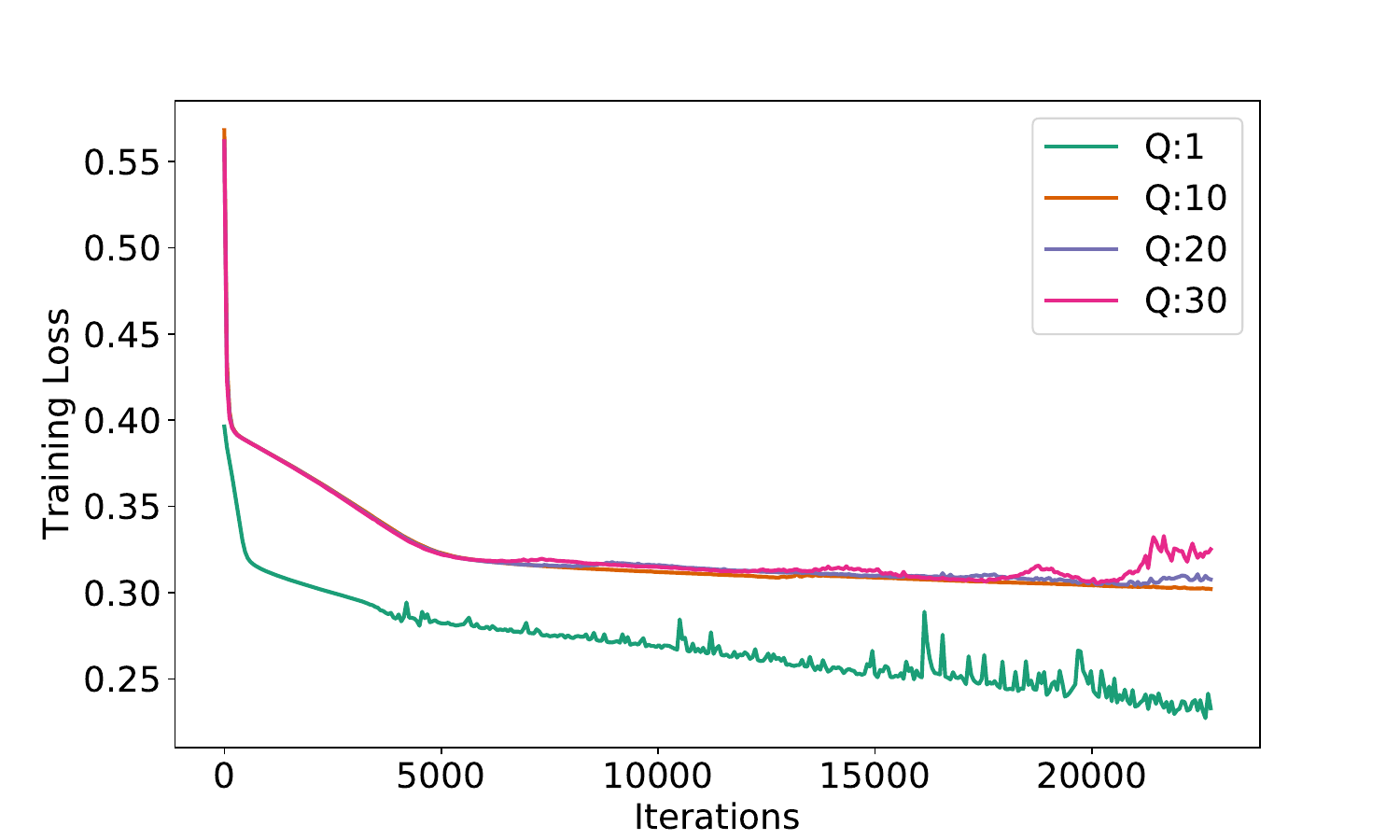}
    \caption{Variation with $Q$, N=4, K=50.}
    \label{fig:mimic3varK50_loss}
     \end{subfigure}
\begin{subfigure}[b]{.48\linewidth}
      \centering
    \includegraphics[width=\linewidth]{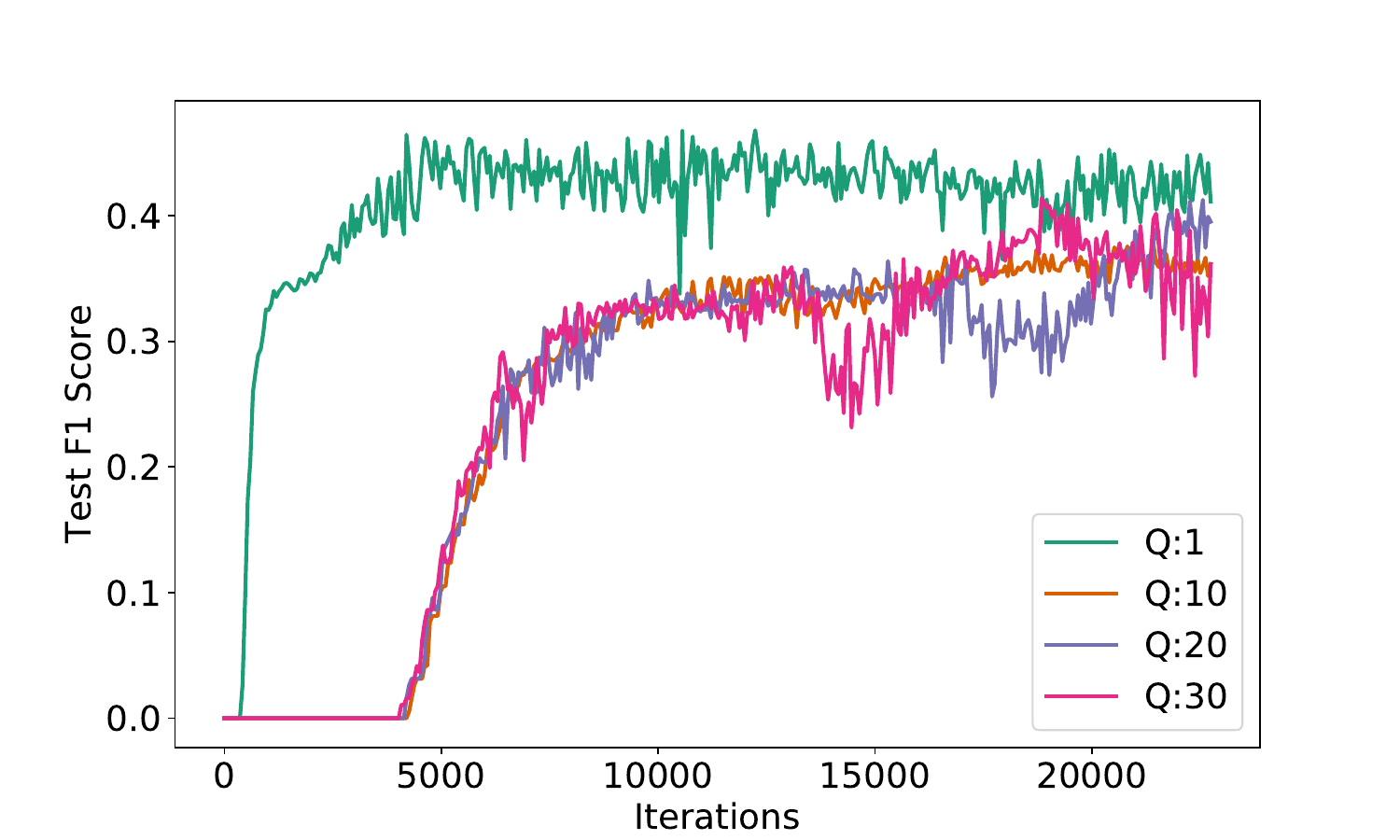}
    \caption{Variation with $Q$, N=4, K=50.}
    \label{fig:mimic3varK50_f1score}
     \end{subfigure}
\caption{Performance of TDCD on MIMIC-III in terms of the number of training iterations for various values of $Q$. $N=4$ in all figures. $K=20$ in Figures (a) and (b), and $K=50$ in Figures (c) and (d). }
\label{fig:mimic3varQ}
\squeezeuppicture
\end{figure*}

\begin{figure*}[t]
\centering
\begin{subfigure}[b]{.48\linewidth}
      \centering
    \includegraphics[width=\linewidth]{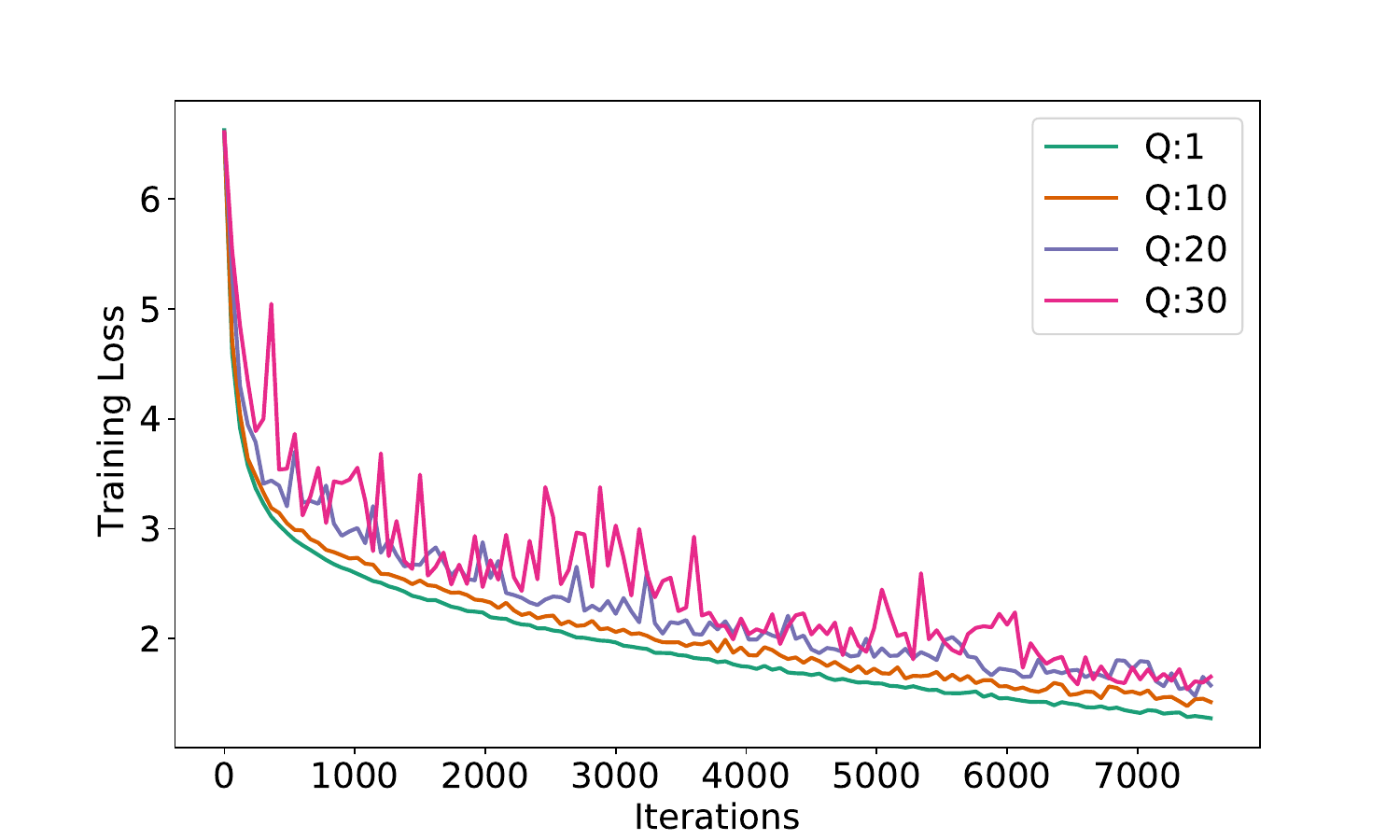}
    \caption{Variation with $Q$, N=12, K=10.}
    \label{fig:modelnetvarK50_loss}
     \end{subfigure}
\begin{subfigure}[b]{.48\linewidth}
      \centering
    \includegraphics[width=\linewidth]{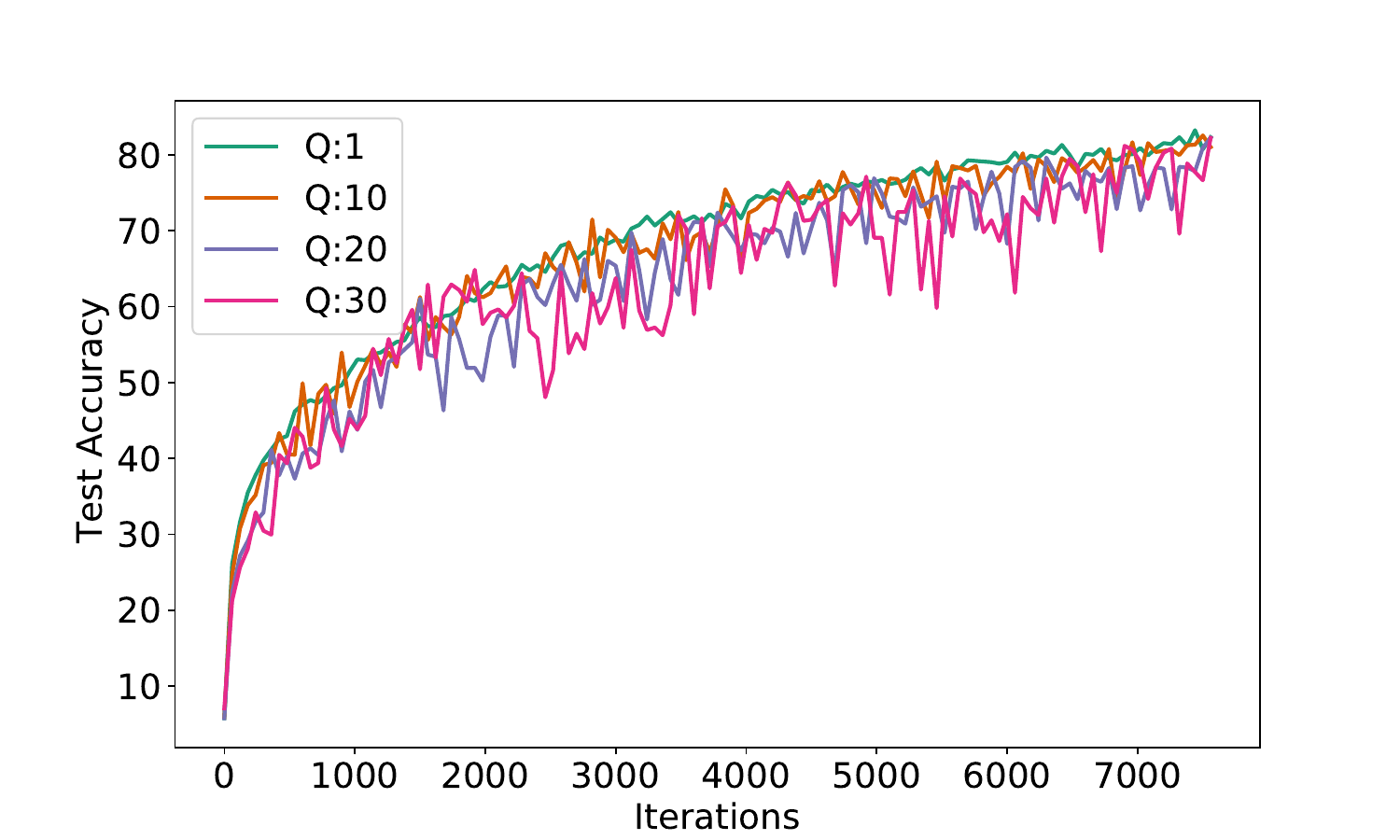}
    \caption{Variation with $Q$, N=12, K=10.}
    \label{fig:modelnetvarK50_f1score}
     \end{subfigure}
\begin{subfigure}[b]{.48\linewidth}
      \centering
    \includegraphics[width=\linewidth]{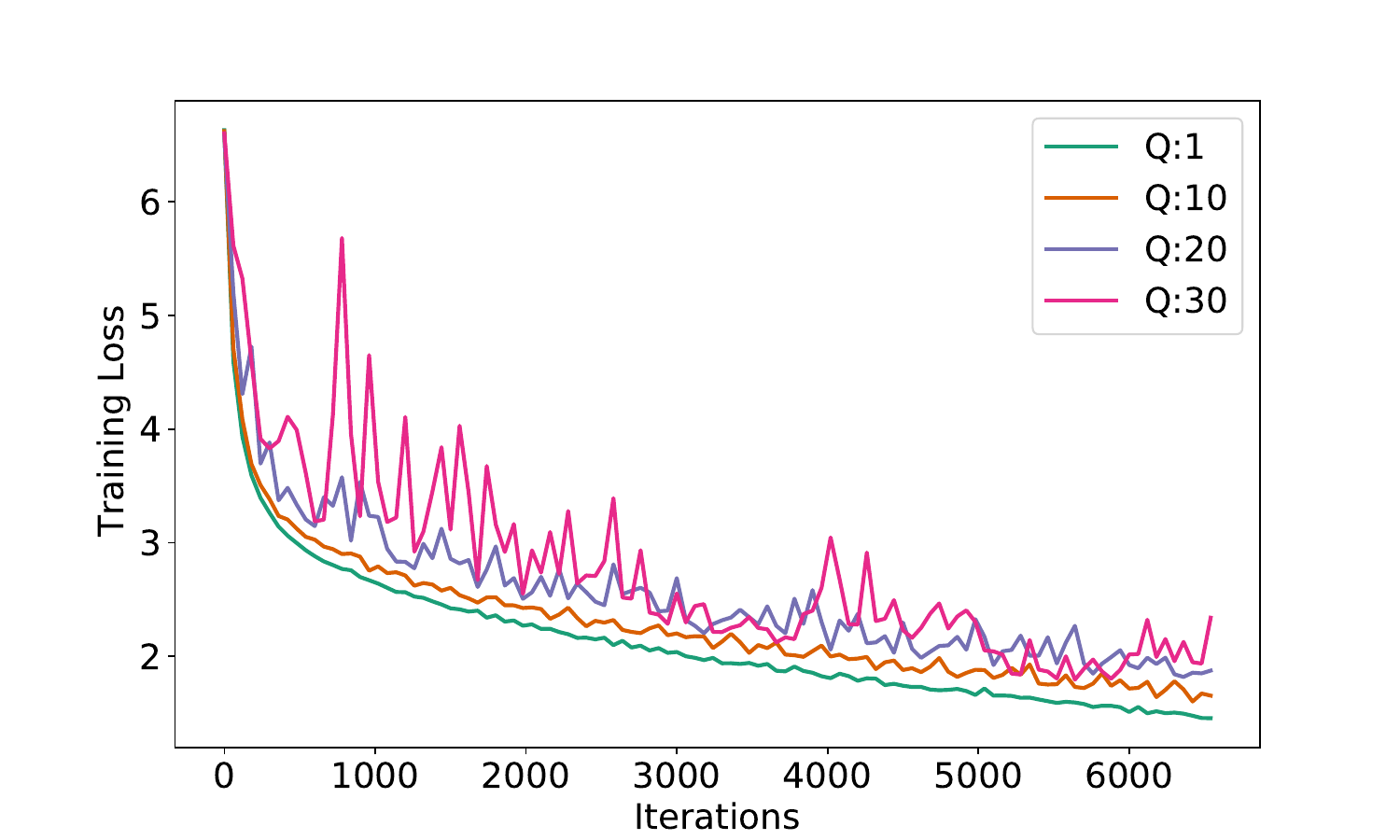}
    \caption{Variation with $Q$, N=12, K=20.}
    \label{fig:modelnetvarK100_loss}
     \end{subfigure}
\begin{subfigure}[b]{.48\linewidth}
      \centering
    \includegraphics[width=\linewidth]{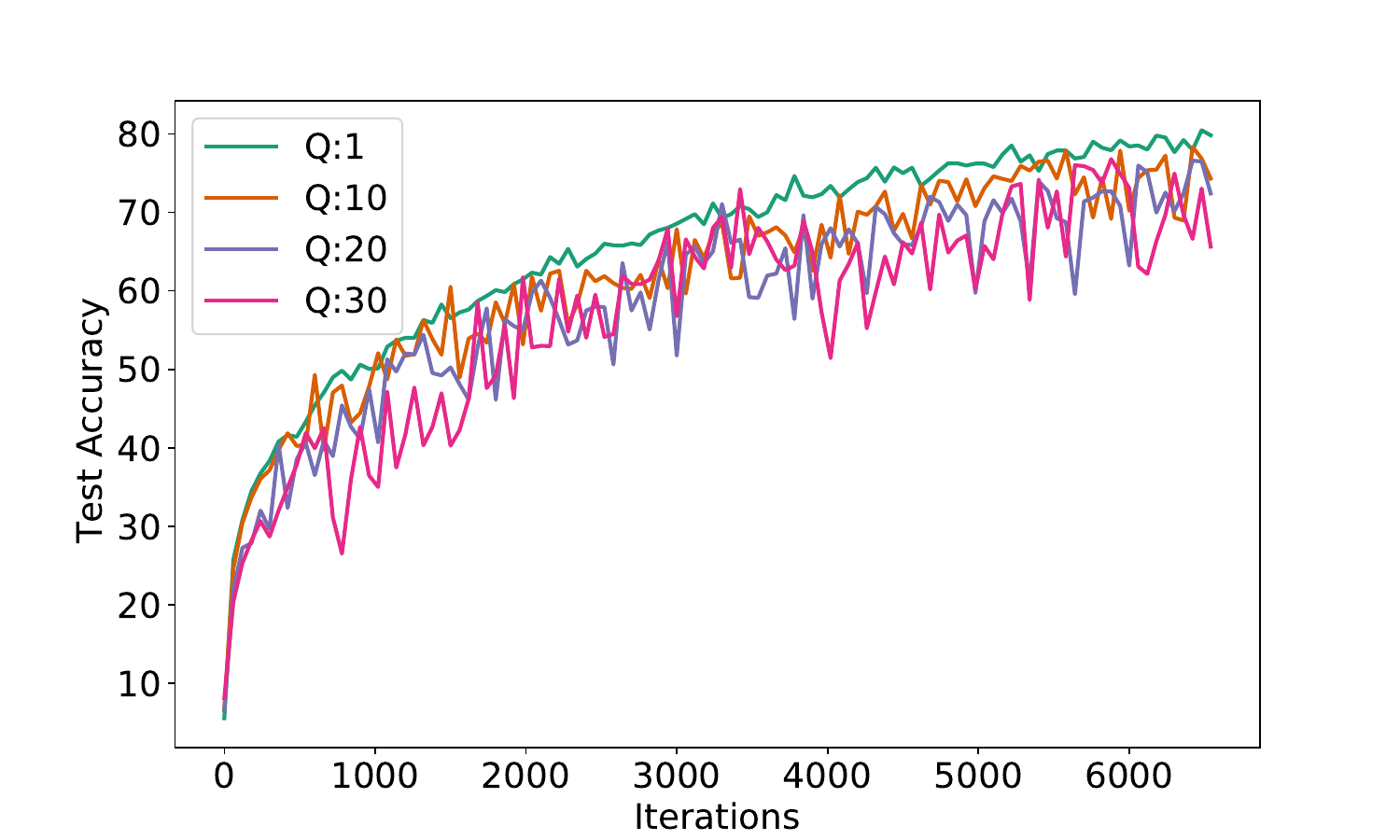}
    \caption{Variation with $Q$, N=12, K=20.}
    \label{fig:modelnetvarK100_f1score}
     \end{subfigure}
    \caption{\addresscomments{Training loss and top-$5$ accuracy of TDCD on ModelNet40 in terms of the number of training iterations for various values of $Q$. $N=12$ in all figures. $K=10$ in Figures (a) and (b), and $K=20$ in Figures (c) and (d). }}
\label{fig:cifar10varQ}
\squeezeuppicture
\end{figure*}

\subsubsection{Performance in Larger Networks}
Next, we repeat the experiments in the previous section using a larger number of clients per silo. This also results in larger number of horizontal data partitions in each silo. 

For CIFAR-10, we increase the number of clients in each silo from $K=50$ to $K=100$, resulting in total 200 participants with $N=2$ vertical partitions. We show the test performance results vs. the training latency in Fig.~\ref{fig:cifar10_comm_vs_Q_10_K100}, Fig.~\ref{fig:cifar10_comm_vs_Q_100_K100}. As in the previous section, the training latency improves as $Q$ increases. Comparing the results for $K=100$ with $K=50$, we observe that the convergence rate is similar, but the convergence error and corresponding test accuracy are slightly worse in the larger network. It is our intuition that this is because each client holds a smaller number of samples, and thus the variance of its stochastic partial derivatives is higher.

For MIMIC-III, we increase from $K=20$ to $K=50$ clients, resulting in 200 total clients, with $N=4$ vertical partitions.   
We show the test performance in terms of the F1 score vs. the training latency in 
Fig.~\ref{fig:mimic3_comm_vs_Q_10_K50} and Fig.~\ref{fig:mimic3_comm_vs_Q_100_K50} for MIMIC-III. 
The results are similar to those in Sec.~\ref{sec.latency_vs_convergence}, though we observe a smaller performance degradation with the larger $K$ than in the experiments with CIFAR-10.

\addresscomments{For ModelNet40, we increase from $K=10$ to $K=20$ clients, resulting in 240 total clients, 
with $N=12$ vertical partitions.   
We show the test performance in terms of the top-$5$ accuracy vs. the training latency in 
Fig.~\ref{fig:modelnet_comm_vs_Q_10_K50} and Fig.~\ref{fig:modelnet_comm_vs_Q_100_K50} for ModelNet40. 
As with the other datasets, test accuracy is slightly worse in the larger network.}

In Fig.~\ref{fig:cifar10varK100_loss} and Fig.~\ref{fig:cifar10varK100_f1score}, 
Fig.~\ref{fig:mimic3varK50_loss} and Fig.~\ref{fig:mimic3varK50_f1score}, 
and Fig.~\ref{fig:modelnetvarK50_loss} and Fig.~\ref{fig:modelnetvarK50_f1score},
we show the training loss and test performance as a function of the number of training iterations in these larger networks. In all cases, TDCD converges for all of the tested values of $Q$. However, we observe that for the same value of $Q$ across the figures, larger values of $K$ result in a larger convergence error in the training loss curves.
We again believe this is due to the increased variance of the stochastic partial derivatives that the clients compute during training.

\begin{figure*}[t]
\centering
\begin{subfigure}[b]{.48\linewidth}
      \centering%
    \includegraphics[width=\linewidth]{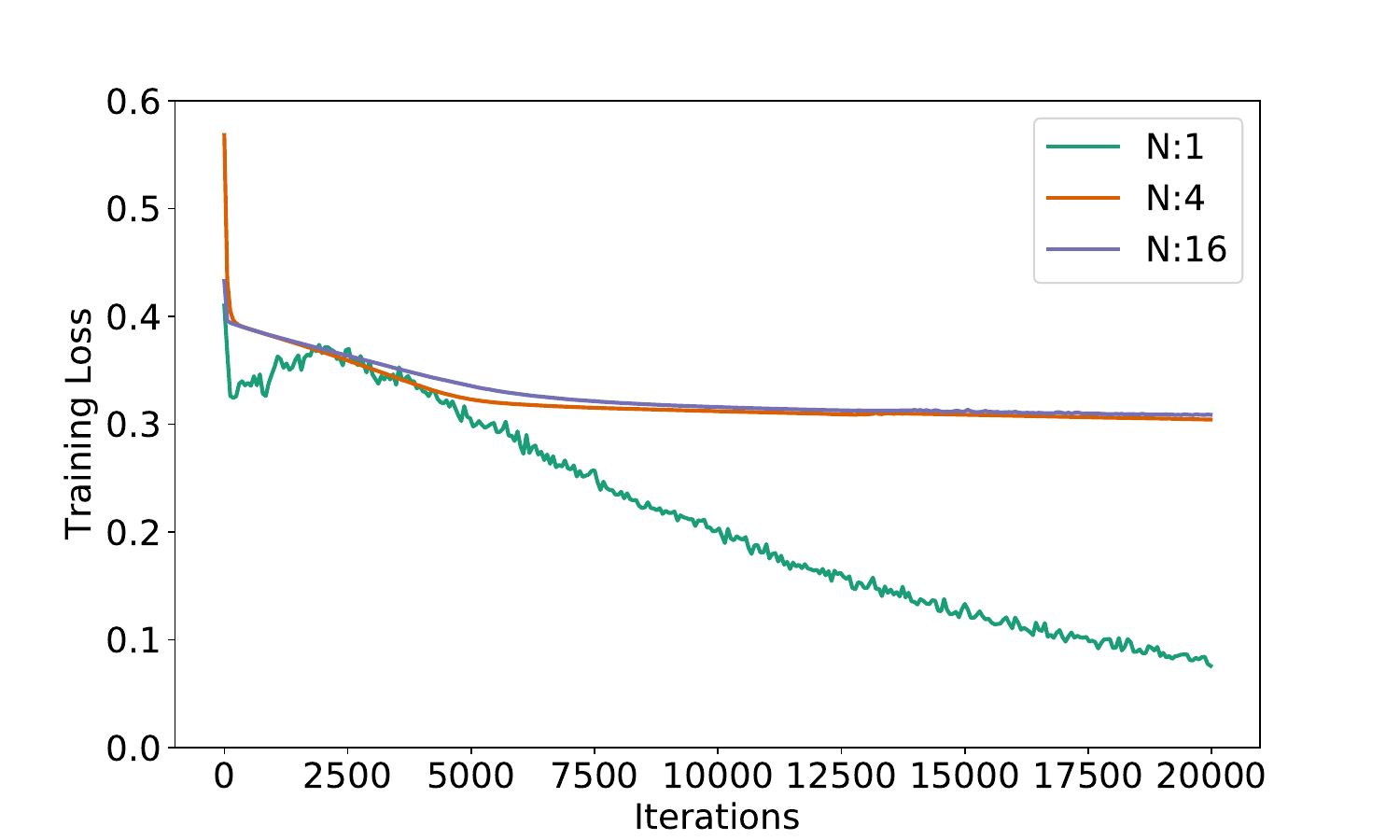}
    \caption{Variation of training loss with $N$.}
    \label{fig:mimic3highN}
     \end{subfigure}%
\begin{subfigure}[b]{.48\linewidth}
      \centering%
    \includegraphics[width=\linewidth]{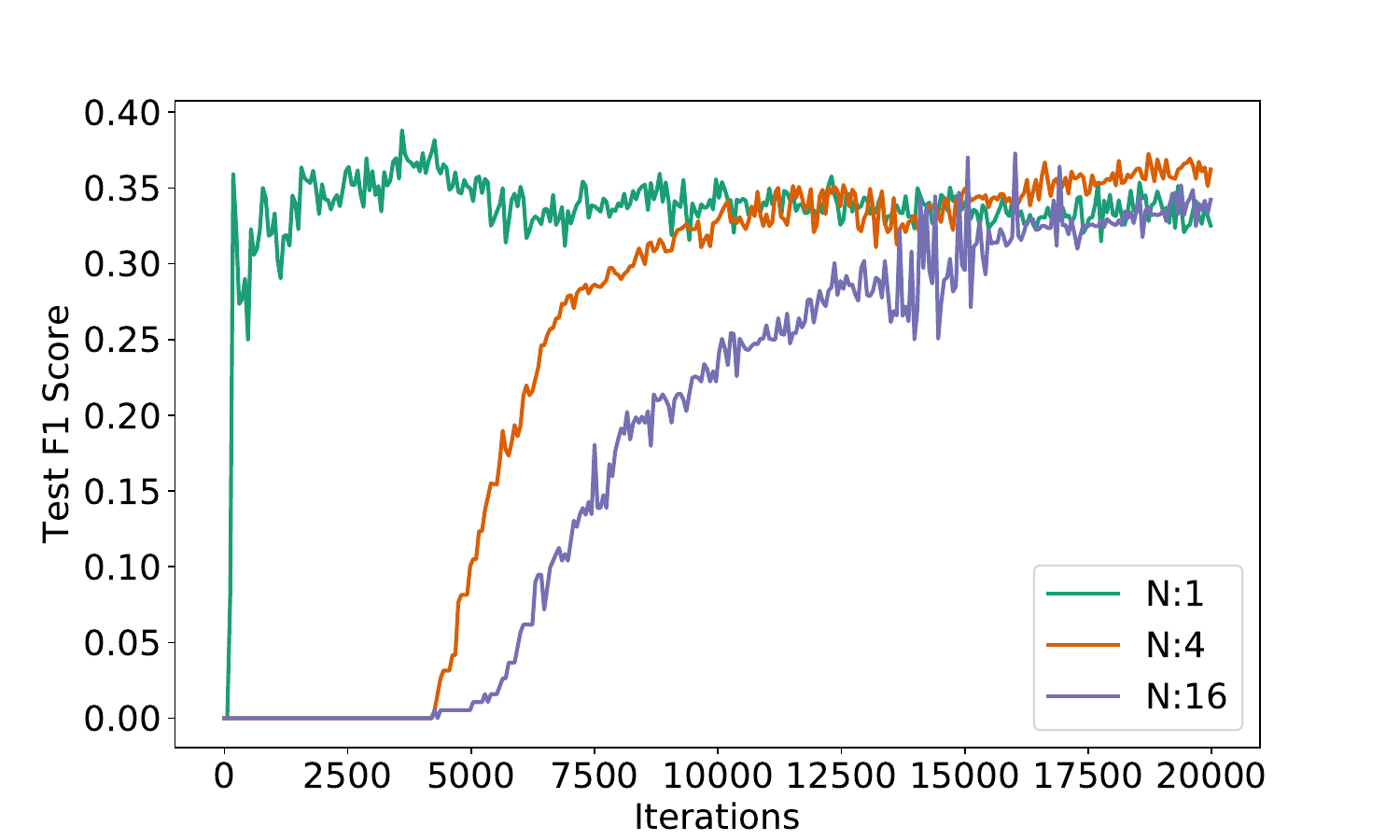}
    \caption{Variation with test F1 score with $N$.}
    \label{fig:mimic3highN_f1score}
     \end{subfigure}
\caption{Performance of TDCD on MIMIC III. Figure (a) shows training loss vs iterations $t$ and Figure (b) shows F1 Score vs. iterations $t$ for various values of $N$. We have $K=50$, $Q=10$.}
\label{fig.mimic3_var_N}
\squeezeuppicture
\end{figure*}

\subsubsection{Effect of the Number of Vertical Partitions $N$}
Finally, we study the impact of the number of vertical partitions $N$ on the convergence rate. We perform this experiment using the MIMIC-III dataset. 
We show the results in Fig.~\ref{fig:mimic3highN} and Fig.~\ref{fig:mimic3highN_f1score} for $Q=10$ and $K=50$. We plot the number of iterations on the X-axis. Since we fix the value of $Q$, the same trend will be observed for both the number of iterations and the training latency.
It is to be noted that varying $N$ introduces other forms of dissimilarity in this experiment as the neural network architectures are different.
We observe that the convergence error is worse for higher values of $N$. We also observe that lower values of $N$ reach comparable F1 scores faster than higher values, in terms of local iterations.
This is as expected from Theorem~\ref{thm.main_theorem}.

%%%%%%%%%%%%%%%%%%%%%%%%%%%%%%%%%%%%%%%%%%%%%%%%%%%%%%%%%%%%%%%%%%%%%%%%%%%%%%%%
%%%%%%%%%%%%%%%%%%%%%%%%%%  CONCLUSION    %%%%%%%%%%%%%%%%%%%%%%%%%%%%%%%%%%%%%%
%%%%%%%%%%%%%%%%%%%%%%%%%%%%%%%%%%%%%%%%%%%%%%%%%%%%%%%%%%%%%%%%%%%%%%%%%%%%%%%%
\section{CONCLUSION}
\label{sec:conclusion}
We have introduced TDCD, a communication efficient decentralized algorithm for a multi-tier network model with both horizontally and vertically partitioned data. We have provided a theoretical analysis of the algorithm convergence and its dependence on the number of vertical partitions and the number of local iterations. Finally, we have presented experimental results that validate our theory in practice. 
In future work, we plan to explore the possibility of hubs communicating with each other asynchronously to share information.


%\appendix
%%%%%%%%%%%%%%%%%%%%%%%%%%%%%%%%%%%%%%%%%%%%%%%%%%%%%%%%%%%%%%%%%%%%%%%%%%%%%%%%
%%%%%%%%%%%%%%%%%%%%%%%%%%  Appendices    %%%%%%%%%%%%%%%%%%%%%%%%%%%%%%%%%%%%%%
%%%%%%%%%%%%%%%%%%%%%%%%%%%%%%%%%%%%%%%%%%%%%%%%%%%%%%%%%%%%%%%%%%%%%%%%%%%%%%%%
\appendix
%\input{acm_supplementary_rev2}

%%%%%%%%%%%%%%%%%%%%%%%%%%%%%%%%%%%%%%%%%%%%%%%%%%%%%%%%%%%%%%%%%%%%%%%%%%%%%%%%
%%%%%%%%%%%%%%%%%%%%%%%%%%  Appendices    %%%%%%%%%%%%%%%%%%%%%%%%%%%%%%%%%%%%%%
%%%%%%%%%%%%%%%%%%%%%%%%%%%%%%%%%%%%%%%%%%%%%%%%%%%%%%%%%%%%%%%%%%%%%%%%%%%%%%%%

\section{Proof of the Theorem and Supporting Lemmas} \label{AppendixOne}
In this section we provide the proofs of our theorem and the associated helping lemmas.

To facilitate the analysis, we first define an auxiliary local vector that represents the local view of the global model at each client. 
Let $y_{k,j}^t$ denote the auxiliary weight vector used to calculate the partial derivative $g_{k,j}$, 
\begin{align}
\ykj^t = [(\tjminus)^{\intercal}, (\tjk^t)^{\intercal}]^{\intercal}
\end{align}
where, $\tjminus$ denotes the vector of all coordinates of $\gm$ excluding block $j$ at iteration $t_0 \leq t$. $t_0$ is the most recent iteration when the client $k$ last updated the value of $\tjminus$ from its hub. $(\boldsymbol{\cdot})^\intercal$ represents transpose operation.
The partial derivative $\gkj$ at iteration $t$ is a function of $\ykj^t$ and the minibatch selected at $t_0$. With slight abuse of notation, we let:
\begin{align}
\gkj (\ykj^t) \eqdef  g_{k,j}(\Phi_{-k,j}^{\zeta^{t_0}},\Phi_{k,j}^{\zeta^{t_0}}; \zeta^{t_0}_{k,j}). 
\end{align}

We recall the  definition of the local stochastic partial derivative $g_{k,j}^t$ from (\ref{eqn.partial_derivative_def}). $g_{k,j}^t$ at iteration $t$ is a function of $\ykj^t$ and the minibatch selected at $t_0$. Our analysis looks at the evolution of the $\ykj^t~~,~~ t_0\leq t \leq t_0+Q$, i.e. the interval between the communication round $t_0$ and communication round $t_0+Q$.

We note that TDCD reuses each mini-batch for $Q$ local iterations between each communication round. This reuse implies that the local stochastic partial derivatives are not unbiased during the local iterations $t_0 +1 \leq t \leq t_0 +Q -1$. Here, $t_0$ denotes the last iteration before $t$ when clients synced with their hubs.
Using the conditional expectation on Assumption~\ref{assumption.unbiased} for the gradients calculated at iteration $t_0$, the following is satisfied at $t_0$:
\begin{align}
    &\ex_{\zeta^{t_0}}\left[\gkj(\ykj^{t_0}) \; \big| \{\ykj^r \}_{r=0}^{t_0} \right] =\pdj \Lc(\ykj^{t_0})
\\
    &\ex_{\zeta^{t_0}} \left[\| \gkj(\ykj^{t_0}) - \pdj \Lc(\ykj^{t_0}) \|^2  \;  \big| \{\ykj^r \}_{r=0}^{t_0}\right] \leq \sigma^2_j
\\
    &\ex_{\zeta^{t_0}} [\| \gkj(\ykj^{t_0}) \|^2 \;  \big| \{\ykj^r \}_{r=0}^{t_0}] \leq \sigma^2_j + \ex_{\zeta^{t_0}} \| \nabla_{(j)} \Lc(\ykj^{t_0}) \|^2. \label{eqn.simplified_variance_t0}
\end{align}
where, the expectation is over the mini-batch selected at iteration $t_0$, conditioned on the algorithm iterates $\ykj^r$ up until $t_0$.

For simplicity, we define $\condY \eqdef \{\ykj^r\}_{r=0}^{t_0}$ and ease the notation of the conditional expectation as the following:
\begin{align}
\ex_{\zeta^{t_0}} \left[\; \boldsymbol{\cdot} \; \fullconditional \right] \eqdef \ex^{t_0}.
\end{align}

Lastly, we denote the next local iteration when the clients will exchange intermediate information with their respective hub by 
\[
t_0^+ = t_0+Q-1.
\]

%%%%%%%%%%%%%%%%%%%%%%CONVERGENCE ANALYSIS%%%%%%%%%%%%%%%%%%%%%%%%%%%
\subsection{Proof of Supporting Lemmas}
We recall that we can write the evolution of the global model as:
\begin{align}
\gm^{t+1} = \gm^t - \eta \bm{G}^t. \label{eqn.original}
\end{align}
Here, we  study the evolution of $\gm$ at each iteration. 

We next give several  lemmas to simplify our analysis. 

We first define a lemma to relate the expected local and expected global gradient.
\begin{lemmavfl} \label{lemma.expected_local_global_gradient}
    \begin{align}
    \BB{E} \left\Vert g_{k,j}(\ykj^{t_0}) \right\Vert^2  \leq \sigma_j^2 + \BB{E} \left\Vert \nabla_{(j)} \BC{L}(\ykj^{t_0}) \right\Vert^2 
    \end{align}
    where $\ex$ is the total expectation.
\end{lemmavfl}
\begin{proof}
We observe that at iteration $t$, where the latest communication between clients and hubs took place at iteration $t_0\leq t$:
\begin{align}
\sigma_j^2 &\geq \BB{E} \left\Vert g_{k,j}(\ykj^{t_0}) - \nabla_{(j)} \BC{L}(\ykj^{t_0})) \right\Vert^2 
\\ 
&= \BB{E} \left[ \ex^{t_0} \left[ \left\Vert g_{k,j}(\ykj^{t_0}) - \nabla_{(j)} \BC{L}(\ykj^{t_0})) \right\Vert^2 \right] \right ] 
\\
& = \BB{E} \left[\ex^{t_0} \left[\left\Vert g_{k,j}(\ykj^{t_0}) - \ex^{t_0} \left[ g_{k,j}(\ykj^{t_0}) \right] \right\Vert^2  \right] \right] 
\\
& = \BB{E} \left[\ex^{t_0} \left[ \left\Vert g_{k,j}(\ykj^{t_0}) \right\Vert ^2  - \left\Vert \ex^{t_0} \left[ g_{k,j}(\ykj^{t_0})\right]\right\Vert^2 \right]\right] 
\\ 
&\overset{(a)}{=} \BB{E} \left[\ex^{t_0} \left\Vert g_{k,j}(\ykj^{t_0}) \right\Vert ^2 \right] - \BB{E} \left[\left\Vert \ex^{t_0} [g_{k,j}(\ykj^{t_0}) ] \right\Vert^2 \right]  
\\
& = \BB{E} \left\Vert g_{k,j}(\ykj^{t_0}) \right\Vert^2  - \BB{E} \left\Vert \nabla_{(j)} \BC{L}(\ykj^{t_0}) \right\Vert^2. \label{eqn.expected_local_global_gradient_end}
\end{align}

Here the equality in (a) follows from (\ref{assumption.3a}) in  Assumption~\ref{assumption.unbiased}. We can obtain Lemma~\ref{lemma.expected_local_global_gradient} by rearranging~(\ref{eqn.expected_local_global_gradient_end}).
\end{proof}

\begin{lemmavfl} \label{lemma.new_lemma}
Under assumptions on the learning rate $\eta \leq \frac{1}{2Q L_{max}}$,
\begin{align}
\sum \limits_{t=t_0}^{t_0^+} \sum\limits_{j=1}^N \frac{1}{K_j} \sum \limits_{k=1}^{K_j} \ex^{t_0} \| \gkj(\ykj^t) - \gkj(\ykj^{t_0}) \|^2 \leq 4Q^2(1+ Q) L^{ 2}_{max} \eta^2 \sum\limits_{j=1}^N \frac{1}{K_j} \sum \limits_{k=1}^{K_j} \ex^{t_0} \| \gkj(\ykj^{t_0})\|^2.
\end{align}

\begin{proof}
We start with simplifying the L.H.S. expression:
\begin{align}
&\sum\limits_{j=1}^N \frac{1}{K_j} \sum \limits_{k=1}^{K_j} \ex^{t_0} \| \gkj(\ykj^t) - \gkj(\ykj^{t_0}) \|^2
\\
&= \sum\limits_{j=1}^N \frac{1}{K_j} \sum \limits_{k=1}^{K_j} \ex^{t_0} \| \gkj(\ykj^t) + \gkj(\ykj^{t-1}) - \gkj(\ykj^{t-1}) - \gkj(\ykj^{t_0}) \|^2
\\
& \leq \sum\limits_{j=1}^N \frac{1}{K_j} \sum \limits_{k=1}^{K_j} \ex^{t_0} \left[ \left(1+ n\right)\| \gkj(\ykj^t) - \gkj(\ykj^{t-1})\|^2 + \left(1 +  \frac{1}{n}\right) \|\gkj(\ykj^{t-1}) - \gkj(\ykj^{t_0}) \|^2 \right ]
\\
& \overset{(a)}{\leq} \sum\limits_{j=1}^N \frac{1}{K_j} \sum \limits_{k=1}^{K_j} \ex^{t_0} \left[ \left(1+ n\right) L^{ 2}_{max}\| \ykj^t - \ykj^{t-1}\|^2 + \left(1 +  \frac{1}{n}\right) \|\gkj(\ykj^{t-1}) - \gkj(\ykj^{t_0}) \|^2 \right ]
\\
& \leq \sum\limits_{j=1}^N \frac{1}{K_j} \sum \limits_{k=1}^{K_j} \ex^{t_0} \left[ \left(1+ n\right) L^{ 2}_{max} \eta^2\| \gkj(\ykj^{t-1})\|^2 + \left(1 +  \frac{1}{n}\right) \|\gkj(\ykj^{t-1}) - \gkj(\ykj^{t_0}) \|^2 \right ]
\\
& \leq \sum\limits_{j=1}^N \frac{1}{K_j} \sum \limits_{k=1}^{K_j} \ex^{t_0} \left[ \left(1+ n\right) L^{ 2}_{max} \eta^2\| \gkj(\ykj^{t-1}) - \gkj(\ykj^{t_0}) + \gkj(\ykj^{t_0})\|^2 \right ] \nonumber
\\
&~~~~~+ \leq \sum\limits_{j=1}^N \frac{1}{K_j} \sum \limits_{k=1}^{K_j} \ex^{t_0} \left(1 +  \frac{1}{n}\right) \|\gkj(\ykj^{t-1}) - \gkj(\ykj^{t_0}) \|^2
\\
& \leq \sum\limits_{j=1}^N \frac{1}{K_j} \sum \limits_{k=1}^{K_j} \ex^{t_0} \left[ 2\left(1+ n\right) L^{ 2}_{max} \eta^2\| \gkj(\ykj^{t-1}) - \gkj(\ykj^{t_0})\|^2 + \left(1 +  \frac{1}{n}\right) \|\gkj(\ykj^{t-1}) - \gkj(\ykj^{t_0}) \|^2\right ] \nonumber
\\
&~~~~~+ 2\left(1+ n\right) L^{ 2}_{max} \eta^2 \sum\limits_{j=1}^N \frac{1}{K_j} \sum \limits_{k=1}^{K_j} \ex^{t_0} \| \gkj(\ykj^{t_0})\|^2
\\
& = \sum\limits_{j=1}^N \frac{1}{K_j} \sum \limits_{k=1}^{K_j} \left ( 2\left(1+ n\right) L^{ 2}_{max} \eta^2 + \left(1 +  \frac{1}{n}\right) \right) \ex^{t_0}\| \gkj(\ykj^{t-1}) - \gkj(\ykj^{t_0})\|^2 \nonumber
\\
&~~~~~ + 2\left(1+ n\right) L^{ 2}_{max} \eta^2 \sum\limits_{j=1}^N \frac{1}{K_j} \sum \limits_{k=1}^{K_j}\ex^{t_0} \| \gkj(\ykj^{t_0})\|^2. \label{eqn.new_method_8}
\end{align}
where (a) follows from Assumption~\ref{assumption.lipschitz}.
\\
We are now interested in forming a recurrence relation.
\\
On setting $n = Q$ and $\eta \leq \frac{1}{2Q L_{max}}$ in the first term of (\ref{eqn.new_method_8}), we get,
\begin{align}
&  \sum\limits_{j=1}^N \frac{1}{K_j} \sum \limits_{k=1}^{K_j} \ex^{t_0}\| \gkj(\ykj^t) - \gkj (\ykj^{t_0})\|^2
\\
& \leq\sum\limits_{j=1}^N \frac{1}{K_j} \sum \limits_{k=1}^{K_j} \left ( 2(1+ Q) L^{ 2}_{max} \eta^2 + (1 +  \frac{1}{Q}) \right) \ex^{t_0} \| \gkj(\ykj^{t-1}) - \gkj(\ykj^{t_0})\|^2 \nonumber
\\
&~~~~~ + 2(1+ Q) L^{ 2}_{max} \eta^2 \sum\limits_{j=1}^N \ex^{t_0} \| \gkj(\ykj^{t_0})\|^2
\\
& \leq\sum\limits_{j=1}^N \frac{1}{K_j} \sum \limits_{k=1}^{K_j} \left ( \frac{1+Q}{2Q^2} + 1 + \frac{1}{Q} \right) \ex^{t_0} \| \gkj(\ykj^{t-1}) - \gkj(\ykj^{t_0})\|^2 \nonumber
\\
&~~~~~ + 2(1+ Q) L^{ 2}_{max} \eta^2 \sum\limits_{j=1}^N \frac{1}{K_j} \sum \limits_{k=1}^{K_j} \ex^{t_0} \| \gkj(\ykj^{t_0})\|^2
\\
& \leq\sum\limits_{j=1}^N \frac{1}{K_j} \sum \limits_{k=1}^{K_j} \left ( \frac{2Q}{2Q^2} + 1 + \frac{1}{Q} \right) \ex^{t_0} \| \gkj(\ykj^{t-1}) - \gkj(\ykj^{t_0})\|^2 \nonumber
\\
&~~~~~ + 2(1+ Q)  L^{ 2}_{max} \eta^2 \sum\limits_{j=1}^N \frac{1}{K_j} \sum \limits_{k=1}^{K_j}\ex^{t_0} \| \gkj(\ykj^{t_0})\|^2
\\
& \leq\sum\limits_{j=1}^N \frac{1}{K_j} \sum \limits_{k=1}^{K_j} \left (1 + \frac{2}{Q} \right) \ex^{t_0} \| \gkj(\ykj^{t-1}) - \gkj(\ykj^{t_0})\|^2 \nonumber
\\
&~~~~~ + 2(1+ Q) L^{ 2}_{max} \eta^2  \sum\limits_{j=1}^N \frac{1}{K_j} \sum \limits_{k=1}^{K_j} \ex^{t_0} \| \gkj(\ykj^{t_0})\|^2.
\end{align}

We now have a recurrence relation:
\begin{align}
&\underbrace{\sum\limits_{j=1}^N \frac{1}{K_j} \sum \limits_{k=1}^{K_j} \ex^{t_0}\| \gkj(\ykj^t) - \gkj (\ykj^{t_0})\|^2}_{\Psi^t} \nonumber
\\
&\leq \underbrace{\left (1 + \frac{2}{Q} \right)}_{A} \underbrace{\sum\limits_{j=1}^N \frac{1}{K_j} \sum \limits_{k=1}^{K_j} \ex^{t_0} \| \gkj(\ykj^{t-1}) - \gkj(\ykj^{t_0})\|^2}_{\Psi^{t-1}} \nonumber
\\
&~~~~~ + \underbrace{2(1+ Q) L^{ 2}_{max} \eta^2  \sum\limits_{j=1}^N \frac{1}{K_j} \sum \limits_{k=1}^{K_j} \ex^{t_0} \| \gkj(\ykj^{t_0})\|^2}_{\Xi} .\label{eqn.recursion_main}
\end{align}
\\
More specifically,  we have :
\begin{align}
\Psi^t \leq A \Psi^{t-1} + \Xi.
\end{align}
\\
We then expand the recurrence relation as follows from $t=t_0+1,t_0+2...$:
\begin{align}
&\Psi^{t} \leq A^{t - t_0} \Psi^{t_0} + \Xi \sum \limits_{p=0}^{t-t_0-1} A^p.
\end{align}
\\
We note that $\Psi_0 = \sum\limits_{j=1}^N \frac{1}{K_j} \sum \limits_{k=1}^{K_j} \ex^{t_0}\| \gkj(\ykj^{t_0}) - \gkj (\ykj^{t_0})\|^2 = 0$. Therefore we have that
\begin{align}
\Psi^{t} & \leq \Xi \sum \limits_{p=0}^{t-t_0-1} A^p
\\
& \leq \Xi \frac{A^{t-t_0}-1}{A-1}.
\end{align}\tabularnewline
\\
Now summing $\Psi^{t}$ over the set of local iterations between $t_0$ and $t_0^+$, where $t_0^+ \eqdef t_0 + Q-1$, we get:
\begin{align}
\sum \limits_{t=t_0}^{t_0^+} \Psi^t
& \leq \Xi \sum \limits_{t=t_0}^{t_0^+} \frac{A^{t-t_0}-1}{A-1}
\\
& \leq \frac{\Xi}{A-1} \left ( \left( \sum \limits_{t=t_0}^{t_0^+} A^{t-t_0} \right)- Q \right)
\\
& \leq \frac{\Xi}{A-1} \left ( \frac{A^Q -1}{A-1}- Q \right)
\\
& \leq \frac{\Xi}{(1+ \frac{2}{Q})-1} \left ( \frac{(1+ \frac{2}{Q})^Q -1}{(1+ \frac{2}{Q})-1}- Q \right)
\\
& = \frac{Q\Xi}{2} \left ( \frac{Q(1+ \frac{2}{Q})^Q -1}{2}- Q \right)
\\
& = \frac{Q^2 \Xi}{2} \left ( \frac{(1+ \frac{2}{Q})^Q -1}{2}- 1 \right)
\\
& \overset{(a)}{\leq} \frac{Q^2 \Xi}{2} \left ( \frac{e^2-1}{2}- 1 \right)
\\
& \leq 2Q^2 \Xi \label{eqn.new_method_9}
\end{align}
where we upper bound the terms inside the parenthesis of (a) by 4. 
\\
Plugging in the value of $\Xi$ from (\ref{eqn.recursion_main}) into (\ref{eqn.new_method_9}) we get:
\begin{align}
& \sum \limits_{t=t_0}^{t_0^+} \sum\limits_{j=1}^N \frac{1}{K_j} \sum \limits_{k=1}^{K_j} \ex^{t_0} \| \gkj(\ykj^t) - \gkj(\ykj^{t_0}) \|^2 \nonumber
\\
&\leq 4Q^2(1+ Q) L^{ 2}_{max} \eta^2 \sum\limits_{j=1}^N \frac{1}{K_j} \sum \limits_{k=1}^{K_j} \ex^{t_0} \| \gkj(\ykj^{t_0})\|^2.
\end{align}

\end{proof}
\end{lemmavfl}

%%%%%%%%%%%%%%%%%%%%%%%%%%%%%%%%%%%%%%%%%%%%%%%%%%%%%%%%%%%%%%%%%%%%%%%%%%%%%%%%
%%%%%%%%%%%%%%%%%%%%%%%%%%  New Proof    %%%%%%%%%%%%%%%%%%%%%%%%%%%%%%%%%%%%%%
%%%%%%%%%%%%%%%%%%%%%%%%%%%%%%%%%%%%%%%%%%%%%%%%%%%%%%%%%%%%%%%%%%%%%%%%%%%%%%%%
\subsection{Proof of the Theorem~\ref{thm.main_theorem}}
We now prove our main result.

\begin{proof}
We have that $t_0^+ = t_0 +Q -1$. We return to the expression in (\ref{eqn.original}), on which we use the Lipschitz condition and first order taylor series expansion to obtain the following:
\begin{align}
&\Lc(\gm^{t_0^+}) - \Lc(\gm^{t_0})\nonumber
\\
& \leq \left \langle \nabla \Lc(\gm^{t_0}), \gm^{t_0^+} - \gm^{t_0} \right \rangle + \frac{L}{2} \| \gm^{t_0^+} - \gm^{t_0} \|^2 
\\
& \overset{(a)}{\leq} - \left \langle \nabla \Lc(\gm^{t_0}), \eta \sum \limits_{t=t_0}^{t_0^+} \bm{G}^t \right \rangle + \frac{L}{2} \|  \eta \sum \limits_{t=t_0}^{t_0^+} \bm{G}^t  \|^2
\\
& = - \eta \sum \limits_{t=t_0}^{t_0^+} \sum \limits_{j=1}^N \left \langle \nabla \Lc_{(j)}(\gm^{t_0}), \bm{G}_{(j)}^t \right \rangle + \frac{L}{2} \|  \eta \sum \limits_{t=t_0}^{t_0^+} \bm{G}^t  \|^2
\\
& = - \eta \sum \limits_{t=t_0}^{t_0^+} \sum \limits_{j=1}^N \frac{1}{K_j} \sum \limits_{k=1}^{K_j} \left \langle \nabla \Lc_{(j)}(\gm^{t_0}), \gkj (\ykj^t) \right \rangle + \frac{L}{2} \|  \eta \sum \limits_{t=t_0}^{t_0^+} \bm{G}^t  \|^2
\\
& = \eta \sum \limits_{t=t_0}^{t_0^+} \sum \limits_{j=1}^N \frac{1}{K_j} \sum \limits_{k=1}^{K_j} \left \langle - \nabla \Lc_{(j)}(\gm^{t_0}), \gkj (\ykj^t)- \gkj (\ykj^{t_0}) \right \rangle
\\
&~~~~~~~~~~~~ - \eta \sum \limits_{t=t_0}^{t_0^+} \sum \limits_{j=1}^N \frac{1}{K_j} \sum \limits_{k=1}^{K_j} \left \langle \nabla \Lc_{(j)}(\gm^{t_0}), \gkj (\ykj^{t_0}) \right \rangle + \frac{L}{2} \|  \eta \sum \limits_{t=t_0}^{t_0^+} \bm{G}^t  \|^2
\\
& \overset{(b)}{\leq} \frac{\eta}{2} \sum \limits_{t=t_0}^{t_0^+} \sum \limits_{j=1}^N \frac{1}{K_j} \sum \limits_{k=1}^{K_j} \left[ \| \nabla \Lc_{(j)}(\gm^{t_0})\|^2 + \|\gkj (\ykj^t)- \gkj (\ykj^{t_0})\|^2 \right ]
\\
&~~~~~~~~~~~~ - \eta \sum \limits_{t=t_0}^{t_0^+} \sum \limits_{j=1}^N \frac{1}{K_j} \sum \limits_{k=1}^{K_j} \left \langle \nabla \Lc_{(j)}(\gm^{t_0}), \gkj (\ykj^{t_0}) \right \rangle + \frac{L}{2} \|  \eta \sum \limits_{t=t_0}^{t_0^+} \bm{G}^t  \|^2
 \label{line.new_method_1}
\end{align}
where (a) follows from (\ref{eqn.original}) and (b) follows since $A\cdot B = \frac{1}{2} (A^2 + B^2 - (A+B)^2) \leq \frac{1}{2} (A^2 + B^2)$.
\\
Taking conditional expectation at $t_0$ on both sides of (\ref{line.new_method_1}), we get:
\begin{align}
& \ex^{t_0} [ \Lc(\gm^{t_0^+})] - \Lc(\gm^{t_0}) \nonumber
\\
& \leq \frac{\eta}{2} \sum \limits_{t=t_0}^{t_0^+} \sum \limits_{j=1}^N \frac{1}{K_j} \sum \limits_{k=1}^{K_j} \left[ \ex^{t_0}\| \nabla \Lc_{(j)}(\gm^{t_0})\|^2 + \ex^{t_0}\|\gkj (\ykj^t)- \gkj (\ykj^{t_0})\|^2 \right ] \nonumber
\\
&~~~~~~~~~~~~ - \eta \sum \limits_{t=t_0}^{t_0^+} \sum \limits_{j=1}^N \frac{1}{K_j} \sum \limits_{k=1}^{K_j} \ex^{t_0} \left \langle \nabla \Lc_{(j)}(\gm^{t_0}), \gkj (\ykj^{t_0}) \right \rangle + \frac{L}{2} \ex^{t_0} \|  \eta \sum \limits_{t=t_0}^{t_0^+} \bm{G}^t  \|^2
\\
& \leq \frac{\eta}{2} \sum \limits_{t=t_0}^{t_0^+} \sum \limits_{j=1}^N \frac{1}{K_j} \sum \limits_{k=1}^{K_j} \| \nabla \Lc_{(j)}(\gm^{t_0})\|^2 + \frac{\eta}{2} \sum \limits_{t=t_0}^{t_0^+} \sum \limits_{j=1}^N \frac{1}{K_j} \sum \limits_{k=1}^{K_j}  \ex^{t_0}\|\gkj (\ykj^t)- \gkj (\ykj^{t_0})\|^2 \nonumber
\\
&~~~~~~~~~~~~ - \eta \sum \limits_{t=t_0}^{t_0^+} \sum \limits_{j=1}^N \frac{1}{K_j} \sum \limits_{k=1}^{K_j} \ex^{t_0} \left \langle \nabla \Lc_{(j)}(\gm^{t_0}), \gkj (\ykj^{t_0}) \right \rangle + \frac{L}{2} \ex^{t_0} \|  \eta \sum \limits_{t=t_0}^{t_0^+} \bm{G}^t  \|^2
\\
& = \frac{\eta}{2} \sum \limits_{t=t_0}^{t_0^+} \sum \limits_{j=1}^N \frac{1}{K_j} \sum \limits_{k=1}^{K_j} \| \nabla \Lc_{(j)}(\gm^{t_0})\|^2 + \frac{\eta}{2} \sum \limits_{t=t_0}^{t_0^+} \sum \limits_{j=1}^N \frac{1}{K_j} \sum \limits_{k=1}^{K_j}  \ex^{t_0}\|\gkj (\ykj^t)- \gkj (\ykj^{t_0})\|^2 \nonumber
\\
&~~~~~~~~~~~~ - \eta \sum \limits_{t=t_0}^{t_0^+} \sum \limits_{j=1}^N \frac{1}{K_j} \sum \limits_{k=1}^{K_j} \| \nabla \Lc_{(j)}(\gm^{t_0})\|^2 + \frac{L}{2} \ex^{t_0} \|  \eta \sum \limits_{t=t_0}^{t_0^+} \bm{G}^t  \|^2
\\
& = -\frac{\eta}{2} \sum \limits_{t=t_0}^{t_0^+} \sum \limits_{j=1}^N \frac{1}{K_j} \sum \limits_{k=1}^{K_j} \| \nabla \Lc_{(j)}(\gm^{t_0})\|^2 + \frac{\eta}{2} \sum \limits_{t=t_0}^{t_0^+} \sum \limits_{j=1}^N \frac{1}{K_j} \sum \limits_{k=1}^{K_j}  \ex^{t_0}\|\gkj (\ykj^t)- \gkj (\ykj^{t_0})\|^2 \nonumber
\\
&~~~~~~~~~~~~ + \frac{L}{2} \ex^{t_0} \|  \eta \sum \limits_{t=t_0}^{t_0^+} \bm{G}^t  \|^2
\\
& \overset{(a)}{\leq} -\frac{\eta}{2} \sum \limits_{t=t_0}^{t_0^+} \sum \limits_{j=1}^N \frac{1}{K_j} \sum \limits_{k=1}^{K_j} \| \nabla \Lc_{(j)}(\gm^{t_0})\|^2 + \frac{\eta}{2} \sum \limits_{t=t_0}^{t_0^+} \sum \limits_{j=1}^N \frac{1}{K_j} \sum \limits_{k=1}^{K_j}  \ex^{t_0}\|\gkj (\ykj^t)- \gkj (\ykj^{t_0})\|^2 \nonumber
\\
&~~~~~~~~~~~~ + \eta^2 LQ \ex^{t_0} \sum \limits_{t=t_0}^{t_0^+} \| \bm{G}^t - \bm{G}^{t_0} \|^2 + \eta^2 LQ \ex^{t_0} \sum \limits_{t=t_0}^{t_0^+} \|\bm{G}^{t_0}  \|^2
\\
& \leq -\frac{\eta}{2} \sum \limits_{t=t_0}^{t_0^+} \sum \limits_{j=1}^N \frac{1}{K_j} \sum \limits_{k=1}^{K_j} \| \nabla \Lc_{(j)}(\gm^{t_0})\|^2 + \frac{\eta}{2} \sum \limits_{t=t_0}^{t_0^+} \sum \limits_{j=1}^N \frac{1}{K_j} \sum \limits_{k=1}^{K_j}  \ex^{t_0}\|\gkj (\ykj^t)- \gkj (\ykj^{t_0})\|^2 \nonumber
\\
&~~~~~~~~~~~~ + \eta^2 LQ \ex^{t_0} \sum \limits_{t=t_0}^{t_0^+} \sum \limits_{j=1}^N \frac{1}{K_j} \sum \limits_{k=1}^{K_j}  \ex^{t_0}\|\gkj (\ykj^t)- \gkj (\ykj^{t_0})\|^2 + \eta^2 LQ \ex^{t_0} \sum \limits_{t=t_0}^{t_0^+} \|\bm{G}^{t_0}  \|^2
\\
& = -\frac{\eta}{2} \sum \limits_{t=t_0}^{t_0^+} \sum \limits_{j=1}^N \frac{1}{K_j} \sum \limits_{k=1}^{K_j} \| \nabla \Lc_{(j)}(\gm^{t_0})\|^2 + \frac{\eta}{2} (1 + 2\eta L Q) \sum \limits_{t=t_0}^{t_0^+} \sum \limits_{j=1}^N \frac{1}{K_j} \sum \limits_{k=1}^{K_j}  \ex^{t_0}\|\gkj (\ykj^t)- \gkj (\ykj^{t_0})\|^2 \nonumber
\\
&~~~~~~~~~~~~ + \eta^2 LQ \ex^{t_0} \sum \limits_{t=t_0}^{t_0^+} \|\bm{G}^{t_0}  \|^2
\\
& \leq -\frac{\eta}{2} \sum \limits_{t=t_0}^{t_0^+} \sum \limits_{j=1}^N \frac{1}{K_j} \sum \limits_{k=1}^{K_j} \| \nabla \Lc_{(j)}(\gm^{t_0})\|^2 + \frac{\eta}{2} (1 + 2\eta L Q) \sum \limits_{t=t_0}^{t_0^+} \sum \limits_{j=1}^N \frac{1}{K_j} \sum \limits_{k=1}^{K_j}  \ex^{t_0}\|\gkj (\ykj^t)- \gkj (\ykj^{t_0})\|^2 \nonumber
\\
&~~~~~~~~~~~~ + \eta^2 LQ \sum \limits_{t=t_0}^{t_0^+} \sum \limits_{j=1}^N \frac{1}{K_j} \sum \limits_{k=1}^{K_j} \ex^{t_0} \|\gkj(\ykj^{t_0})  \|^2
\\
& \overset{(b)}{\leq} -\frac{\eta}{2} \sum \limits_{t=t_0}^{t_0^+} \sum \limits_{j=1}^N \frac{1}{K_j} \sum \limits_{k=1}^{K_j} \| \nabla \Lc_{(j)}(\gm^{t_0})\|^2 + 2\eta Q^2 (1 + 2\eta L Q)(1+ Q) L^{ 2}_{max} \eta^2 \sum\limits_{j=1}^N \frac{1}{K_j} \sum \limits_{k=1}^{K_j} \ex^{t_0} \| \gkj(\ykj^{t_0})\|^2 \nonumber
\\
&~~~~~~~~~~~~ + \eta^2 LQ  \sum \limits_{t=t_0}^{t_0^+} \sum \limits_{j=1}^N \frac{1}{K_j} \sum \limits_{k=1}^{K_j} \ex^{t_0} \|\gkj(\ykj^{t_0})  \|^2
\\
& = -\frac{\eta Q}{2} \sum \limits_{j=1}^N \frac{1}{K_j} \sum \limits_{k=1}^{K_j} \| \nabla \Lc_{(j)}(\gm^{t_0})\|^2 + 2\eta Q^2 (1 + 2\eta L Q)(1+ Q) L^{ 2}_{max} \eta^2 \sum\limits_{j=1}^N \frac{1}{K_j} \sum \limits_{k=1}^{K_j} \ex^{t_0} \| \gkj(\ykj^{t_0})\|^2 \nonumber
\\
&~~~~~~~~~~~~ + \eta^2 LQ^2 \sum \limits_{j=1}^N \frac{1}{K_j} \sum \limits_{k=1}^{K_j} \ex^{t_0} \|\gkj(\ykj^{t_0})  \|^2
\\
& = -\frac{\eta Q}{2} \sum \limits_{j=1}^N \frac{1}{K_j} \sum \limits_{k=1}^{K_j} \| \nabla \Lc_{(j)}(\gm^{t_0})\|^2 \nonumber
\\
&~~~~~~~~~~~~ + (\eta^2 L Q^2 + 2\eta Q^2 (1 + 2\eta L Q)(1+ Q) L^{ 2}_{max} \eta^2 ) \sum\limits_{j=1}^N \frac{1}{K_j} \sum \limits_{k=1}^{K_j} \ex^{t_0} \| \gkj(\ykj^{t_0})\|^2
\\
& \overset{(c)}{\leq} -\frac{\eta Q}{2} \sum \limits_{j=1}^N \frac{1}{K_j} \sum \limits_{k=1}^{K_j} \| \nabla \Lc_{(j)}(\gm^{t_0})\|^2 \nonumber
\\
&~~~~~~~~~~~~ + (\eta^2 L Q^2 + 2\eta Q^2 (1 + 2\eta L Q)(1+ Q) L^{ 2}_{max} \eta^2 ) \sum\limits_{j=1}^N \frac{1}{K_j} \sum \limits_{k=1}^{K_j}\left( \| \nabla \Lc_{(j)}(\gm^{t_0})\|^2 + \sigma_j^2\right)
\\
& \leq-  \left(\frac{\eta Q}{2} - \eta^2 L Q^2 - 2\eta Q^2 (1 + 2\eta L Q)(1+ Q) L^{ 2}_{max} \eta^2 \right) \sum \limits_{j=1}^N \frac{1}{K_j} \sum \limits_{k=1}^{K_j} \| \nabla \Lc_{(j)}(\gm^{t_0})\|^2 \nonumber
\\
&~~~~~~~~~~~~ + (\eta^2 L Q^2 + 2\eta Q^2 (1 + 2\eta L Q)(1+ Q) L^{ 2}_{max} \eta^2 ) \sum\limits_{j=1}^N \frac{1}{K_j} \sum \limits_{k=1}^{K_j}\sigma_j^2
\\
& = -  \frac{\eta Q}{2} \left(1 - 2\eta L Q - 4Q (1 + 2\eta L Q)(1+ Q) L^{ 2}_{max} \eta^2 \right) \sum \limits_{j=1}^N \frac{1}{K_j} \sum \limits_{k=1}^{K_j} \| \nabla \Lc_{(j)}(\gm^{t_0})\|^2 \nonumber
\\
&~~~~~~~~~~~~ + (\eta^2 L Q^2 + 2\eta Q^2 (1 + 2\eta L Q)(1+ Q) L^{ 2}_{max} \eta^2 ) \sum\limits_{j=1}^N \frac{1}{K_j} \sum \limits_{k=1}^{K_j}\sigma_j^2
\\
& \leq -  \frac{\eta Q}{2} \left(1 - 2\eta L Q - 8Q^2 L^{ 2}_{max} \eta^2 -  16Q^3 L L^{ 2}_{max} \eta^3 \right) \sum \limits_{j=1}^N \frac{1}{K_j} \sum \limits_{k=1}^{K_j} \| \nabla \Lc_{(j)}(\gm^{t_0})\|^2 \nonumber
\\
&~~~~~~~~~~~~ + (\eta^2 L Q^2 + 4\eta Q^3 (1 + 2\eta L Q) L^{ 2}_{max} \eta^2 ) \sum\limits_{j=1}^N \frac{1}{K_j} \sum \limits_{k=1}^{K_j}\sigma_j^2 \label{eqn.new_method_4}
\end{align}
where in (a) the extra $Q$ in the last term arises since we pull the summation over $t$ from inside the norm to outside and break the norm into two parts. We obtain (b) from Lemma~\ref{lemma.new_lemma} and (c) follows from (\ref{eqn.simplified_variance_t0}).
\\
Let $\eta \leq \frac{1}{8 Q \max (L, L_{max})}$. We can then bound (\ref{eqn.new_method_4}) as follows: 
\begin{align}
\ex^{t_0} [ \Lc(\gm^{t_0^+})] - \Lc(\gm^{t_0})
& \leq -  \frac{\eta Q}{2} \left(1 - \frac{1}{4} - \frac{1}{8} -  \frac{1}{32} \right) \sum \limits_{j=1}^N \frac{1}{K_j} \sum \limits_{k=1}^{K_j} \| \nabla \Lc_{(j)}(\gm^{t_0})\|^2 \nonumber
\\
&~~~~~~~~~~~~ + (\eta^2 L Q^2 + 4\eta Q^3 (1 + 2\eta L Q) L^{ 2}_{max} \eta^2 ) \sum\limits_{j=1}^N \frac{1}{K_j} \sum \limits_{k=1}^{K_j}\sigma_j^2
\\
& \leq -  \frac{\eta Q}{4} \sum \limits_{j=1}^N \frac{1}{K_j} \sum \limits_{k=1}^{K_j} \| \nabla \Lc_{(j)}(\gm^{t_0})\|^2 \nonumber
\\
&~~~~~~~~~~~~ + (\eta^2 L Q^2 + 4\eta^3 Q^3 L^{ 2}_{max} + 8\eta^4 L Q^4 L^{ 2}_{max} ) \sum\limits_{j=1}^N \frac{1}{K_j} \sum \limits_{k=1}^{K_j}\sigma_j^2 
\\
& \overset{(a)}{\leq} -  \frac{\eta Q}{4} \| \nabla \Lc(\gm^{t_0})\|^2 \nonumber
\\
&~~~~~~~~~~~~ + (\eta^2 L Q^2 + 4\eta^3 Q^3 L^{ 2}_{max} + 8\eta^4 L Q^4 L^{ 2}_{max} ) \sum\limits_{j=1}^N \frac{1}{K_j} \sum \limits_{k=1}^{K_j}\sigma_j^2 \label{eqn.new_method_5}
\end{align}
where (a) follows from using the definition of $\Lc$ in the first term.

Rearranging the terms in (\ref{eqn.new_method_5}), we get:
\begin{align}
\| \nabla \Lc(\gm^{t_0})\|^2 
& \leq 4 \frac{\Lc(\gm^{t_0}) - \ex^{t_0} [ \Lc(\gm^{t_0^+})]}{\eta Q}  \nonumber
\\
&~~~~ + 4 (\eta L Q + 4\eta^2 Q^2 L^{ 2}_{max} + 8\eta^3 Q^3L L^{ 2}_{max} ) \sum\limits_{j=1}^N \frac{1}{K_j} \sum \limits_{k=1}^{K_j}\sigma_j^2 .\label{eqn.new_method_6}
\end{align}

Now averaging over all the communication rounds, i.e. all the intervals of $Q$ local iterations over $t_0 = 0, 1, \cdots R-1$, such that $T=QR$, and taking total expectation: 
\begin{align}
&\frac{1}{R}\sum\limits_{t_0=0}^{R-1} \ex \left[ \| \nabla \Lc(\gm^{t_0})\|^2 \right]  \nonumber
\\
&~~~ \leq  4 \frac{1}{R} \ex\sum\limits_{t_0}^{R-1} \frac{\Lc(\gm^{t_0}) - \ex^{t_0} [ \Lc(\gm^{t_0^+})]}{\eta Q}  \nonumber
\\
&~~~~~~ + \frac{1}{R} \ex \sum\limits_{t_0}^{R-1} 4 (\eta L Q + 4\eta^2 Q^2 L^{ 2}_{max} + 8\eta^3 Q^3 L L^{ 2}_{max} ) \sum\limits_{j=1}^N \frac{1}{K_j} \sum \limits_{k=1}^{K_j}\sigma_j^2  \nonumber
\\
&~~~ \leq \frac{4 (\Lc(\gm^{0}) - \ex [ \Lc(\gm^{T})])}{\eta QR} +  4 (\eta L Q + 4\eta^2 Q^2 L^{ 2}_{max} + 8\eta^3 Q^3 L L^{ 2}_{max} ) \sum\limits_{j=1}^N \sigma_j^2. \label{eqn.new_method_7}
\end{align}

\end{proof}

%%%
%%% The acknowledgments section is defined using the "acks" environment
%%% (and NOT an unnumbered section). This ensures the proper
%%% identification of the section in the article metadata, and the
%%% consistent spelling of the heading.
\begin{acks}
This work is supported by the Rensselaer-IBM AI Research Collaboration (http://airc.rpi.edu), part of the IBM AI Horizons Network (http://ibm.biz/AIHorizons), and by the National Science Foundation under grants CNS 1553340 and CNS 1816307. A preliminary conference version of this work has been published in~\cite{icasspdasa2}.
\end{acks}

%%%%%%%%%%%%%%%%%%%%%%%%%%%%%%%%%%%%%%%%%%%%%%%%%%%%%%%%%%%%%%%%%%%%%%%%%%%%%%%%
%%%%%%%%%%%%%%%%%%%%%%%%%%  REFERENCES    %%%%%%%%%%%%%%%%%%%%%%%%%%%%%%%%%%%%%%
%%%%%%%%%%%%%%%%%%%%%%%%%%%%%%%%%%%%%%%%%%%%%%%%%%%%%%%%%%%%%%%%%%%%%%%%%%%%%%%%
\bibliographystyle{ACM-Reference-Format}
\bibliography{ref}

\end{document}